\documentclass[twoside,11pt]{article}
\usepackage{jair, theapa, rawfonts}
\usepackage{rotating}
\usepackage{slashbox,multirow}
\usepackage[utf8]{inputenc} 
\usepackage[T1]{fontenc}    
\usepackage{hyperref}       
\usepackage{url}            
\usepackage{booktabs}       
\usepackage{amsfonts}       
\usepackage{nicefrac}       
\usepackage{microtype}      
\usepackage{amsthm}     
\usepackage{algorithm,algorithmic}
\usepackage{mathtools}
\usepackage{amsmath}
\usepackage{time}
\allowdisplaybreaks
\usepackage{amssymb}

\usepackage{balance}
\usepackage{comment}

\newcommand{\eq}[1]{{Eq~(#1)}}

\newtheorem{theorem}{Theorem}
\newtheorem{thm}{Theorem}
\newtheorem{lemma}{Lemma}

\newtheorem{proposition}{Proposition}

\usepackage{pifont}
\newcommand{\cmark}{\ding{51}}%

\newcommand{\ov}[1]{{\overline{\mathbf{#1}}}}
\usepackage{amsbsy}
\usepackage{amsmath}
\usepackage{graphicx}
\usepackage{subfigure}
\usepackage{color}
\usepackage{booktabs}

\usepackage{natbib}
\usepackage{etoolbox}
\makeatletter
\patchcmd{\@makefntext}{\insertfootnotetext{#1}}{\insertfootnotetext{\scriptsize#1}}{}{}
\makeatother

\usepackage[normalem]{ulem} 

\newcommand{\vg}{{\mathbf{g}}}

\newcommand{\vv}{{\mathbf{v}}}
\newcommand{\vw}{{\mathbf{w}}}
\newcommand{\vx}{{\mathbf{x}}}
\newcommand{\vy}{{\mathbf{y}}}
\newcommand{\vz}{{\mathbf{z}}}

\newcommand{\vH}{{\mathbf{H}}}

\newcommand{\cI}{{\mathcal{I}}}

\newcommand{\cO}{{\mathcal{O}}}

\newcommand{\cS}{{\mathcal{S}}}

\newcommand{\EE}{\mathbb{E}}

\newcommand{\ZZ}{\mathbb{Z}}

\newcommand{\bc}{\begin{center}}
\newcommand{\ec}{\end{center}}

\newcommand{\bdm}{\begin{displaymath}}
\newcommand{\edm}{\end{displaymath}}

\newcommand{\beq}{\begin{equation}}
\newcommand{\eeq}{\end{equation}}

\newcommand{\bfl}{\begin{flushleft}}
\newcommand{\efl}{\end{flushleft}}

\newcommand{\bt}{\begin{tabbing}}
\newcommand{\et}{\end{tabbing}}

\newcommand{\beqn}{\begin{eqnarray}}
\newcommand{\eeqn}{\end{eqnarray}}

\newcommand{\beqs}{\begin{align*}} 
\newcommand{\eeqs}{\end{align*}}  

\newcommand\numberthis{\addtocounter{equation}{1}\tag{\theequation}}

\newtheorem*{assumption*}{Assumption}
\newtheorem*{claim*}{Claim}
\newtheorem{ass}{Assumption}
\newcommand{\rebuttal}[1]{{\color{black}{#1}}}

\jairheading{78}{2023}{1143-1200}{07/2023}{12/2023}
\ShortHeadings{A Unified Linear Speedup Analysis of Federated Averaging and Nesterov FedAvg}
{Qu, Lin, Li, Zhou, \& Zhou}
\firstpageno{1143}

\begin{document}

\title{A Unified Linear Speedup Analysis of\\ Federated Averaging and Nesterov FedAvg}

\author{Zhaonan Qu, Kaixiang Lin, Zhaojian Li, Jiayu Zhou, and Zhengyuan Zhou
\thanks{$^*$ The first two authors contributed equally to this work.}
\thanks{Zhaonan Qu is with the Department of Economics, Stanford University, Stanford, CA 94305, USA. Email: {zhaonanq}@stanford.edu.}
\thanks{Kaixiang Lin was with the Department of Computer Science and Engineering, Michigan State University, Lansing, MI 48824,
USA. Email: {linkaixi}@msu.edu.
}
\thanks{Zhaojian Li is with the Department of
Mechanical Engineering, Michigan State University, Lansing, MI 48824,
USA. Email: {lizhaoj1}@msu.edu.}
\thanks{Jiayu Zhou is with the Department of Computer Science and Engineering, Michigan State University, Lansing, MI 48824,
USA. Email: {jiayuz}@msu.edu.}
\thanks{Zhengyuan Zhou is with the Department of Technology, Operations, and Statistics at Stern School of Business, New York University, New York, NY 10003,
USA. Email: {zzhou}@stern.nyu.edu.}
}
       
\author{\name Zhaonan Qu \email zhaonanq@stanford.edu \\
\addr 
 Department of Economics and Department of Management Science and Engineering, \\Stanford University, Stanford, CA 94305, USA
\AND
       \name Kaixiang Lin \email {kaixianglin.cs}@gmail.com\\
\addr Department of Computer Science and Engineering,\\ Michigan State University, Lansing, MI 48824,
USA 
       \AND
       \name Zhaojian Li \email {lizhaoj1}@msu.edu\\
       \addr Department of
Mechanical Engineering,\\ Michigan State University, Lansing, MI 48824,
USA
       \AND
       \name  Jiayu Zhou \email {jiayuz}@msu.edu \\
       \addr Department of Computer Science and Engineering,\\ Michigan State University, Lansing, MI 48824,
USA
\AND
\name Zhengyuan Zhou \email {zzhou}@stern.nyu.edu \\ \addr Department of Technology, Operations, and Statistics, \\ Stern School of Business, New York University, New York, NY 10003,
USA
      }

\maketitle

\begin{abstract}
Federated learning (FL) learns a model jointly from a set of participating devices without sharing each other's privately held data. The characteristics of non-\textit{i.i.d.} data across the network, low device participation, high communication costs, and the mandate that data remain private bring challenges in understanding the convergence of FL algorithms, particularly regarding how convergence scales with the number of participating devices. 
In this paper, we focus on Federated Averaging (FedAvg), one of the most popular and effective FL algorithms in use today, as well as its Nesterov accelerated variant, and conduct a systematic study of how their
convergence scale with the number of participating devices under non-\textit{i.i.d.} data and partial participation in convex settings. We provide a unified analysis that establishes convergence
guarantees for FedAvg under strongly convex, convex, and
overparameterized strongly convex problems. We show that FedAvg enjoys
linear speedup in each case, although with different convergence rates and
communication efficiencies.
For strongly convex and convex problems, we also characterize
the corresponding convergence rates for the Nesterov accelerated FedAvg
algorithm, which are the first linear speedup guarantees for momentum variants
of FedAvg in convex settings. Empirical studies of the algorithms in various
settings have supported our theoretical results.
\end{abstract}


\section{Introduction}
Federated learning (FL) is a machine learning paradigm where many clients (e.g., mobile devices or organizations) collaboratively train a model under the orchestration of a central server (e.g., service provider), while keeping the training data decentralized in order to protect privacy and improve efficiency \citep{konevcny2016federated,mcmahan2016communication,smith2017federated,li2018federated, kairouz2019advances,yang2019federated,wang2021field,li2021survey}. In recent years, FL has swiftly emerged as an important learning paradigm, enjoying widespread success in such diverse applications as personalized recommendations and assistance \citep{chen2018federated,lam2019protecting}, keyboard suggestion and prediction \citep{47586,yang2018applied}, smart healthcare \citep{rieke2020future,brisimi2018federated}, and the Internet of Things (IoT) \citep{zhao2019multi,nguyen2021federated}, to name a few. There are at least three reasons for its popularity: First, the rapid proliferation of smart devices that are equipped with both computing power and data-capturing capabilities provided the infrastructure core for FL \citep{qin2021federated}. Second, the rising awareness of privacy and the explosive growth of computational power in mobile devices have made it increasingly attractive to push the computation to the edge \citep{lim2020federated}. Third, the empirical success of \emph{communication-efficient} FL algorithms has enabled increasingly larger-scale parallel computing and learning with less communication overhead \citep{mcmahan2016communication,sattler2019robust}.

Despite its promise and broad applicability, the potential value FL delivers is coupled with the unique challenges it brings \citep{li2020federated}. In particular, when FL learns a single statistical model using data from across all the devices while keeping each individual device's data isolated, it faces two challenges that are absent in centralized optimization and distributed (stochastic) optimization \citep{zhou2017convergence, woodworth2018graph,jiang2018linear, woodworth2020minibatch,charles2021convergence,luo2021no,tan2022towards}:

1) \textbf{Data heterogeneity (non-\textit{i.i.d.} data):} data distributions on local devices/servers are different, and data cannot be shared across devices;

2) \textbf{System heterogeneity (partial participation):} only a subset of devices may access the central server at each time, which happens because the communication bandwidth profiles vary across devices and there is no central server that has control over when a device is active. \footnote{In some prior works, e.g., \citet{reisizadeh2022straggler}, the problem of stragglers, where slow local devices lag behind overall communication frequencies, is also included in discussions on system heterogeneity. In this paper, we focus on the partial participation aspect of the system heterogeneity problem.}

To address these challenges, Federated Averaging (FedAvg)~\citep{mcmahan2016communication} was proposed as a particularly effective heuristic, which has enjoyed great empirical success. This success has since motivated a growing line of research efforts into understanding its theoretical convergence guarantees in various settings \citep{stich2018local,khaled2019first,haddadpour2019convergence,li2019convergence,wang2019adaptive,yu2019linear,yu2019parallel,wang2018cooperative,koloskova2020unified,woodworth2020local,khaled2020tighter,yang2021achieving}. 
In these works, \cite{li2019convergence} was among the first to establish an $\cO(\frac{1}{T})$ convergence rate for FedAvg for strongly convex smooth FL problems with both data and system heterogeneities. When only data heterogeneity is present,  \cite{khaled2020tighter} provides tight convergence results with linear speedup analysis in convex settings.  
In non-convex settings, \cite{yang2021achieving} obtained linear speedup convergence results for FedAvg under both non-\textit{i.i.d.} data and partial participation.

Despite the recent fruitful efforts to understand the theoretical convergence properties of FedAvg, the question of how the number of participating devices affects the convergence speed remains to be answered fully when both data and system heterogeneity are present. In particular, is linear speedup of FedAvg a universal phenomenon across different settings and for any number of devices? What about when FedAvg is accelerated with momentum updates? Does the presence of both data and system heterogeneity in FL imply different communication complexities and require technical novelties over results in distributed and decentralized optimization? Linear speedup is a desirable property of distributed optimization systems, including FedAvg, as it characterizes the impact of scale on such systems. Here we provide affirmative answers to these questions.

\begin{table*}[t!]
\centering
{\small
\begin{tabular}{|c|c|c|c|}\hline 
	\multirow{2}{*}{\backslashbox{{\tiny Participation} }{{\tiny Objective function}}} & \multirow{2}{*}{Strongly Convex}     &\multirow{2}{*}{Convex}    & \multirow{2}{*}{Overparameterized} \\ 
	                                &                        &         &                    \\ \hline 
	Full                         & $\cO(\frac{1}{NT}+\frac{E^{2}}{T^{2}})$    &  $\mathcal{O}\left(\frac{1}{\sqrt{NT}}+\frac{NE^{2}}{T}\right)$   & $\cO(\exp(-\frac{NT}{E\kappa}))$  \\ \hline
	Partial                      &  $\cO\left(\frac{E}{KT}+\frac{E^{2}}{T^{2}}\right)$   &  $\cO\left(\frac{E}{\sqrt{KT}}+\frac{KE^2}{T} \right)$ &  $\cO(\exp(-\frac{KT}{E\kappa}))$   \\ \hline
\end{tabular}
}
\caption{Our convergence results for FedAvg and accelerated FedAvg in this paper.}
\label{tb:convergencerateintro}
\vspace{-1em}
\end{table*}

\subsection{Our Contributions} 
First, we establish an {\small{$\cO(1/KT)$}} convergence rate for FedAvg for strongly convex and smooth problems and  an
{\small{$\cO(1/\sqrt{KT})$}} convergence rate for convex
and smooth problems. Here $K$ is the lower bound of the number of participating devices at each communication round and $T$ is the number of local steps. These results confirm that FedAvg enjoys the desirable linear speedup property with both non-\textit{i.i.d.} data and partial participation. In previous works, the best and most related convergence analyses are given by \cite{li2019convergence}, which established an $\cO(\frac{1}{T})$ convergence rate for strongly convex smooth problems under FedAvg, and by \cite{khaled2020tighter}, which established \emph{linear speedup} in the number of participating local servers under data heterogeneity.  
Our rate matches the same (and optimal) dependence on $T$, but also establishes the linear speedup dependence on $K$, for any $K\leq N$, where $N$ is the total number of devices, whereas~\cite{li2019convergence} does not have linear speedup analysis, and \cite{khaled2020tighter} focuses on full participation $K=N$.  
The concurrent work of \cite{karimireddy2019scaffold} also established linear speedup convergence under partial participation, using a modified version of the FedAvg with distinct learning rates for local steps and communication rounds. Compared to their work, our analyses are carried out for the original FedAvg algorithm that utilizes a decaying rate independent of local vs. communication rounds. 
Our unified analysis highlights the common elements and distinctions between the strongly convex and convex settings, as well as the communication complexity differences between the full and partial participation settings.

Second, we establish the same convergence rates--{\small{$\cO(1/KT)$}} for strongly convex and smooth problems and {\small{$\cO(1/\sqrt{KT})$}} for convex and smooth problems--for Nesterov accelerated FedAvg. 
We analyze the accelerated version of FedAvg here because empirically it tends to perform better; yet, its theoretical convergence guarantee is unknown. To the best of our knowledge, these are the first results that provide a linear speedup characterization of Nesterov accelerated FedAvg in the two convex problem classes. The fact that FedAvg and Nesterov accelerated FedAvg share the same convergence rate is to be expected: this is the case even for general centralized stochastic optimization problems.  
Prior to our results, the most relevant results only concern the non-convex setting \citep{yu2019linear,li2018federated,huo2020faster}, where convergence is measured with respect to stationary points (vanishing of gradient norms, rather than optimality gaps). Our unified analysis of Nesterov FedAvg also illustrates the technical similarities and distinctions compared to the original FedAvg algorithm, whereas prior works in the non-convex setting used different frameworks with distinct proof techniques.

Third, we study a subclass of strongly convex smooth problems where the objective is over-parameterized and establish 
a faster $\cO(\exp(-\frac{KT}{\kappa}))$ geometric convergence rate for FedAvg, in contrast to the $\cO(\exp(-\frac{T}{\kappa}))$ rate for individual solvers~\citep{ma2017power}. Within this class, we further consider the linear regression problem and establish an even sharper rate for FedAvg. To our knowledge, these bounds are among the first to extend the geometric convergence results in the non-distributed overparameterized setting to the federated learning setting with a linear speedup in the number of local servers. 

\subsection{Connections with Distributed and Decentralized Optimization}
Federated learning is closely related to distributed and decentralized optimization, and as such it is important to discuss connections and distinctions between our work and related results from that literature. First, when there is neither system heterogeneity, i.e., all devices participate in parameter averaging during a communication round, nor data heterogeneity, i.e., all devices have access to a common set of stochastic gradients, FedAvg coincides with the ``Local SGD'' of~\cite{stich2018local}, which showed the linear speedup rate $\mathcal{O}(1/NT)$ for strongly convex and smooth functions. \cite{woodworth2020minibatch,woodworth2020local} further improved the communication complexity that guarantees the linear speedup rate. When there is only data heterogeneity, some works such as \cite{khaled2020tighter} have continued to use the term Local SGD to refer to FedAvg, while others subsume it in more general frameworks that include decentralized model averaging based on a network topology or a mixing matrix. They have provided linear speedup analyses for strongly convex and convex problems, e.g.,~\cite{khaled2020tighter,koloskova2020unified} as well as non-convex problems, e.g.,~\cite{jiang2018linear,yu2019parallel,wang2018cooperative}. 

However, most of these results do not consider system heterogeneity, where a \emph{subset} of nodes participate in the updates during a communication round. Even with decentralized model averaging, the assumptions usually imply that model averages over all devices is the same as decentralized model averages based on network topology (e.g.,~\cite{koloskova2020unified} Proposition 1), which precludes system heterogeneity as defined in this paper and prevalent in FL problems. For momentum accelerated FedAvg,~\cite{yu2019linear} provided linear speedup analysis for non-convex problems, while results for strongly convex and convex settings are entirely lacking, even without system heterogeneity. In contrast, our linear speedup analyses for FedAvg and consider both types of heterogeneity present in the full federated learning setting, and are valid for almost any number of participating devices. We also highlight a distinction in communication efficiency when system heterogeneity is present. Moreover, our results for Nesterov accelerated FedAvg completes the picture for strongly convex and convex problems. See Table~\ref{tb:convergencerateintro} for a summary of our convergence bounds for FedAvg and Nesterov Accelerated FedAvg. For a detailed comparison with related works, please refer to Table~\ref{tb:convergenceratev3} in Appendix~\ref{sec:app:comparison}.

Throughout the paper, $N$ is the total number of local devices, and $K \leq N$ is the number of devices that are accessible to the central server during each communication round. $T$ is the total number of stochastic updates performed by each local device, $E$ is the local steps between two consecutive server communications (and hence $T/E$ is the number of communications). Depending on the particular overparameterized setting, $\kappa$ is a type of condition number defined in Section \ref{sec:app:overparameterized} and Appendix \ref{sec:interpolation}.

\section{Setup}
In this paper, we study the following federated learning problem:
\begin{align}
	\min _{\mathbf{w}}\left\{F(\mathbf{w}) \triangleq \sum\nolimits_{k=1}^{N} p_{k} F_{k}(\mathbf{w})\right\},
	\label{eq:problem}
\end{align}
where $N$ is the number of local devices (users/nodes/workers) and $p_k$ is the $k$-th device's weight satisfying $p_k \geq 0$ and $\sum_{k=1}^N p_k = 1$. 
In the $k$-th local device, there are $n_k$ data points:
$\vx_k^1, \vx_k^2, \dots, \vx_k^{n_k}$.  
The local objective $F_k(\cdot)$ is defined as:
$F_{k}(\mathbf{w}) \triangleq \frac{1}{n_{k}} \sum_{j=1}^{n_{k}} \ell\left(\mathbf{w} ; \vx_k^j\right)$,
where $\ell$ denotes a user-specified loss function. Each device only has access to its local data, which gives rise to its own local objective $F_k$. Note that we do not make any assumptions on the
data distributions of each local device. The local minimum
$F^{*}_k= \min_{\vw} F_k(\vw)$ can be far from
the global minimum of \eq{\ref{eq:problem}} (data heterogeneity). 

\subsection{The Federated Averaging (FedAvg) Algorithm}
We first introduce the standard Federated Averaging (FedAvg)  algorithm which was first proposed by~\cite{mcmahan2016communication}.
FedAvg updates the model in each device by local Stochastic Gradient Descent (SGD) and sends the latest model to the central server every $E$
steps. The central server conducts a weighted average over the model parameters
received from active devices and broadcasts the latest averaged model to all devices.
Formally, the updates of FedAvg at round $t$ is described as follows:
\begin{align}
\label{eq:fedavg updates}
\vv_{t+1}^{k} & =\vw_{t}^{k}-\alpha_{t}\vg_{t,k},\\
\mathbf{w}_{t+1}^{k} &=\left\{
\begin{array}{ll}
\mathbf{v}_{t+1}^{k}  & \text { if } t+1 \notin \mathcal{I}_{E}, \\ 
\sum_{k \in \cS_{t+1}} q_k \mathbf{v}_{t+1}^{k}  & \text { if } t+1 \in \mathcal{I}_{E},
\end{array}\right.
\end{align}
where $\vw_t^k$ is the local model parameter maintained in the $k$-th device at the $t$-th iteration and $\vg_{t,k}:=\nabla F_{k}(\vw_{t}^{k},\xi_{t}^{k})$ is the stochastic gradient based on $\xi_{t}^{k}$, the data point sampled from $k$-th device’s local data uniformly at random. $\mathcal{I}_{E}=\{E,2E,\dots\}$ is the set of global communication steps, when local parameters from a set of active devices are averaged and broadcast to all devices. We use $\mathcal{S}_{t+1}$ to represent the (random) set of active devices at $t+1$. $q_k$ is a set of averaging weights that are specific to the sampling procedure used to obtain the set of active devices $\mathcal{S}_{t+1}$.

Since federated learning usually involves an enormous amount of 
local devices, it is often more realistic to assume only a subset of 
local devices is active at each communication round (system heterogeneity). In this work,
we consider both the case of \textbf{full participation} where the model is
averaged over all devices at each communication round, in which case $q_k=p_k$ for all $k$ and 
$\mathbf{w}_{t+1}^{k} = \sum_{k=1}^N p_k \mathbf{v}_{t+1}^{k}$ if $t+1 \in \mathcal{I}_{E}$, and
the case of \textbf{partial participation} where $|\cS_{t+1}| < N$. 

With partial participation, we follow~\cite{li2018federated,karimireddy2019scaffold,li2019convergence} and assume that $\cS_{t+1}$ is obtained by one of two types of
sampling schemes to simulate practical scenarios. One scheme establishes $\cS_{t+1}$ by \emph{i.i.d.} sampling the devices with probability $p_k$ with replacement, and uses $q_k=\frac{1}{K}$, where $K=|\mathcal{S}_{t+1}|$, while the other scheme samples $\cS_{t+1}$ uniformly \emph{i.i.d.} from all devices without replacement, and uses $q_k=p_k\frac{N}{K}$.
Both schemes
guarantee that gradient updates in FedAvg are unbiased stochastic versions of
updates in FedAvg with full participation, which is important in the theoretical analysis of convergence. Because the original sampling scheme and weights proposed by \cite{mcmahan2016communication} lacks this desirable property, it is not considered in this paper. An interesting recent work \citep{chen2022optimal} proposes a new client selection procedure based on importance sampling that achieves better communication complexities than \emph{i.i.d.} sampling. For 
more details on the notations and setup as well as properties of the two sampling schemes, please refer to Appendix~\ref{sec:app:notations}.

\subsection{Assumptions}
\label{sec:assumptions}
We make the following standard assumptions on the objective function $F_1,\dots, F_N$. Assumptions~\ref{ass:lsmooth} and~\ref{ass:stroncvx} are commonly satisfied by a range of popular objective functions, such as $\ell^{2}$-regularized logistic regression and cross-entropy loss functions.
\begin{ass}[Smoothness]
	$F_{1}, \cdots, F_{N}$ are all $L$-smooth: for all  $\mathbf{v}, \mathbf{w}$, 
  \vspace{-0.5em}
 \begin{align*}
     F_{k}(\mathbf{v}) \leq F_{k}(\mathbf{w})+(\mathbf{v}- \mathbf{w})^{T} \nabla F_{k}(\mathbf{w})+\frac{L}{2}\|\mathbf{v}-\mathbf{w}\|_{2}^{2}.
 \end{align*}
  \vspace{-1.5em}
	\label{ass:lsmooth}
\end{ass}
\vspace{-1em}
\begin{ass}[Strong convexity]
	The local objectives $F_{1}, \cdots, F_{N}$ are $\mu$-strongly convex: for all $\mathbf{v}, \mathbf{w}$, 
 \vspace{-0.5em}
  \begin{align*}
 F_{k}(\mathbf{v}) \geq F_{k}(\mathbf{w})+(\mathbf{v}-\mathbf{w})^{T} \nabla F_{k}(\mathbf{w})+\frac{\mu}{2}\|\mathbf{v}-\mathbf{w}\|_{2}^{2}
  \end{align*}
	\label{ass:stroncvx}
 \vspace{-1.5em}
\end{ass}

\begin{ass}[Bounded local variance]
	Let $\xi_{t}^{k}$ be sampled from the $k$-th device's local data uniformly at random. The variance of stochastic gradients in each device is bounded: $\mathbb{E}\left\|\nabla F_{k}\left(\mathbf{w}_{t}^{k}, \xi_{t}^{k}\right)-\nabla F_{k}\left(\mathbf{w}_{t}^{k}\right)\right\|^{2} \leq \sigma_{k}^{2}$,
	for $k=1, \cdots, N$ and any $\mathbf{w}_{t}^{k}$. Let $\sigma^2:=\sum_{k=1}^{N}p_k\sigma_{k}^{2}$.
	\label{ass:boundedvariance}
\end{ass}
\begin{ass}[Bounded local gradient]
	The expected squared norm of stochastic gradients is uniformly bounded. i.e.,
	$\mathbb{E}\left\|\nabla F_{k}\left(\mathbf{w}_{t}^{k}, \xi_{t}^{k}\right)\right\|^{2} \leq G^{2}$, for all $k = 1,..., N$ and $t=0, \dots, T-1$.
	\label{ass:subgrad2}
\end{ass}
Assumptions~\ref{ass:boundedvariance} and~\ref{ass:subgrad2} have been used in many prior
works, e.g.,~\cite{yu2019parallel,li2019convergence,stich2018local,reddi2020adaptive}. \rebuttal{Some more recent works \citep{khaled2020tighter,karimireddy2019scaffold} have relaxed Assumption \ref{ass:subgrad2} to only requiring the bound at a minimizer $\vw^\ast \in \arg \min_{\vw}F(\vw)$ of the global objective instead of everywhere. This is to address the issue that for unconstrained optimization problems, gradients may not be bounded everywhere. For example, \cite{karimireddy2019scaffold} assume a bound similar to 
\begin{align*}
\sum_{k}p_{k}\mathbb{E}\|\nabla F_{k}(\vw,\xi_{t}^{k})\|^{2}\leq G^{2}+2\beta B^{2}(F(\vw)-F(\vw^{\ast}))
\end{align*}
While it is true that in unconstrained optimization, gradients can become unbounded, and the above bound is formally weaker than Assumption \ref{ass:subgrad2} due to the optimality gap in the upper bound, we argue that under convexity, $F$ can be assumed to be bounded due to the boundedness of local updates, so that Assumption \ref{ass:subgrad2} is not substantially stronger. First, if $F$ is itself bounded above,
then the bound above would imply the bound 
\begin{align*}
\sum_{k}p_{k}\mathbb{E}\|\nabla F_{k}(\vw,\xi_{t}^{k})\|^{2}\leq G'^{2}
\end{align*}
with a larger $G'$, which is 
essentially equivalent to Assumption \ref{ass:subgrad2}.
Even if $F$ is unbounded, convexity implies that in expectation, local parameters gravitate towards minima of local objectives during local updates, and thus stay bounded in a ball around $\vw^\ast$. Once local parameters
are bounded, $F$ can essentially be assumed to be bounded in that
region, yielding a bound $G$ in Assumption \ref{ass:subgrad2} that depends on the initialization and local objective functions. 
Therefore, for \emph{convex} problems, Assumption \ref{ass:subgrad2} is not  fundamentally more 
restrictive than assuming the bounds at $\mathbf{w}^{\ast}$ only. Furthermore, compared to assuming \textit{bounded gradient diversity} as in related works~\cite{haddadpour2019convergence,li2018federated}, Assumption~\ref{ass:subgrad2} is much less restrictive. When the optimality gap converges to zero,
bounded gradient diversity restricts local objectives to have the same minimizer as the global objective, contradicting the heterogeneous data setting. 
For detailed discussions of our assumptions, please refer to Appendix Section~\ref{sec:app:comparison}.}

\section{Linear Speedup Analysis of Federated Averaging}
\label{sec:sgd}

In this section, we provide convergence analyses of FedAvg for convex objectives in the general setting with both heterogeneous data (statistical heterogeneity) and partial
participation (system heterogeneity). We show that for strongly convex and smooth objectives,
the convergence of the optimality gap of averaged parameters across
devices is $\mathcal{O}(1/KT)$, while for convex and smooth
objectives, the rate is $\mathcal{O}(1/\sqrt{KT})$. Our results improve upon~\cite{li2019convergence} by showing linear speedup for any number of participating devices, and upon~\cite{khaled2020tighter,koloskova2020unified} by allowing system heterogeneity. The proofs also highlight similarities and distinctions between the strongly convex and convex settings. Detailed proofs are deferred to Appendix Section~\ref{sec:app:fedavg}.

\subsection{Strongly Convex and Smooth Objectives}

We first show that FedAvg has an $\mathcal{O}(1/KT)$ convergence rate
for $\mu$-strongly convex and $L$-smooth objectives. The result relies on a technical improvement over the analysis in~\cite{li2019convergence}. Moreover, it
implies a distinction in communication efficiency that guarantees
this linear speedup for FedAvg with full and partial device participation.
With full participation, $E$ can be chosen as large as $\mathcal{O}(\sqrt{T/N})$
without degrading the linear speedup in the number of workers. On
the other hand, with partial participation, $E$ must be $\mathcal{O}(1)$
to guarantee $\mathcal{O}(1/KT)$ convergence.
\begin{theorem}
	\label{thm:SGD_scvx}Let $\overline{\mathbf{w}}_{T}=\sum_{k=1}^{N}p_{k}\mathbf{w}_{T}^{k}$ in FedAvg,
	$\nu_{\max}=\max_{k}Np_{k}$, and set decaying learning rates $\alpha_{t}=\frac{4}{\mu(\gamma+t)}$
	with $\gamma=\max\{32\kappa,E\}$ and $\kappa=\frac{L}{\mu}$. Then
	under Assumptions~\ref{ass:lsmooth} to \ref{ass:subgrad2} with full device participation, 
	\begin{align*}
	\mathbb{E}F(\overline{\mathbf{w}}_{T})-F^{\ast}=\mathcal{O}\left(\frac{\kappa\rebuttal{\nu_{\max}}\sigma^{2}/\mu}{NT}+\frac{\kappa^{2}E^{2}G^{2}/\mu}{T^{2}}\right),
	\end{align*}
	and with partial device participation with at least $K$ sampled devices
	at each communication round, 
	\begin{align*}
	\mathbb{E}F(\overline{\mathbf{w}}_{T})-F^{\ast}=\mathcal{O}\left(\frac{\kappa \rebuttal{E} G^{2}/\mu}{KT}+\frac{\kappa\rebuttal{\nu_{\max}}\sigma^{2}/\mu}{NT}+\frac{\kappa^{2}E^{2}G^{2}/\mu}{T^{2}}\right).
	\end{align*}
	\label{th:scvx_sgd}
\end{theorem}
\textbf{Proof sketch.} Because our unified analyses of results in the main text follow the same framework with variations in technical details, we first give an outline of proof  for Theorem~\ref{th:scvx_sgd} to illustrate the main ideas. For full participation, the main ingredient is a recursive contraction bound 
	\begin{align*}
	\mathbb{E}\|\ov{w}_{t+1}-\vw^{\ast}\|^{2} &  \leq(1-\mu\alpha_{t})\mathbb{E}\|\ov{w}_{t}-\vw^{\ast}\|^{2}\\
 &+\alpha_{t}^{2}\frac{1}{N}\nu_{\max}\sigma^{2}+6\alpha_{t}^{3}LE^{2}G^{2}
	\end{align*}
 where the $\mathcal{O}(\alpha_{t}^{3}E^{2}G^{2})$ term is the key improvement over the bound in~\cite{li2019convergence}, which has $\mathcal{O}(\alpha_{t}^{2}E^2 G^2)$ instead. We then use induction to obtain a non-recursive bound on $\mathbb{E}\|\ov{w}_{T}-\vw^{\ast}\|^{2}$, which is converted to a bound on $\mathbb{E}F(\overline{\mathbf{w}}_{T})-F^{\ast}$ using $L$-smoothness. For partial participation, an additional term $\mathcal{O}(\frac{1}{K} \alpha_t^2 E^2 G^2)$ of leading order resulting from sampling variance is added to the contraction bound, but only every $E$ steps. To facilitate the understanding of our analysis, please refer to a high-level summary in Appendix~\ref{sec:app:sum}.

\textbf{Linear speedup. }We compare our bound with that in \cite{li2019convergence},
which is $\mathcal{O}(\frac{1}{NT}+\frac{E^{2}}{KT}+\frac{E^{2}G^{2}}{T})$.
Because the term $\frac{E^{2}G^{2}}{T}$ is also $\mathcal{O}(1/T)$
without a dependence on $N$, for any choice of $E$ their bound cannot
achieve linear speedup. The improvement of our bound comes from the
term $\frac{\kappa^{2}E^{2}G^{2}/\mu}{T^{2}}$, which now is $\mathcal{O}(E^{2}/T^{2})$ and so is not of leading order. As a result, all leading terms scale with $1/N$ in the full device
participation setting, and with $1/K$ in the partial participation
setting. This implies that in both settings, there is a \emph{linear
	speedup} in the number of active workers during a communication
round. We also emphasize that the reason one cannot recover the full participation bound by setting $K=N$ in the partial participation bound is due to the variance generated by sampling.
 
\rebuttal{\textbf{Discussion on $\nu_{\max}$.} The parameter $\nu_{\max}$ is a measure of how unbalanced different local servers are, and is also discussed in \cite{li2019convergence}. Recall that $\nu_{\max}=N\max_k p_k$, where $p_k$ is the weight of the local objective of server $k$ in the FL objective. Often, $p_k$ is the proportion of data stored on server $k$ relative to the total amount of data across all servers, and is therefore small, i.e., $\mathcal{O}(1/N)$. This is a reasonable assumption in many FL applications, e.g., mobile computing. In this case, $\nu_{\max}=\mathcal{O}(1)$ and linear speedup in the number of local servers is guaranteed. However, when some local servers dominate the FL objective, i.e., $\max_k p_k = \Theta(1)$, those local servers will become bottlenecks in the convergence of FedAvg, and linear speedup is not guaranteed. This is already observed in \cite{li2019convergence}, where the convergence of FedAvg is shown empirically to slow down significantly when $\nu_{\max}$  is large. Thus, linear speedup convergence depends on the balance parameter $\nu_{\max}$, and is only guaranteed when $\nu_{\max}=\mathcal{O}(1)$. }

\textbf{Communication Complexity.} Our bound implies a distinction
in the choice of $E$ between the full and partial participation settings.
{With full participation, the term involving $E$, $\mathcal{O}(E^{2}/T^{2})$, is not of leading order $\mathcal{O}(1/T)$, so we can increase $E$ and reduce the number of communication rounds without degrading the linear speedup in iteration complexity $\mathcal{O}(1/NT)$, as long as $E=\mathcal{O}(\sqrt{T/N})$, since then $\mathcal{O}(E^{2}/T^{2})=\mathcal{O}(1/NT)$
matches the leading term. This corresponds to a communication complexity
of $T/E=\mathcal{O}(\sqrt{NT})$.} In contrast, the bound in \cite{li2019convergence}
does not allow $E$ to scale with $\sqrt{T}$ to preserve $\mathcal{O}(1/T)$
rate, even for full participation. On the other hand, with partial
participation, $\frac{\kappa EG^{2}/\mu}{KT}$ is also a leading
term, and so $E$ must be $\mathcal{O}(1)$. In this case, our bound
still yields a linear speedup in $K$, which is also confirmed by
experiments. The requirement that $E=\mathcal{O}(1)$ in order to achieve linear speedup in partial participation cannot be removed for our sampling schemes, as the term $\frac{\kappa EG^{2}/\mu}{KT}$ comes from variance in the sampling process. 

\textbf{Comparison with related works.} To better understand the significance of the obtained bound, we compare our rates to the best-known results in related settings. \cite{haddadpour2019convergence} prove a linear speedup $\mathcal{O}(1/KT)$ result for strongly convex and smooth objectives\footnote{Their result applies to a larger class of non-convex objectives that satisfy the Polyak-Lojasiewicz condition.}, with $\mathcal{O}(K^{1/3}T^{2/3})$ communication complexity with non-\emph{i.i.d.} data and partial participation. However, their results build on the bounded gradient diversity assumption, which implies the existence of $\mathbf{w}^*$ that minimizes all local objectives (see discussions in Section~\ref{sec:assumptions} and Appendix~\ref{sec:app:comparison}), effectively removing statistical heterogeneity. \rebuttal{The bound in \cite{koloskova2020unified} matches our bound in the full participation case, but their framework excludes partial participation \cite[Proposition 1]{koloskova2020unified}. ~\cite{karimireddy2019scaffold} consider both types of heterogeneities for FL and establish linear speedup using a modified version of FedAvg with distinct learning rates for local steps and communication rounds that are $\mathcal{O}(1/T)$. 
In contrast, our linear speedup result is for the standard FedAvg that does not use different learning rates for local and aggregation steps. Moreover, our learning rate decays with the iteration number, and is thus generally larger in practice. When there is no data heterogeneity, i.e. in the classical distributed optimization paradigm, communication complexity can be further improved, e.g. \cite{woodworth2020local,woodworth2020minibatch}, but such results are not directly comparable to ours since we consider the setting where individual devices have access to different datasets. \cite{yang2021achieving} obtain linear speedup results under both data and system heterogeneity for non-convex problems, so can be viewed as complementary results.}
\subsection{Convex Smooth Objectives}
Next we provide linear speedup analysis of FedAvg with convex and
smooth objectives and show that the optimality gap is $\mathcal{O}(1/\sqrt{KT})$. 
This result complements the strongly convex case in the previous part, as well as the non-convex
smooth setting in \cite{jiang2018linear,yu2019parallel,haddadpour2019convergence},
where $\mathcal{O}(1/\sqrt{KT})$ results are given in terms of averaged
gradient norm, and it also extends the result in~\cite{khaled2020tighter}, which has the best linear speedup result in the convex setting with full participation.
\begin{theorem}
	\label{thm:SGD_cvx}Under Assumptions~\ref{ass:lsmooth},\ref{ass:boundedvariance},\ref{ass:subgrad2} and constant learning
	rate $\alpha_{t}=\mathcal{O}(\sqrt{\frac{N}{T}})$, FedAvg satisfies
	\begin{align*}
	\min_{t\leq T}F(\overline{\mathbf{w}}_{t})-F(\mathbf{w}^{\ast}) & =\mathcal{O}\left(\frac{\rebuttal{\nu_{\max}}\sigma^{2}}{\sqrt{NT}}+\frac{NE^{2}LG^{2}}{T}\right)
	\end{align*}
	with full participation, and with partial device participation with $K$ sampled devices at
	each communication round and learning rate $\alpha_{t}=\mathcal{O}(\sqrt{\frac{K}{T}})$,
	\begin{align*}
	\min_{t\leq T}F(\overline{\mathbf{w}}_{t})-F(\mathbf{w}^{\ast}) & =\mathcal{O}\left(\frac{\rebuttal{\nu_{\max}}\sigma^{2}}{\sqrt{KT}}+\frac{\rebuttal{E}G^{2}}{\sqrt{KT}}+\frac{KE^{2}LG^{2}}{T}\right).
	\end{align*}
\end{theorem}
The analysis again relies on a recursive bound, but without contraction: 
\begin{align*}
	&\mathbb{E}\|\ov{w}_{t+1}-\vw^{\ast}\|^{2}+\alpha_{t}(F(\ov{w}_{t})-F(\vw^{\ast})) \\ 
  \leq &\mathbb{E}\|\ov{w}_{t}-\vw^{\ast}\|^{2}+\alpha_{t}^{2}\frac{1}{N}\nu_{\max}\sigma^{2}+6\alpha_{t}^{3}E^{2}LG^{2}
	\end{align*}
which is then summed over time steps to give the desired bound, with $\alpha_{t}=\mathcal{O}(\sqrt{\frac{N}{T}})$. \\
\textbf{Choice of $E$ and linear speedup. }With full participation,
as long as $E=\mathcal{O}(T^{1/4}/N^{3/4})$, the convergence
rate is $\mathcal{O}(1/\sqrt{NT})$ with $\mathcal{O}(N^{3/4}T^{3/4})$
communication rounds. In the partial participation setting, $E$ must
be $O(1)$ in order to achieve linear speedup of $\mathcal{O}(1/\sqrt{KT})$. {This is again due to the fact that the sampling variance $\mathbb{E}\|\ov{w}_t-\ov{v}_t\|^2=\mathcal{O}(\alpha_t^2 E^2G^2)$ cannot be made independent of $E$, as illustrated by Proposition \ref{prop:tight}. See also the proof in Section \ref{sec:app:fedavg} for how the sampling variance and the term $EG^2/\sqrt{KT}$ are related.}  Our result again demonstrates the difference in communication complexities
between full and partial participation. 

\section{Linear Speedup Analysis of Nesterov Accelerated Federated Averaging}
\label{sec:Nesterov}

{A natural extension of the FedAvg algorithm is to use momentum-based 
local updates instead of local SGD updates in order to accelerate FedAvg. As we know from standard stochastic optimization settings, Nesterov and other momentum
updates fail to provably accelerate over SGD in general \citep{liu2018accelerating,kidambi2018insufficiency,yuan2020federated}. This is in contrast to the classical acceleration result of Nesterov-accelerated gradient descent over GD. See, however, \cite{jain2017accelerating,even2021continuized} for acceleration results for quadratic objectives. Thus in the FL setting, the best provable convergence rate, in terms of dependence on $T$, for FedAvg with Nesterov updates is the same as FedAvg with SGD updates. Nevertheless, Nesterov and other momentum updates are frequently used in practice, in both non-FL and FL settings, and are observed to perform better empirically. In fact, previous works such as \cite{stich2018local} use FedAvg with Nesterov or other momentum updates in their experiments to achieve target accuracy. Because of the popularity of Nesterov and other momentum-based methods, understanding the linear speedup behavior of FedAvg with momentum updates is important.} \rebuttal{In addition, the communication complexity required to guarantee such a linear speedup convergence is also a relevant question with practical implications. To our knowledge, the majority of convergence analyses of FedAvg with momentum-based stochastic
updates focus on the non-convex smooth case \citep{huo2020faster,yu2019linear,li2018federated}. In convex smooth settings, the results of \cite{even2021continuized}  can be adapted to prove acceleration, in terms of dependence on $T$, of Nesterov FedAvg with full participation for \emph{quadratic} objectives. The work of \cite{yang2022federated} establishes a $\mathcal{O}(1/\sqrt{T})$ rate for Nesterov FedAvg for general convex smooth objectives under full participation. However, their convergence result does not have linear speedup in the number of participating servers.} In this section, we complete the picture by providing the first $\mathcal{O}(1/KT)$
and $\mathcal{O}(1/\sqrt{KT})$ convergence results for Nesterov-accelerated
FedAvg for general convex objectives that match the rates for FedAvg with SGD updates. Detailed proofs of convergence results in this section are deferred to Appendix Section~\ref{sec:app:Nesterovfedavg}.

\subsection{Strongly Convex and Smooth Objectives}
The Nesterov Accelerated Federated Averaging algorithm (Nesterov FedAvg) follows the updates:
\begin{align*} 
\mathbf{v}_{t+1}^{k} & =\mathbf{w}_{t}^{k}-\alpha_{t}\mathbf{g}_{t,k}, \\
\mathbf{w}_{t+1}^{k} & =\begin{cases}
\mathbf{v}_{t+1}^{k}+\beta_{t}(\mathbf{v}_{t+1}^{k}-\mathbf{v}_{t}^{k}) & \text{if }t+1\notin\mathcal{I}_{E},\\
\sum_{k \in \cS_{t+1}}q_k\left[\mathbf{v}_{t+1}^{k}+\beta_{t}(\mathbf{v}_{t+1}^{k}-\mathbf{v}_{t}^{k})\right] & \text{if }t+1\in\mathcal{I}_{E},
\end{cases}
\end{align*}
where $\mathbf{g}_{t,k}:=\nabla F_{k}(\mathbf{w}_{t}^{k},\xi_{t}^{k})$ is
the stochastic gradient sampled on the $k$-th device at time $t$, and $q_k$ again depends on participation and sampling schemes.  
\begin{theorem}
	\label{thm:nesterov_scvx}Let $\overline{\mathbf{v}}_{T}=\sum_{k=1}^{N}p_{k}\mathbf{v}_{T}^{k}$ in Nesterov accelerated FedAvg,
	and set learning rates $\alpha_{t}=\frac{6}{\mu}\frac{1}{t+\gamma}$,  $\beta_{t-1}=\frac{3}{14(t+\gamma)(1-\frac{6}{t+\gamma})\max\{\mu,1\}}$. Then under Assumptions~\ref{ass:lsmooth},\ref{ass:stroncvx},\ref{ass:boundedvariance},\ref{ass:subgrad2} with full device participation, 
	\small{\begin{align*}
	\mathbb{E}F(\overline{\mathbf{v}}_{T})-F^{\ast}=\mathcal{O}\left(\frac{\kappa\rebuttal{\nu_{\max}}\sigma^{2}/\mu}{NT}+\frac{\kappa^{2}E^{2}G^{2}/\mu}{T^{2}}\right),
	\end{align*}}
	and with partial device participation with $K$ sampled devices at
	each communication round, 
	\small{
	\begin{align*}
	\mathbb{E}F(\overline{\mathbf{v}}_{T})-F^{\ast}=\mathcal{O}\left(\frac{\kappa\rebuttal{\nu_{\max}}\sigma^{2}/\mu}{NT}+\frac{\kappa \rebuttal{E}G^{2}/\mu}{KT}+\frac{\kappa^{2}E^{2}G^{2}/\mu}{T^{2}}\right).
	\end{align*}}
\end{theorem}
Similar to FedAvg, the key step in the proof of this result is a recursive contraction bound, but different in that it involves three time steps, due to the update format of Nesterov SGD (see Lemma~\ref{lem:nest-scvxoner} in Appendix~\ref{sec:convexsmoothsgd}). 
Then we can again use induction and $L$-smoothness to obtain the desired bound.
To our knowledge, this is the first convergence result for Nesterov
accelerated FedAvg in the strongly convex and smooth setting. The
same discussion about linear speedup of FedAvg applies to the Nesterov
accelerated variant. In particular, to achieve $\mathcal{O}(1/NT)$
linear speedup, $T$ iterations of the algorithm require only $\mathcal{O}(\sqrt{NT})$
communication rounds with full participation. 

\rebuttal{To our knowledge, this is the first work that establishes linear speedup convergence of Nesterov-accelerated FedAvg in the convex setting under both non-\emph{i.i.d.} data and partial participation. Recently, there have been significant efforts to develop 
novel acceleration algorithms for Federated Learning. A notable work among these is \cite{yuan2020federated}, which developed a new momentum-accelerated variant of FedAvg called \texttt{FedAc}, based on the generalized accelerated SGD of \cite{ghadimi2012optimal}. They provided linear speedup convergence rates under \emph{full} participation that match our $\mathcal{O}(1/KT)$
and $\mathcal{O}(1/\sqrt{KT})$ complexities in the leading terms, but with an improved dependence on $E$ in the non-leading terms. This results in a better communication complexity that guarantees linear speedup. However, this improvement is only present in the \emph{full} participation setting. Under partial participation, the sampling variance dominates the convergence, resulting in the same communication complexity requirements for Nesterov-accelerated FedAvg and \texttt{FedAc} in order to achieve linear speedup.}

\subsection{Convex Smooth Objectives}

We now show that the optimality gap of Nesterov-accelerated FedAvg has $\mathcal{O}(1/\sqrt{KT})$ rate for convex and smooth objectives. This result complements the strongly convex case in the previous
part, as well as the non-convex smooth setting in \cite{huo2020faster,yu2019linear,li2018federated},
where a similar $\mathcal{O}(1/\sqrt{KT})$ rate is given in terms
of averaged gradient norm. \rebuttal{ A later work by \citet{yang2022federated}  establishes an $\mathcal{O}(1/\sqrt{T})$ rate in the convex smooth setting for FedAvg with Nesterov updates, but only under full participation and without linear speedup in $N$. In contrast, we establish linear speedup convergence for both the full and partial participation settings.}
\begin{theorem}
	\label{thm:Nesterov_cvx}Set learning rates $\alpha_{t}=\beta_{t}=\mathcal{O}(\sqrt{\frac{N}{T}})$. Then under Assumptions~\ref{ass:lsmooth},\ref{ass:boundedvariance},\ref{ass:subgrad2} Nesterov accelerated FedAvg with
	full device participation has rate
	\small{\begin{align*}
		\min_{t\leq T}F(\overline{\mathbf{v}}_{t})-F^{\ast} & =\mathcal{O}\left(\frac{\rebuttal{\nu_{\max}}\sigma^{2}}{\sqrt{NT}}+\frac{NE^{2}LG^{2}}{T}\right),
		\end{align*}}
	and with partial device participation with $K$ sampled devices at
	each communication round and learning rates $\alpha_{t}=\beta_{t}=\mathcal{O}(\sqrt{\frac{K}{T}})$,
	\small{\begin{align*}
		\min_{t\leq T}F(\overline{\mathbf{v}}_{t})-F^{\ast} & =\mathcal{O}\left(\frac{\rebuttal{\nu_{\max}}\sigma^{2}}{\sqrt{KT}}+\frac{\rebuttal{E}G^{2}}{\sqrt{KT}}+\frac{KE^{2}LG^{2}}{T}\right).
		\end{align*}}
\end{theorem}

\rebuttal{We emphasize again that in the \emph{stochastic} optimization setting with general objectives, the optimal convergence rate that FedAvg with Nesterov udpates can achieve is the same as FedAvg with SGD updates. When objectives are quadratic, \cite{jain2017accelerating,even2021continuized} provide acceleration results for Nesterov SGD in the centralized and decentralized settings, but acceleration with Nesterov is impossible in general. Nevertheless, due to the popularity and superior performance of momentum methods in practice, it is still important to understand the linear speedup behavior of such FedAvg variants. Our results in this section fill exactly this gap, and is to our knowledge the first work to establish such results.} 

\section{Geometric Convergence of FedAvg in Overparameterized Settings}
\label{sec:app:overparameterized}

Overparameterization is a prevalent machine learning setting where the
statistical model has much more parameters than the number of training samples
and the existence of parameter choices with zero training loss is
ensured~\citep{allen2018convergence,zhang2016understanding}. This is also called the interpolating regime. Due to the
property of \textit{automatic variance reduction} in the overparameterized setting, a
line of recent works have proved that SGD and accelerated methods achieve geometric
convergence~\citep{ma2017power,moulines2011non,needell2014stochastic,schmidt2013fast,strohmer2009randomized}.
A natural question is whether such a result still holds in the Federated
Learning setting. 
In this section, we establish the geometric convergence of FedAvg for overparameterized strongly convex and smooth problems,
and show that it preserves linear speedup at the same time. We then sharpen this result in the special case of
linear regression.
Detailed proofs are deferred to Section~\ref{sec:interpolation}. In particular, we do not need Assumptions~\ref{ass:boundedvariance} and~\ref{ass:subgrad2} and use modified versions of Assumptions~\ref{ass:lsmooth} and~\ref{ass:stroncvx} detailed in this section.

\subsection{Geometric Convergence of FedAvg in the Overparameterized Setting}
Recall the FL problem $\min_{w}\sum_{k=1}^{N}p_{k}F_{k}(\mathbf{w})$
with $F_{k}(\mathbf{w})=\frac{1}{n_{k}}\sum_{j=1}^{n_{k}}\ell(\mathbf{w};\mathbf{x}_{k}^{j})$.
In this section, we consider the standard Empirical Risk Minimization (ERM) setting where $\ell$
is non-negative, $l$-smooth, and convex, and as before, each $F_{k}(\mathbf{w})$ is $L$-smooth and $\mu$-strongly convex. Note that $l\geq L$. This
setup includes many important problems in practice. In the overparameterized
setting, there exists $\mathbf{w}^{\ast}\in\arg\min_{w}\sum_{k=1}^{N}p_{k}F_{k}(\mathbf{w})$
such that $\ell(\mathbf{w}^{\ast};\mathbf{x}_{k}^{j})=0$ for all
$\mathbf{x}_{k}^{j}$. We first show that FedAvg achieves geometric convergence
with linear speedup in the number of workers. 
\begin{theorem}
	\label{thm:overparameterized_general}In the overparameterized setting with full participation,
	FedAvg with communication every $E$ iterations
	and constant step size $\overline{\alpha}=\mathcal{O}(\frac{1}{E}\frac{N}{l\nu_{\max}+L(N-\nu_{\min})})$
	has geometric convergence:
	\begin{align*}
	&\mathbb{E}F(\overline{\mathbf{w}}_{T})  \leq\frac{L}{2}(1-\overline{\alpha})^{T}\|\mathbf{w}_{0}-\mathbf{w}^{\ast}\|^{2}\\
 =&\mathcal{O}\left(L\exp\left(-\frac{\mu}{E}\frac{NT}{l\nu_{\max}+L(N-\nu_{\min})}\right)\cdot\|\mathbf{w}_{0}-\mathbf{w}^{\ast}\|^{2}\right).
	\end{align*}
\end{theorem}
\textbf{Linear speedup and Communication Complexity} The linear speedup factor is on the order of $\mathcal{O}(N/E)$ 
for $N\leq\mathcal{O}(\frac{l}{L})$, i.e. FedAvg with $N$ workers and communication
every $E$ iterations provides a geometric convergence speedup factor
of $\mathcal{O}(N/E)$, for $N\leq\mathcal{O}(\frac{l}{L})$. In this regime, the convergence rate is of order $\mathcal{O}(\exp(-\frac{N}{E\kappa}T))$ where $\kappa=\frac{l}{\mu}$ is the ``condition number'' of the objective. When $N$ is above
this threshold, however, the speedup is almost constant in the number
of workers. This matches the findings in \cite{ma2017power}. Our
result also illustrates that $E$ can be taken $\mathcal{O}(T^{\beta})$
for any $\beta<1$ to achieve geometric convergence, achieving better communication efficiency than
the standard FL setting. We emphasize again that compared to the single-server results in~\cite{ma2017power}, the difference of our result lies in the factor of $N$ in the speedup, which cannot be obtained if one simply applied the single-server result to each device in our problem. Recently, \cite{qin2022faster} provided a convergence bound independent of the number $E$ of local steps by using a larger learning rate and under a strong growth condition on the local gradients. However, their bound does not exhibit linear speedup in the number of local servers.

\rebuttal{We also remark that similar geometric convergence results have been established for decentralized SGD (also called gossip averaging), which only allows communications between connected local servers, where the network topology is given by a possibly time-varying graph. This setting includes FedAvg with full participation (local SGD) as a special case. See the work of \cite{koloskova2020unified} for details. In comparison, our convergence result includes a linear speedup term when $N\leq \mathcal{O}(\frac{l}{L})$.}

\subsection{Overparameterized Linear Regression Problems}

We now turn to quadratic problems and show that the bound in Theorem~\ref{thm:overparameterized_general} can be improved to $\mathcal{O}(\exp(-\frac{N}{E\kappa_{1}}T))$ for a larger range of $N$. The local device objectives are now given by the sum of squares {\small$F_{k}(\mathbf{w})=\frac{1}{2n_{k}}\sum_{j=1}^{n_{k}}(\mathbf{w}^{T}\mathbf{x}_{k}^{j}-z_{k}^{j})^{2}$},
and there exists $\mathbf{w}^{\ast}$ such that $F(\mathbf{w}^{\ast})\equiv0$. A notion of condition number is important in our result: $\kappa_1$ which is based on local Hessians~\citep{liu2018accelerating}. See Section \ref{sec:interpolation} for a detailed definition of $\kappa_1$. The larger range of $N$ for which linear speedup holds is due to $\kappa_1 > \kappa$ where $\kappa$ is the condition number used in Theorem \ref{thm:overparameterized_general}.

\begin{theorem}
	\label{thm:overparameterized_quadratic}For the overparamterized linear regression problem with full participation, FedAvg with communication every $E$
	iterations with constant step size $\overline{\alpha}=\mathcal{O}(\frac{1}{E}\frac{N}{l\nu_{\max}+\mu(N-\nu_{\min})})$
	has geometric convergence:
	\begin{align*}
	\mathbb{E}F(\overline{\mathbf{w}}_{T}) & \leq\mathcal{O}\left(L\exp(-\frac{NT}{E(\nu_{\max}\kappa_{1}+(N-\nu_{\min}))})\|\mathbf{w}_{0}-\mathbf{w}^{\ast}\|^{2}\right).
	\end{align*}
\end{theorem}

When $N=\mathcal{O}(\kappa_{1})$, the convergence rate is $\mathcal{O}((1-\frac{N}{E\kappa_{1}})^{T})=\mathcal{O}(\exp(-\frac{NT}{E\kappa_{1}}))$,
which exhibits linear speedup in the number of workers, as well as
a $1/\kappa_{1}$ dependence on the condition number $\kappa_{1}$.

\begin{figure*}[ht!]
\centering
{\small
\begin{tabular}{ccc}
\hspace{-2em}\includegraphics[width=0.33\textwidth]{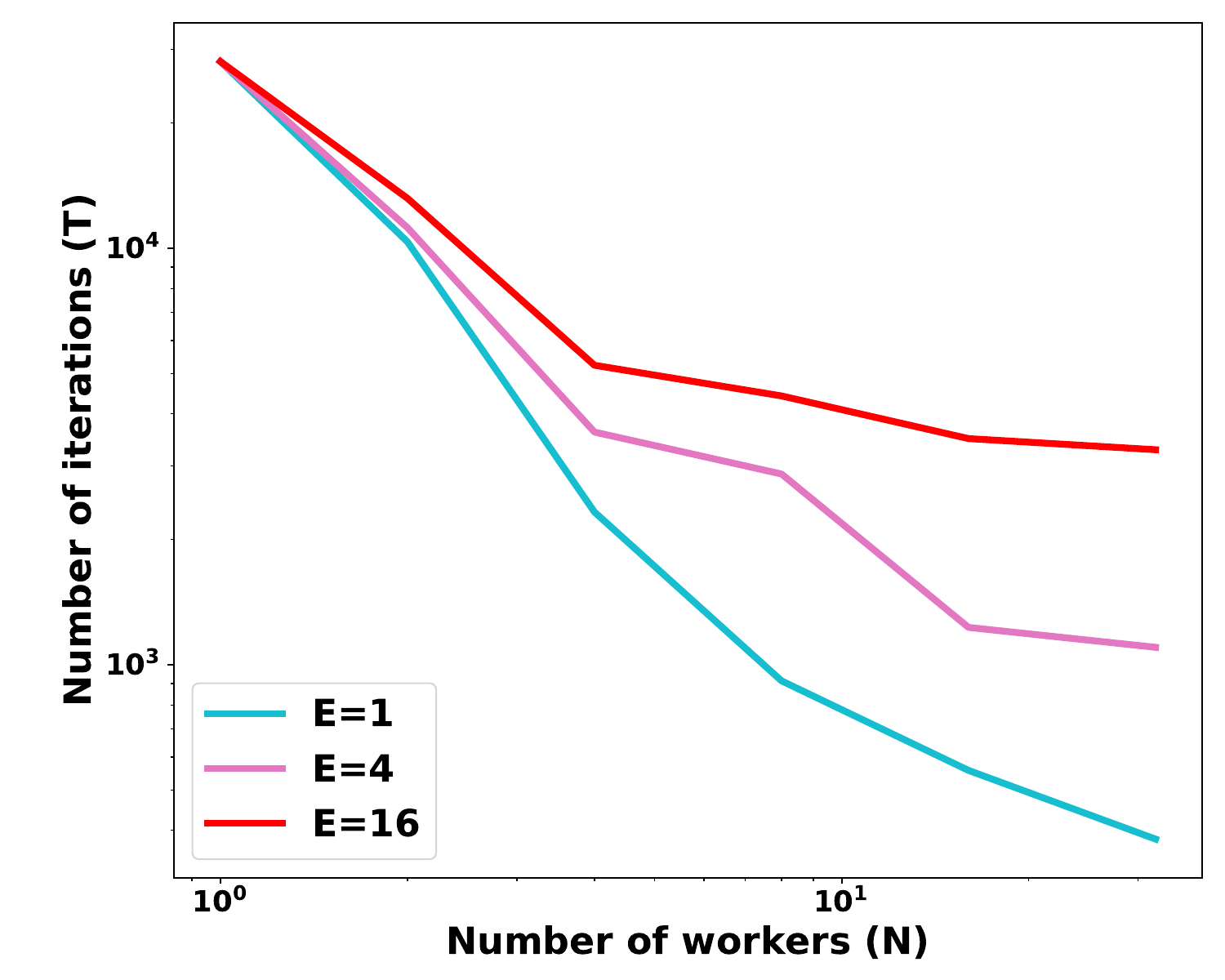} &
\includegraphics[width=0.33\textwidth]{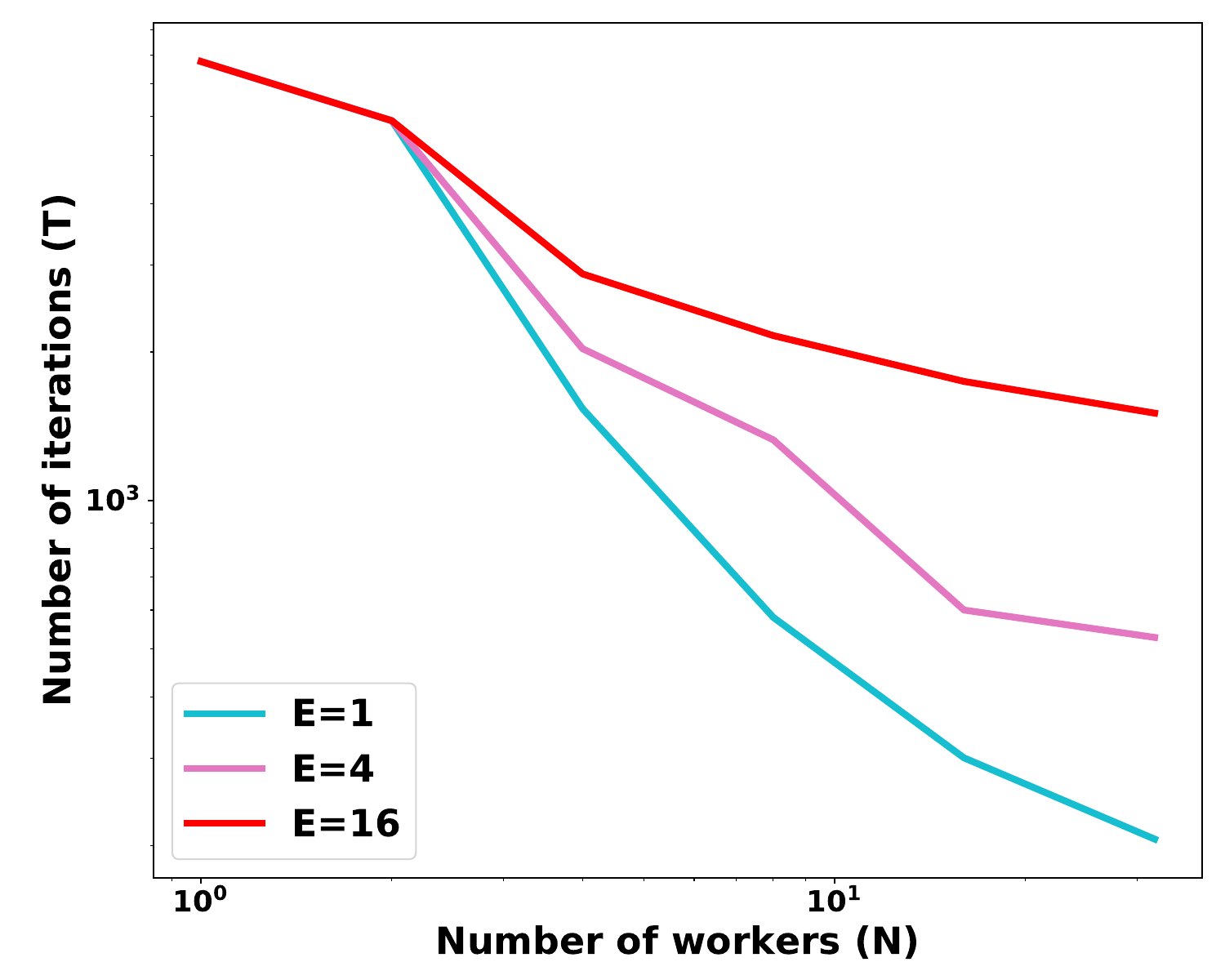} &
\includegraphics[width=0.33\textwidth]{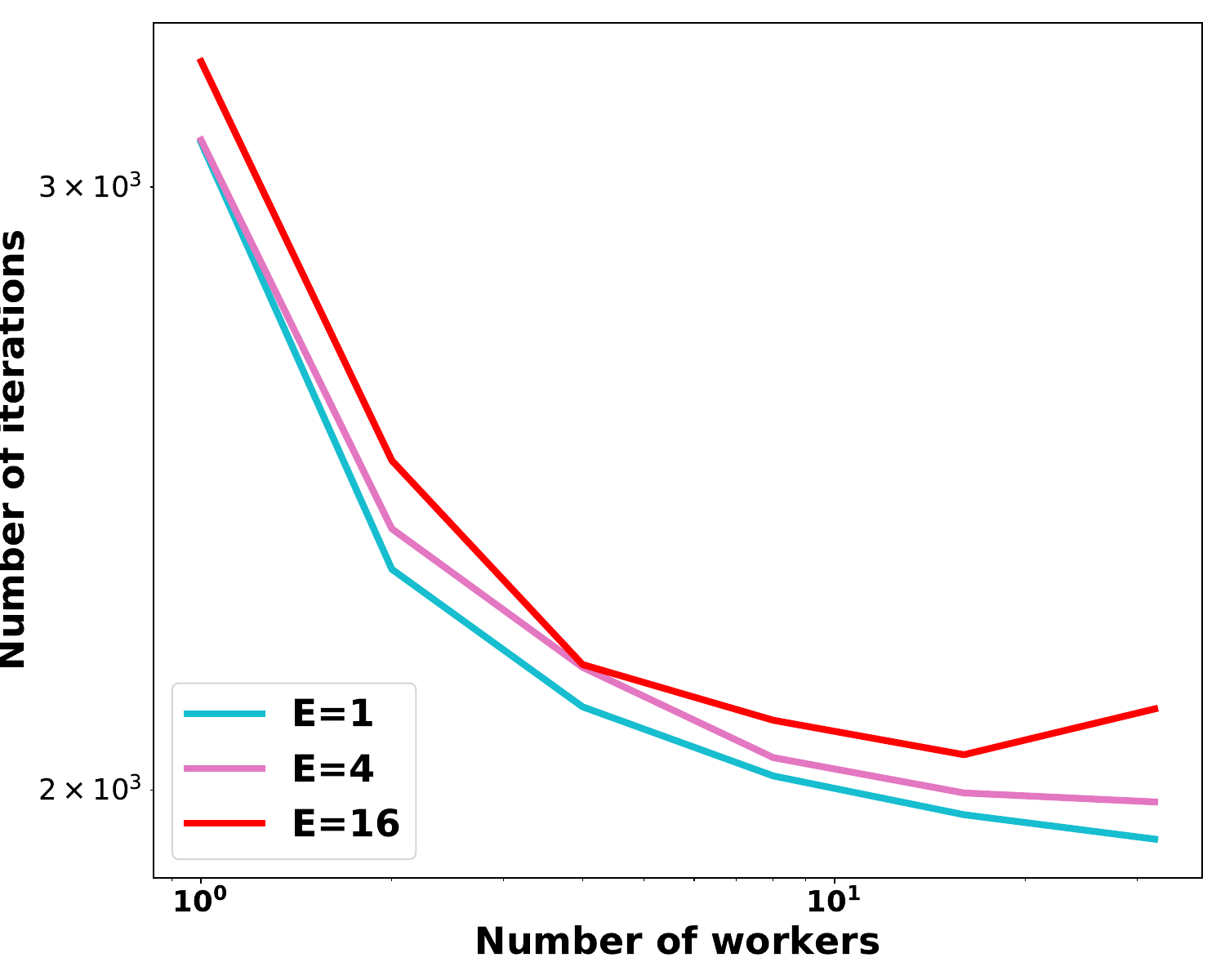}\\
\hspace{-2em}\includegraphics[width=0.33\textwidth]{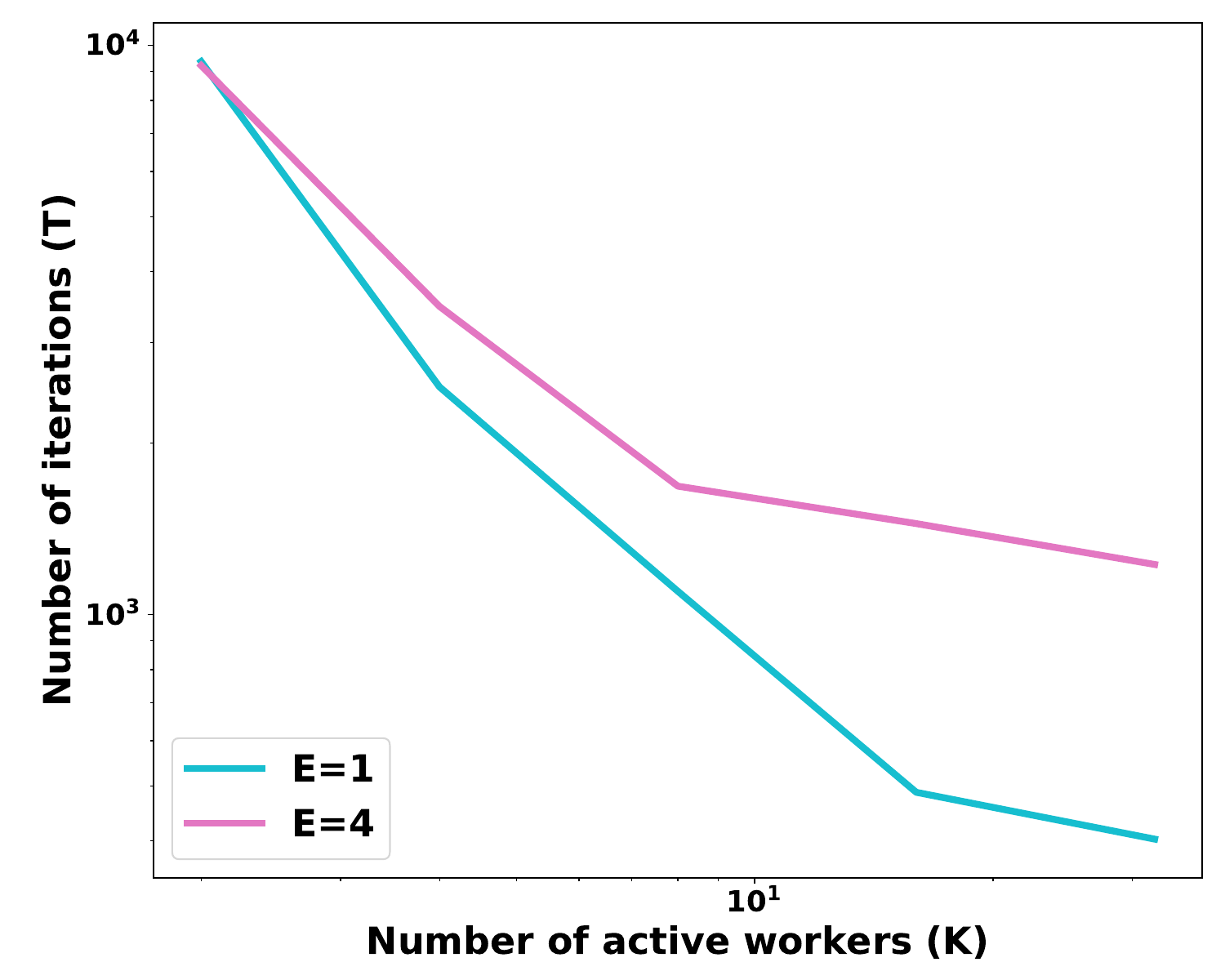} &
\includegraphics[width=0.33\textwidth]{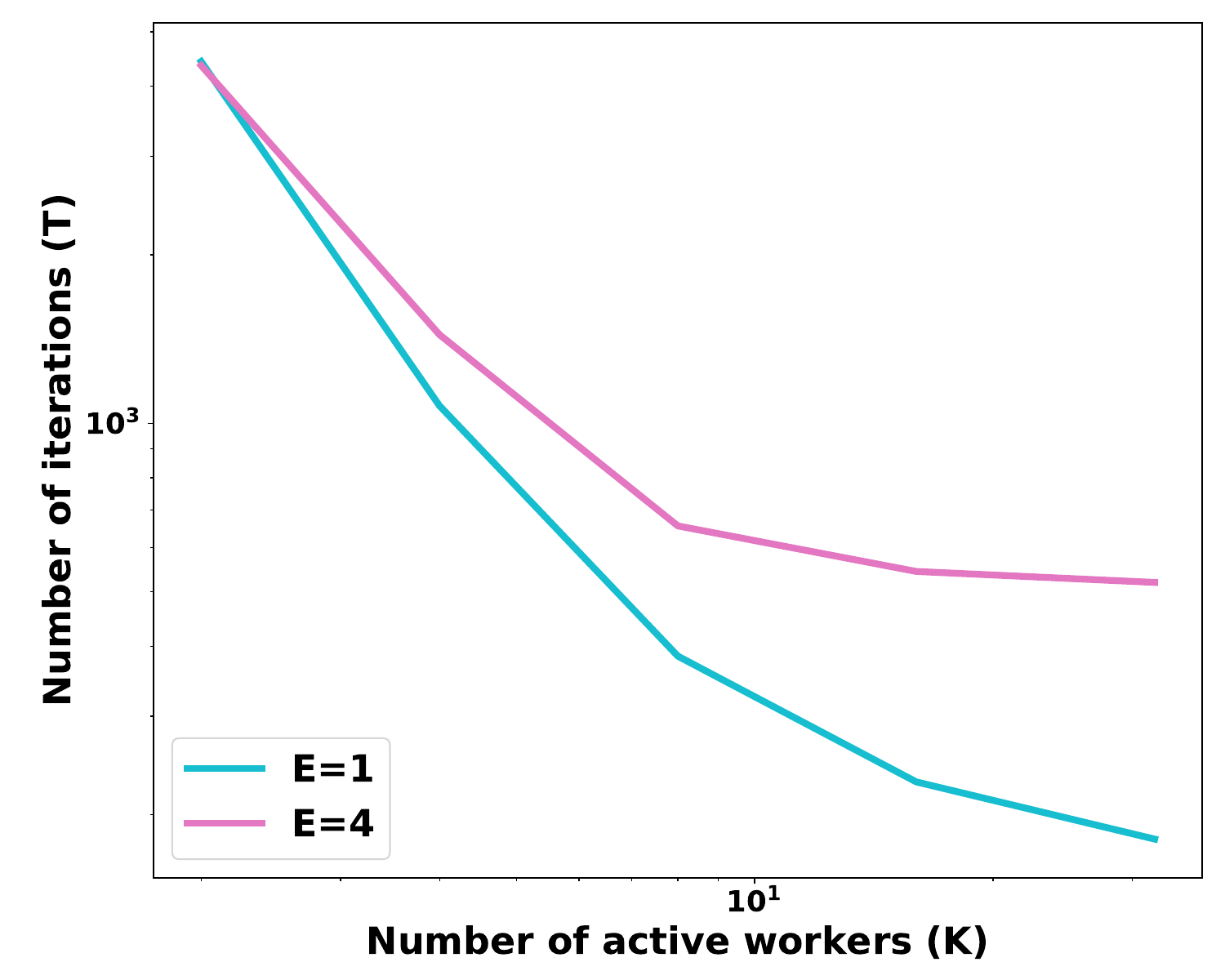} &
\includegraphics[width=0.33\textwidth]{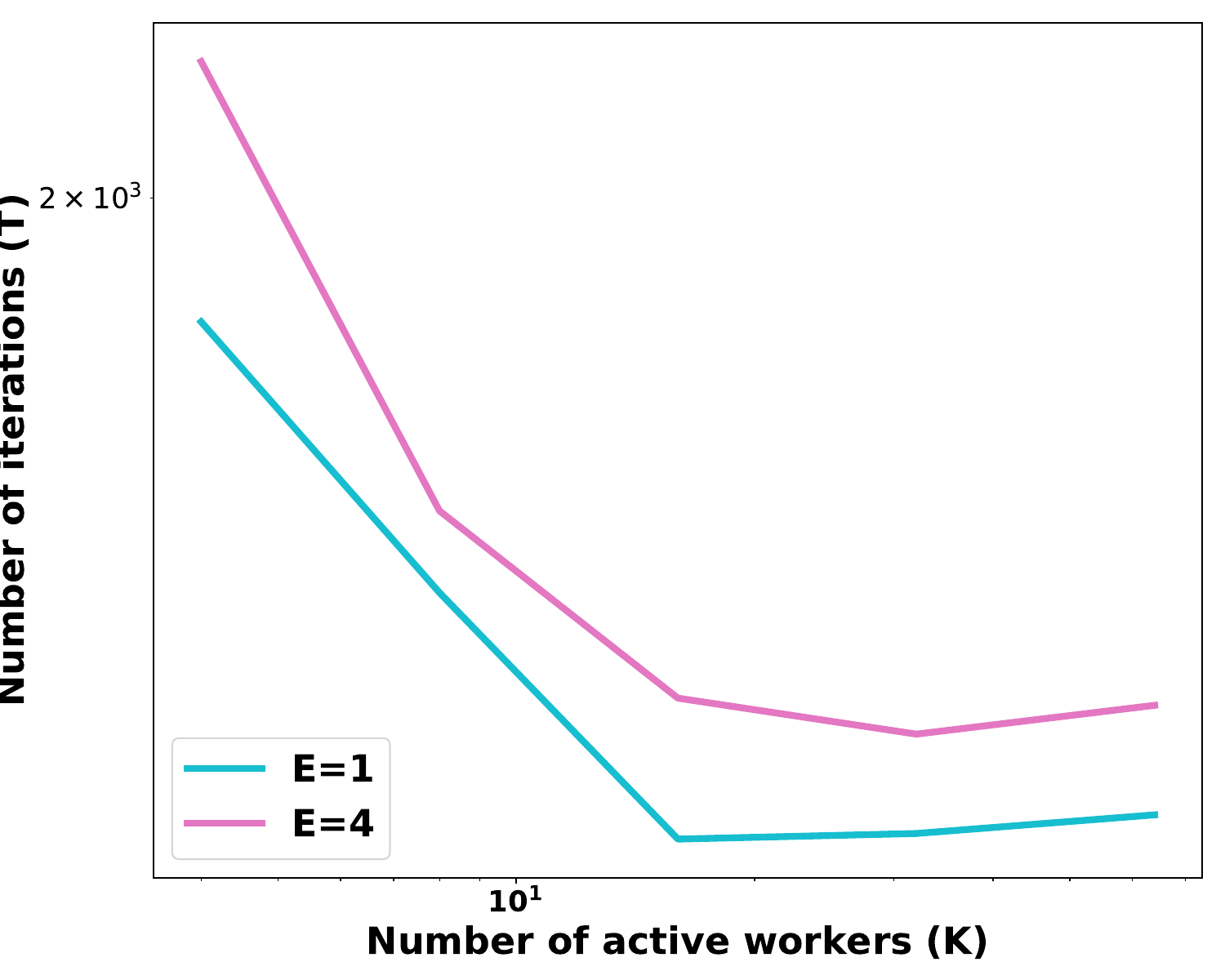}\\
\hspace{-2em}\includegraphics[width=0.33\textwidth]{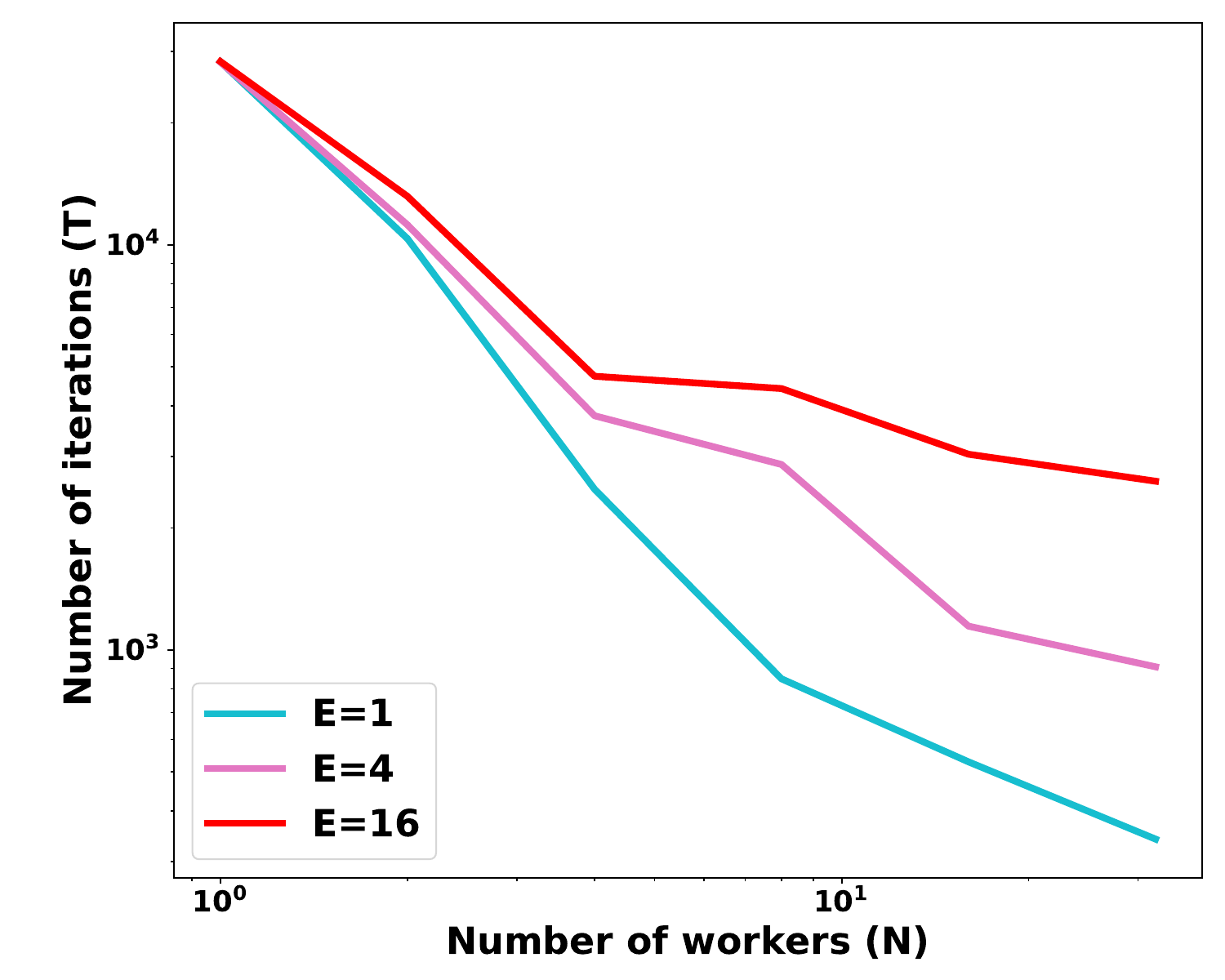} & 
\includegraphics[width=0.33\textwidth]{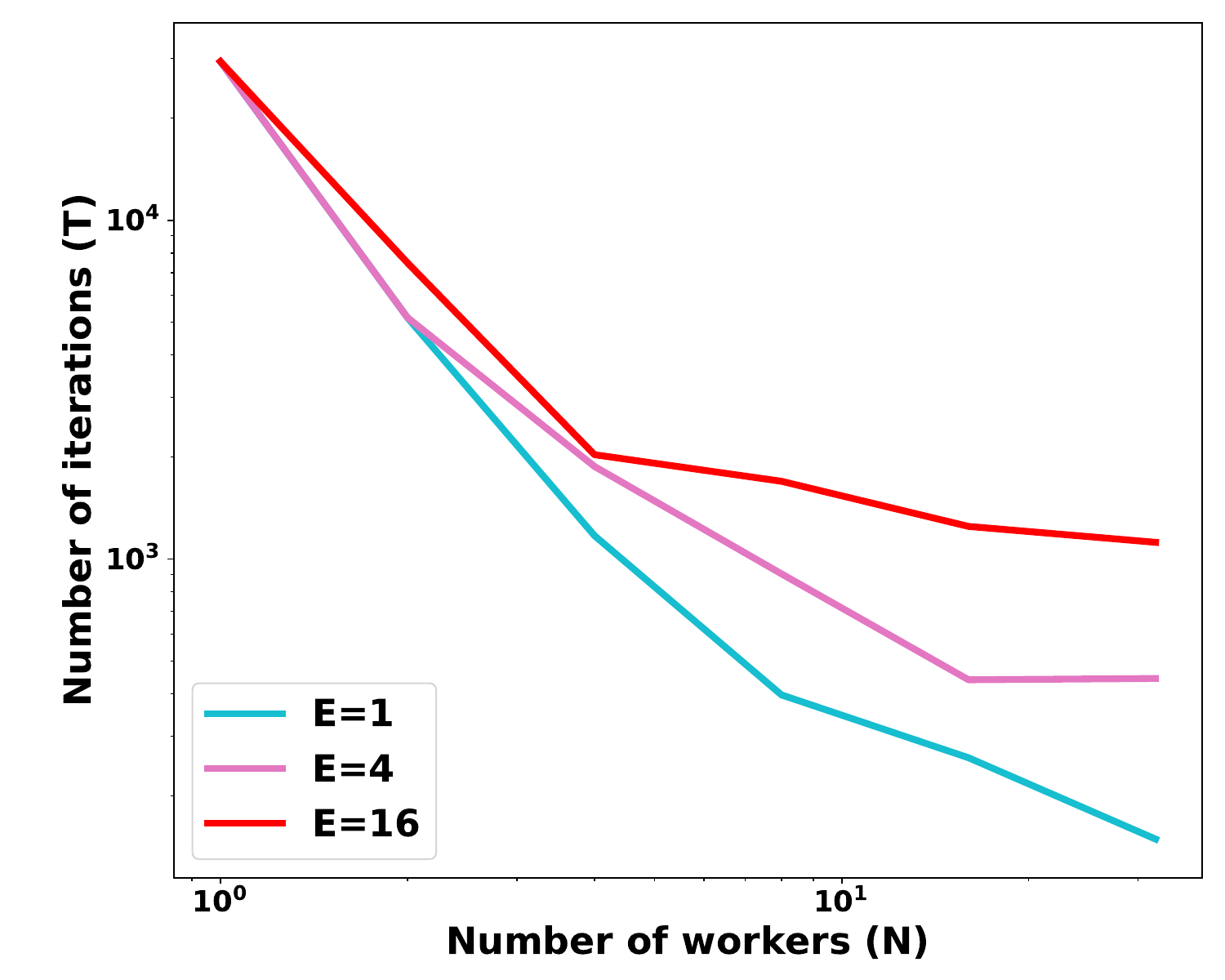}
& 
\includegraphics[width=0.33\textwidth]{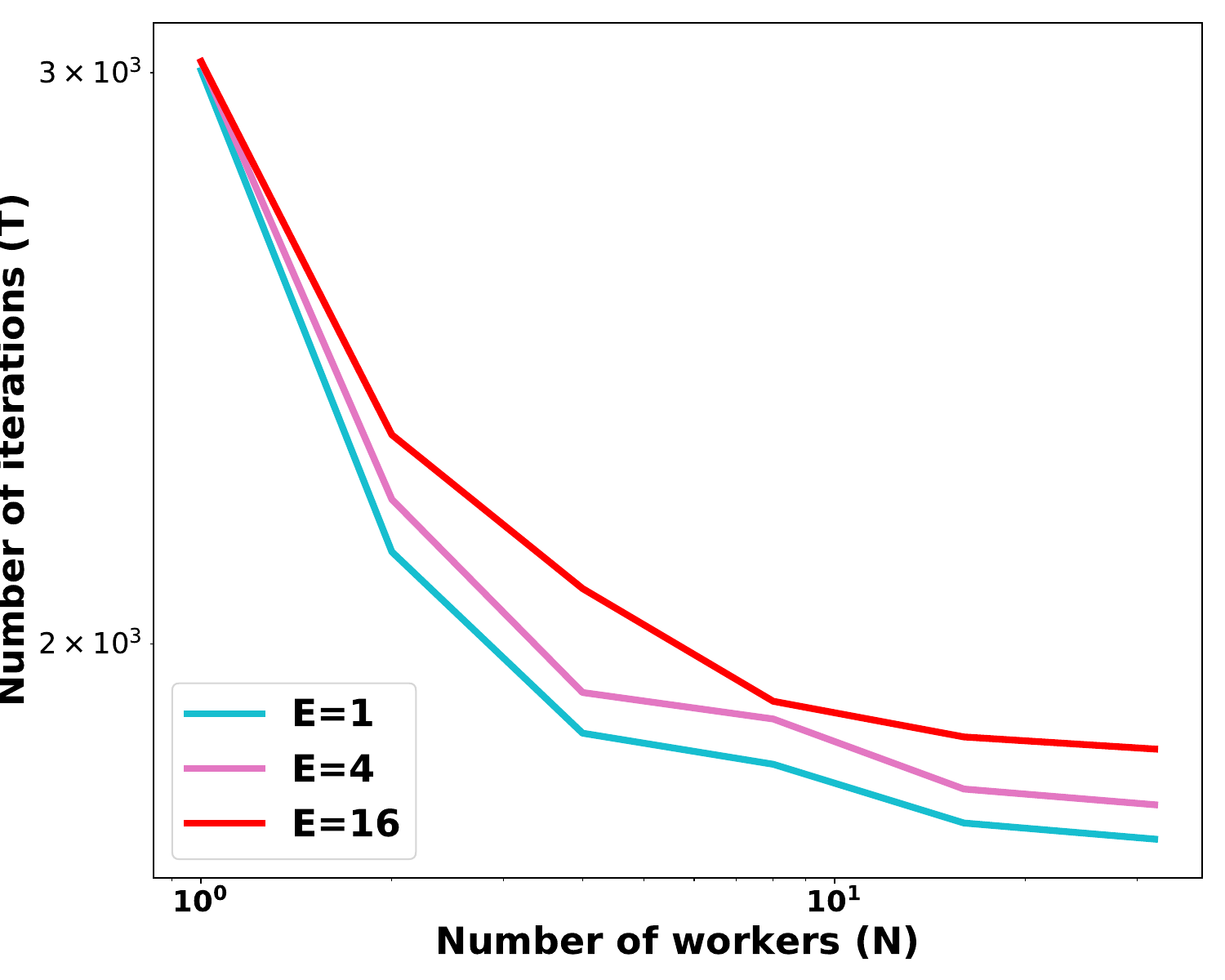}\\
(a) Strongly convex objective & (b) Convex smooth objective & (c) Linear regression
	\end{tabular}}
	\vspace{-1em}
\caption{The linear speedup of FedAvg in full participation, partial participation, and the linear speedup of Nesterov accelerated FedAvg, respectively. Both the x-axis and y-axis are logarithmic-scale.
}
\vspace{-1em}
\label{fig:speedup}
\end{figure*}

\section{Numerical Experiments}
\label{sec:exp}

In this section, we empirically examine the linear speedup convergence of FedAvg and Nesterov accelerated FedAvg in various settings, including strongly convex function, convex smooth function, and overparameterized objectives. 

\textbf{Setup.} Following the experimental setting in~\cite{stich2018local}, we
conduct experiments on both synthetic datasets and real-world dataset
w8a~\citep{platt1998fast} $(d=300, n=49749)$. See \citet{lai2022fedscale} for a comprehensive FL bench-marking suite containing more datasets.

We consider the distributed
objectives $F(\vw) = \sum_{k=1}^N p_kF_k(\vw)$, and the objective function on the
$k$-th local device includes three cases: 1) \textbf{Strongly convex
objective}: the regularized binary logistic regression problem, $F_k(\vw) =
\frac{1}{N_k} \sum_{i=1}^{N_k} \log( 1+ \exp(-y_i^k \vw^T\vx_i^k) + \frac{\lambda}{2}
\|\vw\|^2$. The regularization parameter is set to $\lambda = 1/n \approx
2e-5$. 2) \textbf{Convex smooth objective}: the binary logistic regression
problem without regularization. 3) \textbf{Overparameterized setting}:
the linear regression problem without adding noise to the label, $F_k(\vw) =
\frac{1}{N_k} \sum_{i=1}^{N_k} (\vw^T\vx_i^k + b  - y_i^k)^2$.  

\textbf{Linear speedup of FedAvg and Nesterov accelerated FedAvg.} To verify the linear speedup convergence as shown in Theorems~\ref{thm:SGD_scvx}~\ref{thm:SGD_cvx}~\ref{thm:nesterov_scvx}~\ref{thm:Nesterov_cvx}, we evaluate the number of iterations needed to reach
$\epsilon$-accuracy in three objectives. We initialize all runs with $\vw_0 = \textbf{0}_d$ and measure the number of iterations to reach the target accuracy $\epsilon$. For each configuration $(E, K)$, we extensively search the learning rate from $\min(\eta_0, \frac{nc}{1 + t})$, where
$\eta_0 \in \{0.1, 0.12, 1, 32 \}$ according to different problems and $c$ can
take the values $c = 2^i \ \forall i \in \ZZ$. As the results shown in Figure~\ref{fig:speedup},
the number of iterations decreases as the number of (active) workers increasing, which is consistent for FedAvg and Nesterov accelerated FedAvg across all scenarios.
For additional experiments on the impact of $E$, detailed experimental setup, and hyperparameter setting, please refer to the Appendix Section~\ref{sec:expsupp}.

 \textbf{Partial participation.} We verify the linear speedup in the
 partial participation settings, where we set $50\%$ of devices
 are active. As the results are shown in Figure~\ref{fig:speedup} (2nd row), FedAvg
 enjoys linear speedup in various settings even with partial
 device participation.

 \textbf{Nesterov accelerated FedAvg.} In the third row of Figure~\ref{fig:speedup}, 
 we report the last iteration to converge to $\epsilon$-accuracy of Nesterov accelerated FedAvg. The empirical observations align with Theorems~\ref{thm:nesterov_scvx}~\ref{thm:Nesterov_cvx} that the accelerated version of FedAvg can also achieve the linear speedup
 w.r.t the number of workers.
 
 \section{Concluding Remarks}
 
 In this paper, we provided a unified linear speedup analysis of the convergence of stochastic FedAvg and Nesterov accelerated FedAvg in convex smooth, strongly convex smooth, and overparameterized regimes in the presence of both system and data heterogeneity, while also highlighting the distinct communication efficiency differences between full and partial participation of local devices. It is well known that Nesterov and other momentum variants fail to accelerate over SGD in both the overparameterized
and convex settings.
Thus in general one cannot hope to obtain theoretical acceleration results for
the FedAvg algorithm with stochastic Nesterov updates, unless objectives are quadratic \citep{even2021continuized}. We refer to recent works such as \cite{yuan2020federated} for new federated learning algorithms that achieve linear speedup with better communication complexities in the full participation setting only. 

In this paper,  participating devices are assumed to be non-adversarial. However, in real-world applications, there may be malicious devices that try to poison the learning process by sending adversarial updates to the central server. An interesting direction is to investigate the robustness of Federated Learning algorithms to such attacks \citep{bhagoji2019analyzing,zhang2023delving}.

Lastly, we remark that the desirable linear speedup property has been studied in other federated versions of classical learning environments, such as federated reinforcement learning \citep{khodadadian2022federated}, and in entirely new FL regimes, such as the so-called anarchic federated learning \citep{yang2022anarchic}, where local devices have greater freedom in choosing when to participate in FL and the number of local updates.


\section*{Acknowledgments}
The first two authors ZQ and KL contributed equally to this work. This research is supported in part by a Stanford Interdisciplinary Graduate Fellowship (ZQ), NSF CMMI-2030411 (ZL), NSF CNS-2219488 (ZL), NSF IIS-1749940 (JZ), ONR N00014-20-1-2382 (JZ), NSF CCF-2312205 (ZZ), and NSF CCF-2106508 (ZZ). 

\appendix
\section{Additional Notations and Bounds for Sampling Schemes}
\label{sec:app:notations}
In this section, we introduce additional notations that are used throughout
the proofs. Following common practice, e.g.~\cite{stich2018local,li2019convergence}, we define two virtual sequences $\overline{\mathbf{v}}_{t}=\sum_{k=1}^{N} p_{k} \mathbf{v}_{t}^{k}$ and $\overline{\mathbf{w}}_{t}=\sum_{k=1}^{N} p_{k} \mathbf{w}_{t}^{k}$, where we recall the FedAvg updates from \eqref{eq:fedavg updates}:
\begin{align*}
\vv_{t+1}^{k} & =\vw_{t}^{k}-\alpha_{t}\vg_{t,k}, \hspace{1em}
\mathbf{w}_{t+1}^{k} =\left\{
\begin{array}{ll}
\mathbf{v}_{t+1}^{k} & \text { if } t+1 \notin \mathcal{I}_{E}, \\ 
\sum_{k \in \cS_{t+1}} q_k \mathbf{v}_{t+1}^{k} & \text { if } t+1 \in \mathcal{I}_{E}.
\end{array}\right.
\end{align*}
The following observations apply to FedAvg updates, while Nesterov accelerated FedAvg requires modifications. For full device participation or partial participation with $t \notin \cI_E$, note that
$\ov{v}_t = \ov{w}_t =\sum_{k=1}^{N} p_{k} \mathbf{v}_{t}^{k}$. For partial participation with $t \in \cI_E$, $\ov{w}_t \neq \ov{v}_t$ since $\ov{v}_t=\sum_{k=1}^{N} p_{k} \mathbf{v}_{t}^{k}$ while $\ov{w}_t=\sum_{k\in \cS_t}q_k\mathbf{w}_{t}^{k}$. However, we can
use unbiased sampling strategies such that $ \EE_{\cS_t} \ov{w}_t = \ov{v}_t$.
Note that $\overline{\mathbf{v}}_{t+1}$ is one-step SGD from $\overline{\mathbf{w}}_{t}$. 
\begin{align}
\overline{\mathbf{v}}_{t+1}=\overline{\mathbf{w}}_{t}-\alpha_{t} \mathbf{g}_{t},	\label{eq:vbar}
\end{align}
where $\vg_{t} = \sum_{k=1}^{N} p_{k} \vg_{t,k} $ is the one-step stochastic gradient averaged over all devices. 
\begin{align*}
\vg_{t,k} = \nabla F_{k}\left(\mathbf{w}_{t}^{k},\mathbf{\xi}_{t}^{k} \right),   
\end{align*}
Similarly, we denote the expected one-step gradient $\ov{g}_{t}= \EE_{\xi_t}[\vg_t] = \sum_{k=1}^{N} p_{k} \EE_{\mathbf{\xi}_{t}^{k}} \vg_{t,k}$, where
\begin{align}
\EE_{\mathbf{\xi}_{t}^{k}} \vg_{t,k}  = \nabla F_{k}\left(\mathbf{w}_{t}^{k}\right), 
\label{eq:gradient}
\end{align}
and $\mathbf{\xi}_t = \{\mathbf{\xi}_t^k\}_{k=1}^N$ denotes random samples at all devices at time step $t$. \\
Since in this work we also consider the case of partial participation, the sampling strategy to approximate the system heterogeneity can also affect the convergence. Here we
follow the prior works~\cite{li2019convergence} and~\cite{li2018federated} and consider two types of sampling
schemes that guarantee $ \EE_{\cS_t} \ov{w}_t = \ov{v}_t$. 
The sampling scheme I establishes $\cS_{t+1}$ by \emph{i.i.d.} sampling the devices according to probabilities $p_k$ with replacement, and setting $q_k=\frac{1}{K}$.
In this case the upper bound of expected square norm of $\ov{w}_{t+1} - \ov{v}_{t+1}$ is given by~\cite[Lemma 5]{li2019convergence}:
\begin{align}
\EE_{\cS_{t+1}}\left\|\ov{w}_{t+1} - \ov{v}_{t+1}\right\|^2	\leq \frac{4}{K} \alpha_t^2 E^2G^2.
\end{align}
The sampling scheme II establishes $\cS_{t+1}$ by uniformly sampling all devices without
replacement and setting $q_k=p_k\frac{N}{K}$, in which case we have
\begin{align}
\EE_{\cS_{t+1}}\left\|\ov{w}_{t+1} - \ov{v}_{t+1}\right\|^2	\leq \frac{4(N - K)}{K(N-1)} \alpha_t^2 E^2G^2.
\end{align}
We summarize these upper bounds as follows: 
\begin{align}
	\EE_{\cS_{t+1}}\left\|\ov{w}_{t+1} - \ov{v}_{t+1}\right\|^2 \leq  \frac{4}{K} \alpha_t^2 E^2G^2.
	\label{eq:partialsample}
\end{align}
and this bound will be used in the convergence proof of the partial participation result.

\section{Comparison of Convergence Rates with Related Works}
\label{sec:app:comparison}
In this section, we compare our convergence rate with the best-known results in the literature (see Table~\ref{tb:convergenceratev3}). 
In~\cite{haddadpour2019convergence}, the authors provide $\cO(1/NT)$ convergence rate of non-convex problems under Polyak-Łojasiewicz (PL) condition, which
means their results can directly apply to the strongly convex problems. However, their assumption is based on bounded gradient diversity, defined as follows: 
\begin{align*}
	\Lambda(\vw) = \frac{\sum_{k}p_{k}\|\nabla F_{k}(\mathbf{w})\|_{2}^{2}}{\|\sum_{k}p_{k}\nabla F_{k}(\mathbf{w})\|_{2}^{2}} \leq B
\end{align*} 
This is a more restrictive assumption comparing to assuming bounded gradient under the case of target accuracy $\epsilon \rightarrow 0$ and PL condition.
To see this, consider the gradient diversity at the global optimal $\vw^*$, i.e., $\Lambda(\vw^*) = \frac{\sum_{k}p_{k}\|\nabla F_{k}(\mathbf{w})\|_{2}^{2}}{\|\sum_{k}p_{k}\nabla F_{k}(\mathbf{w})\|_{2}^{2}}$. For $\Lambda(\vw^*)$ to be bounded, it requires $\|\nabla F_{k}(\mathbf{w}^*)\|_{2}^{2} = 0$, $\forall \ k$. This indicates 
$\vw^*$ is also the minimizer of each local objective, which contradicts to the practical setting of heterogeneous data. Therefore, their bound 
is not effective for arbitrary small $\epsilon$-accuracy under general heterogeneous data while our convergence results still hold in this case.

\begin{table*}[h!]
{
\tiny
\centering
\resizebox{\textwidth}{!}{\begin{tabular}{|c|c|c|c|c|c|c|c|}
\hline Reference                 & Convergence rate    & E                           			& NonIID & Participation & Extra Assumptions  		  & Setting  \\ \hline\hline 
FedAvg\cite{li2019convergence}         & $\cO(\frac{E^2}{T})$& $\cO(1)$                    		& \cmark & Partial       & Bounded gradient   		  & Strongly convex  \\ \hline
FedAvg\cite{haddadpour2019convergence} & $\cO(\frac{1}{KT})$ & $\cO(K^{-1/3}T^{2/3})^{\dagger}$ & \cmark$^{\ddagger\ddagger}$ & Partial       & Bounded gradient diversity   & Strongly convex$^{\mathsection}$  \\ \hline
FedAvg\cite{koloskova2020unified} & $\cO(\frac{1}{NT})$ & $\cO(N^{-1/2}T^{1/2})$     	& \cmark & Full       & Bounded gradient   & Strongly convex  \\ \hline
FedAvg\cite{karimireddy2019scaffold} & $\cO(\frac{1}{KT})$ & $\cO(N^{-1/2}T^{1/2})$   	& \cmark & Partial       & Bounded gradient dissimilarity   & Strongly convex  \\ \hline
FedAvg/N-FedAvg (our analysis)                 & $\cO(\frac{1}{KT})$ & $\cO(N^{-1/2}T^{1/2})^{\ddagger}$ & \cmark	 & Partial       & Bounded gradient             & Strongly convex  \\\hline\hline
FedAvg\cite{khaled2020tighter}  & $\cO(\frac{1}{\sqrt{NT}})$ & $\cO(N^{-3/2}T^{1/2})$     	    & \cmark& Full        & Bounded gradient             & Convex  \\\hline
FedAvg\cite{koloskova2020unified} & $\cO(\frac{1}{\sqrt{NT}})$ & $\cO(N^{-3/4}T^{1/4})$    & \cmark & Full       & Bounded gradient             &  Convex  \\ \hline
FedAvg\cite{karimireddy2019scaffold} & $\cO(\frac{1}{\sqrt{KT}})$ & $\cO(N^{-3/4}T^{1/4})$  & \cmark & Partial       & Bounded gradient dissimilarity   &  Convex  \\ \hline
FedAvg/N-FedAvg (our analysis)      & $\cO\left(\frac{1}{\sqrt{KT}}\right)$ & $\cO(N^{-3/4}T^{1/4})^{\ddagger}$& \cmark			& Partial     & Bounded gradient            &  Convex   \\ \hline
FedAvg (our analysis) & $\cO\left(\exp(-\frac{NT}{E\kappa_1})\right)$ & $ \cO(T^{\beta})$                   & \cmark&  Partial     & Bounded gradient    & Overparameterized \\
\hline
\end{tabular}}
}
\caption{A high-level summary of the convergence results in this paper compared to prior state-of-the-art FL algorithms. This table only highlights the
dependence on $T$ (number of iterations), $E$ (the maximal number of local steps), $N$ (the total number of devices), and $K\leq N$ the number of participated devices. 
$\kappa$ is the condition number of the system and $\beta \in (0,1)$. We denote Nesterov accelerated FedAvg as N-FedAvg in this table.}
{\raggedright 
         $^{\dagger}$ This $E$ is obtained under i.i.d. setting. \\
         $^{\ddagger}$ This $E$ is obtained under full participation setting. \\ 
         $^{\mathsection}$ In~\cite{haddadpour2019convergence}, the convergence rate is for non-convex smooth problems with PL condition, which also applies to strongly convex problems. Therefore, we compare it with our strongly convex results here.\\
         $^{\ddagger\ddagger}$ The bounded gradient diversity assumption is not applicable for general heterogeneous data when converging to arbitrarily small $\epsilon$-accuracy (see discussions in Sec~\ref{sec:app:comparison}).\\
           \par}
\label{tb:convergenceratev3}
\end{table*}

\section{A High-level Summary of Our FedAvg Analysis}
\label{sec:app:sum}

To facilitate the understanding of our analysis and highlight the improvement of our work comparing to prior arts, we summarize the general steps 
used in the proofs across the various settings. In this section,
we take the strongly convex case as an example to illustrate our analysis. The corresponding proof for general convex functions follows the 
same framework.  

\begin{algorithm}[h!]\small
\begin{algorithmic}[1]
\STATE \textbf{Server input:} initial model $\vw_0$, initial step size $\alpha_0$, local steps $E$. 
\STATE \textbf{Client input:} 
\FOR {each round $r = 0, 1, ..., R$, where $r = t*E$} 
\STATE  Sample clients $\cS_t \subseteq \{1,...,N\}$
\STATE Broadcast $\vw$ to all clients $k \in \cS_t$
\FOR {each client $k \subseteq \cS_t$}
\STATE initialize local model $\vw_t^k = \vw$
\FOR {$t = r * E + 1, \dots, (r+1)*E$}
\STATE $\vw_{t+1}^{k}  =\vw_{t}^{k}-\alpha_{t}\vg_{t,k}$
\ENDFOR
\ENDFOR
\STATE Average the local models at server end: $\ov{w}_t = \sum_{k\in \cS_t} \vw_t^k$.
\ENDFOR 
\end{algorithmic}
\caption{\textsc{FedAvg}: Federated Averaging}
\label{alg:fedavg}
\end{algorithm}

\textbf{One step progress bound} \\ 
This step establishes the progress of distance ($\|\ov{w}_{t}-\vw^{\ast}\|^{2}$) to optimal solution after one step SGD update (see line 9, Alg~\ref{alg:fedavg}), as the following equation shows:
\begin{align*}
	\mathbb{E}\|\ov{w}_{t+1}-\vw^{\ast}\|^{2} & \leq \cO(\eta_t\mathbb{E}\|\ov{w}_{t}-\vw^{\ast}\|^{2} + \alpha_t^2 \sigma^2/N + \alpha_t^3E^2G^2).
\end{align*}
The above bound consists of three main ingredients, the distance to optima
in previous step (with $\eta_t \in (0, 1)$ to obtained a contraction bound), 
the variance of stochastic gradients in local clients (second term), the variance
across different clients (third term). 
Notice that the third term in this bound is the primary source of improvement in the rate.  Comparing to the bound in~\cite{li2019convergence}, we improve 
the third term from $\cO(\alpha_t^2E^2G^2)$ to $\cO(\alpha_t^2E^2G^2)$, which enables the linear speedup in the convergence rate.

\textbf{Iterating the one-step bound}\\
This step uses the \textit{one step progress bound} iteratively to 
connect the the current distance to optimal solution with the initial distance ($\|\ov{w}_{0}-\vw^{\ast}\|^{2}$), as follows:
\begin{align*}
	\mathbb{E}\|\ov{w}_{t+1}-\vw^{\ast}\|^{2} & \leq \cO(\mathbb{E}\|\ov{w}_{0}-\vw^{\ast}\|^{2} \frac{1}{T}).
\end{align*}
Then we can use the distance to optima to upper bound the optimality gap ($F(\vw_t) - F^* \leq \cO(1/T)$), as follows:
\begin{align*}
	\mathbb{E}(F(\ov{w}_{t}))-F^{\ast} & \leq\cO(\mathbb{E}\|\ov{w}_{t}-\vw^{\ast}\|^{2}). 
\end{align*}
The convergence rate of the optimality gap is equally obtained as the convergence rate of the distance to optima.

\textbf{From full participation to partial participation}\\
There are three sources of variances that affect the convergence rate.
The first two sources come from the variances of within local clients and
across clients (second and third term in one step progress bound). 
The partial participation, which involves a sampling procedure, is the third
source of variance. Therefore, comparing to the rate in full participation, this will add another term of variance into the convergence rate, where we follow a similar derivation as in~\cite{li2019convergence}.

\section{Technical Lemmas}

To facilitate reading, we first summarize some basic properties of $L$-smooth and $\mu$-strongly
convex functions, found in e.g.~\cite{rockafellar1970convex}, which are used in various steps of proofs in the appendix. 
\begin{lemma}
	Let $F$ be a convex $L$-smooth function. Then we have the following
	inequalities:
	
	1. Quadratic upper bound: $0\leq F(\vw)-F(\vw')-\langle\nabla F(\vw'),\vw-\vw'\rangle\leq\frac{L}{2}\|\vw-\vw'\|^{2}$. 
	
	2. Coercivity: $\frac{1}{L}\|\nabla F(\vw)-\nabla F(\vw')\|^{2}\leq\langle\nabla F(\vw)-\nabla F(\vw'),\vw-\vw'\rangle$.
	
	3. Lower bound: $F(\vw)\geq F(\vw')+\langle\nabla F(\vw'),\vw-\vw'\rangle+\frac{1}{2L}\|\nabla F(\vw)-\nabla F(\vw')\|^{2}$.
	In particular, $\|\nabla F(\vw)\|^{2}\leq2L(F(\vw)-F(\vw^{\ast}))$.
	
	4. Optimality gap: $F(\vw)-F(\vw^{\ast})\leq$$\langle\nabla F(\vw),\vw-\vw^{\ast}\rangle$.
\label{lem:lsmooth}
\end{lemma}
\begin{lemma}
	Let $F$ be a $\mu$-strongly convex function. Then 
	\begin{align*}
	& F(\vw)  \leq F(\vw')+\langle\nabla F(\vw'),\vw-\vw'\rangle+\frac{1}{2\mu}\|\nabla F(\vw)-\nabla F(\vw')\|^{2}\\
	& F(\vw)-F(\vw^{\ast})  \leq\frac{1}{2\mu}\|\nabla F(\vw)\|^{2}
	\end{align*}
\end{lemma}


\section{Proof of Convergence Results for FedAvg}
\label{sec:app:fedavg}

\subsection{Strongly Convex Smooth Objectives}

To organize our proofs more effectively and highlight the significance of our results compared to prior works, 
we first state the following key lemmas used in proofs of main results and defer their proofs to later. 
\begin{lemma}[\textbf{One step progress, strongly convex}] Let $\overline{\mathbf{w}}_{t}=\sum_{k=1}^{N}p_{k}\mathbf{w}_{t}^{k}$, and
suppose our functions satisfy Assumptions~\ref{ass:lsmooth},\ref{ass:stroncvx},\ref{ass:boundedvariance},\ref{ass:subgrad2}, and set step size $\alpha_{t}=\frac{4}{\mu(\gamma+t)}$
	with $\gamma=\max\{32\kappa,E\}$ and $\kappa=\frac{L}{\mu}$, then the updates of FedAvg with full participation satisfy
	\begin{align*}
	\mathbb{E}\|\ov{w}_{t+1}-\vw^{\ast}\|^{2} & \leq(1-\mu\alpha_{t})\mathbb{E}\|\ov{w}_{t}-\vw^{\ast}\|^{2}+\alpha_{t}^{2}\frac{1}{N}\nu_{\max}\sigma^{2}+6E^{2}L\alpha_{t}^{3}G^{2}.
	\end{align*}
\label{lem:scvxoner}
\end{lemma}
We emphasize that the above lemma is the key step that allows us to obtain a bound that improves on the convergence result of~\cite{li2019convergence} with linear speedup. Its proof will make use of the following two results. 
\begin{lemma}[\textbf{Bounding gradient variance (Lemma 2 \cite{li2019convergence})} ]
Given Assumption~\ref{ass:boundedvariance}, the upper bound of gradient variance is given as follows,
\begin{align*}
	\mathbb{E}\|\vg_{t}-\ov{g}_{t}\|^{2} \leq \sum_{k=1}^{N}p_{k}^{2}\sigma_{k}^{2}.
	\end{align*}
\label{lem:bgv}
\end{lemma}

\begin{lemma}[\textbf{Bounding the divergence of $\vw_t^k$ (Lemma 3 \cite{li2019convergence})} ]
Given Assumption~\ref{ass:subgrad2}, and assume that $\alpha_t$ is non-increasing and $\alpha_t \leq 2\alpha_{t+E}$ for all $t\geq 0$, we have
	\begin{align*}
	\mathbb{E}\left[\sum_{k=1}^{N}p_{k}\|\ov{w}_{t}-\vw_{t}^{k}\|^{2} \right]\leq 4E^{2}\alpha_{t}^{2}G^{2}.
	\end{align*}
\label{lem:bdw}
\end{lemma}

We now restate Theorem~\ref{th:scvx_sgd} from the main text and then prove it using Lemma~\ref{lem:scvxoner}.
\begin{thm}
	Let $\overline{\mathbf{w}}_{T}=\sum_{k=1}^{N}p_{k}\mathbf{w}_{T}^{k}$ in FedAvg,
	$\nu_{\max}=\max_{k}Np_{k}$, and set decaying learning rates $\alpha_{t}=\frac{4}{\mu(\gamma+t)}$
	with $\gamma=\max\{32\kappa,E\}$ and $\kappa=\frac{L}{\mu}$. Then
	under Assumptions~\ref{ass:lsmooth},\ref{ass:stroncvx},\ref{ass:boundedvariance},\ref{ass:subgrad2} with full device participation, 
	\begin{align*}
	\mathbb{E}F(\overline{\mathbf{w}}_{T})-F^{\ast}=\mathcal{O}\left(\frac{\kappa\nu_{\max}\sigma^{2}/\mu}{NT}+\frac{\kappa^{2}E^{2}G^{2}/\mu}{T^{2}}\right)
	\end{align*}
	and with partial device participation with at most $K$ sampled devices
	at each communication round, 
	\begin{align*}
	\mathbb{E}F(\overline{\mathbf{w}}_{T})-F^{\ast}=\mathcal{O}\left(\frac{\kappa EG^{2}/\mu}{KT}+\frac{\kappa\nu_{\max}\sigma^{2}/\mu}{NT}+\frac{\kappa^{2}E^{2}G^{2}/\mu}{T^{2}}\right)
	\end{align*}
\end{thm}

\begin{proof}
	The road map of the proof for full device participation contains three steps. First, we establish a recursive relationship between $\mathbb{E}\|\ov{w}_{t+1}-\vw^{\ast}\|^{2}$
	and $\mathbb{E}\|\ov{w}_{t}-\vw^{\ast}\|^{2}$, upper bounding the progress of FedAvg from step $t$ to step $t+1$. 
	Second, we show that $\mathbb{E}\|\ov{w}_{t}-\vw^{\ast}\|^{2}=\mathcal{O}(\frac{\nu_{\max}\sigma^{2}/\mu}{tN}+\frac{E^{2}LG^{2}/\mu^2}{t^{2}})$ by induction using the recursive relationship from the previous step. 
	Third, we use the property of $L$-smoothness to bound the optimality gap by $\mathbb{E}\|\ov{w}_{t}-\vw^{\ast}\|^{2}$. 
	
	By Lemma~\ref{lem:scvxoner}, we have the following upper bound for the one step progress: 
	\begin{align*}
	\mathbb{E}\|\ov{w}_{t+1}-\vw^{\ast}\|^{2} & \leq(1-\mu\alpha_{t})\mathbb{E}\|\ov{w}_{t}-\vw^{\ast}\|^{2}+\alpha_{t}^{2}\frac{1}{N}\nu_{\max}\sigma^{2}+6E^{2}L\alpha_{t}^{3}G^{2}.
	\end{align*}
	We show next that $\mathbb{E}\|\ov{w}_{t}-\vw^{\ast}\|^{2}=\mathcal{O}(\frac{\nu_{\max}\sigma^{2}/\mu}{tN}+\frac{E^{2}LG^{2}/\mu^2}{t^{2}})$ using induction. 
	To simplify the presentation, we denote $C\equiv6E^{2}LG^{2}$ and $D\equiv\frac{1}{N}\nu_{\max}\sigma^{2}$.
	Suppose that we have the bound $\mathbb{E}\|\ov{w}_{t}-\vw^{\ast}\|^{2}\leq b\cdot(\alpha_{t}D+\alpha_{t}^{2}C)$
	for some constant $b$ and learning rates $\alpha_{t}$. Then the one step progress from Lemma~\ref{lem:scvxoner} becomes:
	\begin{align*}
	\mathbb{E}\|\ov{w}_{t+1}-\vw^{\ast}\|^{2} & \leq (b(1-\mu\alpha_{t})+\alpha_{t})\alpha_{t}D+(b(1-\mu\alpha_{t})+\alpha_{t})\alpha_{t}^{2}C
	\end{align*}
	To establish the result at step $t+1$, it remains to choose $\alpha_{t}$ and $b$ such that $(b(1-\mu\alpha_{t})+\alpha_{t})\alpha_{t}\leq b\alpha_{t+1}$
	and $(b(1-\mu\alpha_{t})+\alpha_{t})\alpha_{t}^{2}\leq b\alpha_{t+1}^{2}$.
	If we let $\alpha_{t}=\frac{4}{\mu(t+\gamma)}$
	where $\gamma=\max\{E,32\kappa\}$ (choice of $\gamma$ required to guarantee the one step progress) and set $b=\frac{4}{\mu}$, we have:
	\begin{align*}
	& (b(1-\mu\alpha_{t})+\alpha_{t})\alpha_{t}  =\left(b(1-\frac{4}{t+\gamma})+\frac{4}{\mu(t+\gamma)}\right)\frac{4}{\mu(t+\gamma)}
	 \leq b\frac{4}{\mu(t+\gamma+1)}=b\alpha_{t+1}\\
	& (b(1-\mu\alpha_{t})+\alpha_{t})\alpha_{t}^{2} 
	 =b(\frac{t+\gamma-2}{t+\gamma})\frac{16}{\mu^{2}(t+\gamma)^{2}} \leq b\frac{16}{\mu^{2}(t+\gamma+1)^{2}}=b\alpha_{t+1}^{2}
	\end{align*}
	where we have used the following inequalities:
	\begin{align*}
	\frac{t+\gamma-1}{(t+\gamma)^{2}} \leq\frac{1}{(t+\gamma+1)}   \hspace{2em}
	\frac{t+\gamma-2}{(t+\gamma)^{3}} \leq\frac{1}{(t+\gamma+1)^{2}}  \hspace{2em}  \forall\ \gamma\geq1
	\end{align*}
	Thus we have established the result at step $t+1$ assuming the result is correct at step $t$:
	\begin{align*}
	\mathbb{E}\|\ov{w}_{t+1}-\vw^{\ast}\|^{2} & \leq b\cdot(\alpha_{t+1}D+\alpha_{t+1}^{2}C)
	\end{align*}
	At step $t=0$, we can ensure the following inequality by scaling $b$ with $c\|\vw_{0}-\vw^{\ast}\|^{2}$ for a sufficiently large constant
	$c$:
	\begin{align*}
	\|\vw_{0}-\vw^{\ast}\|^{2} \leq b\cdot(\alpha_{0}D+\alpha_{0}^{2}C) =b\cdot(\frac{4}{\mu\gamma}D+\frac{16}{\mu^{2}\gamma^{2}}C)
	\end{align*}
	It follows that 
	\begin{align}
	\mathbb{E}\|\ov{w}_{t}-\vw^{\ast}\|^{2} & \leq c\|\vw_{0}-\vw^{\ast}\|^{2}\frac{4}{\mu}(D\alpha_{t}+C\alpha_{t}^{2})
	\label{eq:th1:1}
	\end{align}
	for all $t\geq0$. 
	
	Finally, the $L$-smoothness of $F$ implies 
	\begin{align*}
	\mathbb{E}(F(\ov{w}_{T}))-F^{\ast} & \leq\frac{L}{2}\mathbb{E}\|\ov{w}_{T}-\vw^{\ast}\|^{2}\\ 
	& \leq\frac{L}{2}c\|\vw_{0}-\vw^{\ast}\|^{2}\frac{4}{\mu}(D\alpha_{T}+C\alpha_{T}^{2})\\
	& =2c\|\vw_{0}-\vw^{\ast}\|^{2}\kappa(D\alpha_{T}+C\alpha_{T}^{2})\\
	& \leq2c\|\vw_{0}-\vw^{\ast}\|^{2}\kappa\left[\frac{4}{\mu(T+\gamma)}\cdot\frac{1}{N}\nu_{\max}\sigma^{2}+6E^{2}LG^{2}\cdot(\frac{4}{\mu(T+\gamma)})^{2}\right]\\
	& =\mathcal{O}(\frac{\kappa}{\mu}\frac{1}{N}\nu_{\max}\sigma^{2}\cdot\frac{1}{T}+\frac{\kappa^{2}}{\mu}E^{2}G^{2}\cdot\frac{1}{T^{2}})
	\end{align*}
	where in the first line, we use the property of $L$-smooth function (see Lemma~\ref{lem:lsmooth}), and in the second line, we use the conclusion in \eq{\ref{eq:th1:1}}. 
	
	\rebuttal{
	With partial participation, the update at each communication round
	is now given by weighted averages over a subset of sampled devices. When $t+1\notin\mathcal{I}_{E}$,
	$\ov{v}_{t+1}=\ov{w}_{t+1}$, while when $t+1\in\mathcal{I}_{E}$,
	we have $\mathbb{E}\ov{w}_{t+1}=\ov{v}_{t+1}$ by design
	of the sampling schemes~(\cite{li2019convergence}, Lemma 4).
	
	Let $t+1$ and $t+1+E$ be two consecutive communication
rounds. Using $\mathbb{E}\ov{w}_{t+E+1}=\ov{v}_{t+E+1}$,
we can write 
\begin{align*}
\mathbb{E}\|{\ov{w}}_{t+E+1}-\vw^{\ast}\|^{2} & =\mathbb{E}\|\ov{w}_{t+E+1}-\ov{v}_{t+E+1}+\ov{v}_{t+E+1}-\vw^{\ast}\|^{2}\\
 & =\mathbb{E}\|\ov{v}_{t+E+1}-\vw^{\ast}\|^{2}+\mathbb{E}\|\ov{w}_{t+E+1}-\ov{v}_{t+E+1}\|^{2}
\end{align*}
We note that applying the one step progress consecutively to the sequence $\ov{v}_{t}$ with full participation setting starting
from round $t+1$ and stopping at round $t+1+E$, we can bound the
first term as 
\begin{align*}
\mathbb{E}\|\ov{v}_{t+E+1}-\vw^{\ast}\|^{2} & \leq(1-\mu\alpha_{t+1})\mathbb{E}\|\ov{w}_{t+1}-\vw^{\ast}\|^{2}+E\alpha_{t+1}^{2}\frac{1}{N}\nu_{\max}\sigma^{2}+6E^{3}L\alpha_{t+1}^{3}G^{2}
\end{align*}
	
	The bound for $\mathbb{E}\|\ov{w}_{t+E+1}-\ov{v}_{t+E+1}\|^{2}$ for the two sampling schemes we consider is provided in \eq{\ref{eq:partialsample}} as
\begin{align*}
\mathbb{E}\|\ov{w}_{t+E+1}-\ov{v}_{t+E+1}\|^{2} & \leq\mathbb{E}\|\ov{w}_{t+E+1}-\ov{w}_{t+1}\|^{2}\\
 & \leq\frac{4}{K}\alpha_{t}^{2}E^{2}G^{2}
\end{align*}
 which yields 
\begin{align*}
\mathbb{E}\|\ov{w}_{t+E+1}-\vw^{\ast}\|^{2} & \leq(1-\mu\alpha_{t+1})\mathbb{E}\|\ov{w}_{t+1}-\vw^{\ast}\|^{2}+E\alpha_{t+1}^{2}\frac{1}{N}\nu_{\max}\sigma^{2}\\
&+6E^{3}L\alpha_{t+1}^{3}G^{2}+\frac{4}{K}\alpha_{t}^{2}E^{2}G^{2}
\end{align*}
We note that this is similar to the one-step progress bound in \cite{karimireddy2019scaffold} for two consecutive communication rounds. From here,
using the same induction argument (effectively $T/E$ times instead of $T$ times) and $L$-smoothness as the full participation case implies 
	\begin{align*}
	\mathbb{E}F(\ov{w}_{T})-F^{\ast}=\mathcal{O}(\frac{\kappa\nu_{\max}\sigma^{2}/\mu}{NT}+\frac{\kappa EG^{2}/\mu}{KT}+\frac{\kappa^{2}E^{2}G^{2}/\mu}{T^{2}})
	\end{align*}
	
The advantage of bounding the square distance to optimum between consecutive communication rounds is that it results in bounding the sampling variance $\mathbb{E}\|\ov{w}_{t}-\ov{v}_{t}\|^{2} $ $T/E$ instead of $T$ times, which gives an  $\mathcal{O}(E/KT)$ term instead of $\mathcal{O}(E^2/KT)$ in the convergence result.}
	\end{proof}

\subsubsection{Deferred Proofs of Key Lemmas}
Here we first rewrite the proofs of lemmas \ref{lem:bgv} and  \ref{lem:bdw} from~\cite{li2019convergence} with slight modifications for the consistency and completeness of this work, since later we will use modified versions of these results in the convergence proof for Nesterov accelerated FedAvg.
\begin{proof}[Proof of lemma~\ref{lem:bgv}]
	\begin{align*}
	\mathbb{E}\|\vg_{t}-\ov{g}_{t}\|^{2} & =\mathbb{E}\|\vg_{t}-\mathbb{E}\vg_{t}\|^{2}=\sum_{k=1}^{N}p_{k}^{2}\|\vg_{t,k}-\mathbb{E}\vg_{t,k}\|^{2}\leq \sum_{k=1}^{N}p_{k}^{2}\sigma_{k}^{2}
	\end{align*}
\end{proof}

\begin{proof}[Proof of lemma~\ref{lem:bdw}]
	Now we bound $\mathbb{E}\sum_{k=1}^{N}p_{k}\|\ov{w}_{t}-\vw_{t}^{k}\|^{2}$ following \cite{li2019convergence}.
	Since communication is done every $E$ steps, for any $t\geq0$, we
	can find a $t_{0}\leq t$ such that $t-t_{0}\leq E-1$ and $\vw_{t_{0}}^{k}=\ov{w}_{t_{0}}$for
	all $k$. Moreover, using $\alpha_{t}$ is non-increasing and $\alpha_{t_{0}}\leq2\alpha{}_{t}$
	for any $t-t_{0}\leq E-1$, we have 
		\begin{align*}
	& \mathbb{E}\sum_{k=1}^{N}p_{k}\|\ov{w}_{t}-\vw_{t}^{k}\|^{2}\\
= & \mathbb{E}\sum_{k=1}^{N}p_{k}\|\vw_{t}^{k}-\ov{w}_{t_{0}}-(\ov{w}_{t}-\ov{w}_{t_{0}})\|^{2}\\
\leq &\mathbb{E}\sum_{k=1}^{N}p_{k}\|\vw_{t}^{k}-\ov{w}_{t_{0}}\|^{2}\\
	= &\mathbb{E}\sum_{k=1}^{N}p_{k}\|\vw_{t}^{k}-\vw_{t_{0}}^{k}\|^{2}\\
	= &\mathbb{E}\sum_{k=1}^{N}p_{k}\|-\sum_{i=t_{0}}^{t-1}\alpha_{i}\vg_{i,k}\|^{2}\\
	\leq & 2\sum_{k=1}^{N}p_{k}\mathbb{E}\sum_{i=t_{0}}^{t-1}E\alpha_{i}^{2}\|\vg_{i,k}\|^{2}\\
	\leq & 2\sum_{k=1}^{N}p_{k}E^{2}\alpha_{t_{0}}^{2}G^{2}\\
	\leq & 4E^{2}\alpha_{t}^{2}G^{2}
	\end{align*}
\end{proof} 
Based on the results of Lemma~\ref{lem:bgv},~\ref{lem:bdw}, we now prove the upper bound of one step SGD progress. This 
proof improves on the previous work~\cite{li2019convergence} and reveals the linear speedup of convergence of FedAvg. 
\begin{proof}[Proof of lemma~\ref{lem:scvxoner}]
	We have 
	\begin{align*}
	\|\ov{w}_{t+1}-\vw^{\ast}\|^{2} & =\|(\ov{w}_{t}-\alpha_{t}\vg_{t})-\vw^{\ast}\|^{2} =\|(\ov{w}_{t}-\alpha_{t}\ov{g}_{t}-\vw^{\ast})-\alpha_{t}(\vg_{t}-\ov{g}_{t})\|^{2}\\
	& = \underbrace{\|\ov{w}_{t}-\vw^{\ast}-\alpha_{t}\ov{g}_{t}\|^{2}}_{A_1} + \underbrace{2\alpha_{t}\langle\ov{w}_{t}-\vw^{\ast}-\alpha_{t}\ov{g}_{t},\ov{g}_{t}-\vg_{t}\rangle}_{A_2} + \underbrace{\alpha_{t}^{2}\|\vg_{t}-\ov{g}_{t}\|^{2}}_{A_3}
	\end{align*}
	where we denote: 
	\begin{align*}
	A_{1} & =\|\ov{w}_{t}-\vw^{\ast}-\alpha_{t}\ov{g}_{t}\|^{2}\\
	A_{2} & =2\alpha_{t}\langle\ov{w}_{t}-\vw^{\ast}-\alpha_{t}\ov{g}_{t},\ov{g}_{t}-\vg_{t}\rangle\\
	A_{3} & =\alpha_{t}^{2}\|\vg_{t}-\ov{g}_{t}\|^{2}
	\end{align*}
	By definition of $\vg_{t}$ and $\ov{g}_{t}$ (see \eq{\ref{eq:gradient}}), we have $\mathbb{E}A_{2}=0$.
	For $A_{3}$, we have the following upper bound (see Lemma~\ref{lem:bgv}):
	\begin{align*}
	\alpha_{t}^{2}\mathbb{E}\|\vg_{t}-\ov{g}_{t}\|^{2} \leq\alpha_{t}^{2}\sum_{k=1}^{N}p_{k}^{2}\sigma_{k}^{2}
	\end{align*} 
	
	Next we bound $A_{1}$: 
	\begin{align*}
	\|\ov{w}_{t}-\vw^{\ast}-\alpha_{t}\ov{g}_{t}\|^{2} & =\|\ov{w}_{t}-\vw^{\ast}\|^{2}+2\langle\ov{w}_{t}-\vw^{\ast},-\alpha_{t}\ov{g}_{t}\rangle+\|\alpha_{t}\ov{g}_{t}\|^{2}
	\end{align*}
	and we will show that the third term $\|\alpha_{t}\ov{g}_{t}\|^{2}$
	can be canceled by an upper bound of the second term, which is one of major improvement comparing to prior art~\cite{li2019convergence}.
	
	The upper bound of second term can be derived as follows, 	using the strong convexity and $L$-smoothness of $F_{k}$:
	\begin{align*}
	& -2\alpha_{t}\langle\ov{w}_{t}-\vw^{\ast},\ov{g}_{t}\rangle\\
	=& -2\alpha_{t}\sum_{k=1}^{N}p_{k}\langle\ov{w}_{t}-\vw^{\ast},\nabla F_{k}(\vw_{t}^{k})\rangle\\
	=& -2\alpha_{t}\sum_{k=1}^{N}p_{k}\langle\ov{w}_{t}-\vw_{t}^{k},\nabla F_{k}(\vw_{t}^{k})\rangle-2\alpha_{t}\sum_{k=1}^{N}p_{k}\langle \vw_{t}^{k}-\vw^{\ast},\nabla F_{k}(\vw_{t}^{k})\rangle\\
	\leq&-2\alpha_{t}\sum_{k=1}^{N}p_{k}\langle\ov{w}_{t}-\vw_{t}^{k},\nabla F_{k}(\vw_{t}^{k})\rangle+2\alpha_{t}\sum_{k=1}^{N}p_{k}(F_{k}(\vw^{\ast})-F_{k}(\vw_{t}^{k}))-\alpha_{t}\mu\sum_{k=1}^{N}p_{k}\|\vw_{t}^{k}-\vw^{\ast}\|^{2}\\
	\leq& 2\alpha_{t}\sum_{k=1}^{N}p_{k}\left[F_{k}(\vw_{t}^{k})-F_{k}(\ov{w}_{t})+\frac{L}{2}\|\ov{w}_{t}-\vw_{t}^{k}\|^{2}+F_{k}(\vw^{\ast})-F_{k}(\vw_{t}^{k})\right]-\alpha_{t}\mu\|\sum_{k=1}^{N}p_{k}\vw_{t}^{k}-\vw^{\ast}\|^{2}\\
	=& \alpha_{t}L\sum_{k=1}^{N}p_{k}\|\ov{w}_{t}-\vw_{t}^{k}\|^{2}+2\alpha_{t}\sum_{k=1}^{N}p_{k}\left[F_{k}(\vw^{\ast})-F_{k}(\ov{w}_{t})\right]-\alpha_{t}\mu\|\ov{w}_{t}-\vw^{\ast}\|^{2}
	\end{align*}
	We record the bound we have obtained so far, as it will also be used in the proof for convex case: 
	\begin{align*}
	 \mathbb{E}\|\ov{w}_{t+1}-\vw^{\ast}\|^{2}
	\leq & \mathbb{E}(1-\mu\alpha_{t})\|\ov{w}_{t}-\vw^{\ast}\|^{2}+\alpha_{t}L\sum_{k=1}^{N}p_{k}\|\ov{w}_{t}-\vw_{t}^{k}\|^{2}\\
	 & +2\alpha_{t}\sum_{k=1}^{N}p_{k}\left[F_{k}(\vw^{\ast})-F_{k}(\ov{w}_{t})\right]+\alpha_{t}^{2}\sum_{k=1}^{N}p_{k}^{2}\sigma_{k}^{2}+\alpha_{t}^{2}\|\ov{g}_{t}\|^{2} \numberthis \label{eq:common one step}
	\end{align*}
	For the term $2\alpha_{t}\sum_{k=1}^{N}p_{k}\left[F_{k}(\vw^{\ast})-F_{k}(\ov{w}_{t})\right]$, which is negative, we can ignore it, but this
	yields a suboptimal bound that fails to provide the desired linear
	speedup. Instead, we upper bound it using the following derivation:
	\begin{align*}
	& 2\alpha_{t}\sum_{k=1}^{N}p_{k}\left[F_{k}(\vw^{\ast})-F_{k}(\ov{w}_{t})\right]\\
	\leq& 2\alpha_{t}\left[F(\ov{w}_{t+1})-F(\ov{w}_{t})\right]\\
	\leq& 2\alpha_{t}\mathbb{E}\langle\nabla F(\ov{w}_{t}),\ov{w}_{t+1}-\ov{w}_{t}\rangle+\alpha_{t}L\mathbb{E}\|\ov{w}_{t+1}-\ov{w}_{t}\|^{2}\\
	 =& -2\alpha_{t}^{2}\mathbb{E}\langle\nabla F(\ov{w}_{t}),\vg_{t}\rangle+\alpha_{t}^{3}L\mathbb{E}\|\vg_{t}\|^{2}\\
	=&-2\alpha_{t}^{2}\mathbb{E}\langle\nabla F(\ov{w}_{t}),\ov{g}_{t}\rangle+\alpha_{t}^{3}L\mathbb{E}\|\vg_{t}\|^{2}\\
	=&-\alpha_{t}^{2}\left[\|\nabla F(\ov{w}_{t})\|^{2}+\|\ov{g}_{t}\|^{2}-\|\nabla F(\ov{w}_{t})-\ov{g}_{t}\|^{2}\right]+\alpha_{t}^{3}L\mathbb{E}\|\vg_{t}\|^{2}\\
	=&-\alpha_{t}^{2}\left[\|\nabla F(\ov{w}_{t})\|^{2}+\|\ov{g}_{t}\|^{2}-\|\nabla F(\ov{w}_{t})-\sum_{k}p_{k}\nabla F(\vw_{t}^{k})\|^{2}\right]+\alpha_{t}^{3}L\mathbb{E}\|\vg_{t}\|^{2}\\
   \leq&-\alpha_{t}^{2}\left[\|\nabla F(\ov{w}_{t})\|^{2}+\|\ov{g}_{t}\|^{2}-\sum_{k}p_{k}\|\nabla F(\ov{w}_{t})-\nabla F(\vw_{t}^{k})\|^{2}\right]+\alpha_{t}^{3}L\mathbb{E}\|\vg_{t}\|^{2}\\
	\leq&-\alpha_{t}^{2}\left[\|\nabla F(\ov{w}_{t})\|^{2}+\|\ov{g}_{t}\|^{2}-L^{2}\sum_{k}p_{k}\|\ov{w}_{t}-\vw_{t}^{k}\|^{2}\right]+\alpha_{t}^{3}L\mathbb{E}\|\vg_{t}\|^{2}\\
	\leq&-\alpha_{t}^{2}\|\ov{g}_{t}\|^{2}+\alpha_{t}^{2}L^{2}\sum_{k}p_{k}\|\ov{w}_{t}-\vw_{t}^{k}\|^{2}+\alpha_{t}^{3}L\mathbb{E}\|\vg_{t}\|^{2}-\alpha_{t}^{2}\|\nabla F(\ov{w}_{t})\|^{2}
	\end{align*}
	where we have used the smoothness of $F$ twice. 
	
	Note that the term $-\alpha_{t}^{2}\|\ov{g}_{t}\|^{2}$ exactly
	cancels the $\alpha_{t}^{2}\|\ov{g}_{t}\|^{2}$ in the bound in \eq{\ref{eq:common one step}}, so that plugging in the bound for $-2\alpha_{t}\langle\ov{w}_{t}-\vw^{\ast},\ov{g}_{t}\rangle$,
	we have so far proved 
	\begin{align*}
	\mathbb{E}\|\ov{w}_{t+1}-\vw^{\ast}\|^{2} & \leq\mathbb{E}(1-\mu\alpha_{t})\|\ov{w}_{t}-\vw^{\ast}\|^{2}+\alpha_{t}L\sum_{k=1}^{N}p_{k}\|\ov{w}_{t}-\vw_{t}^{k}\|^{2}+\alpha_{t}^{2}\sum_{k=1}^{N}p_{k}^{2}\sigma_{k}^{2}\\
	& +\alpha_{t}^{2}L^{2}\sum_{k=1}^{N}p_{k}\|\ov{w}_{t}-\vw_{t}^{k}\|^{2}+\alpha_{t}^{3}L\mathbb{E}\|\vg_{t}\|^{2}-\alpha_{t}^{2}\|\nabla F(\ov{w}_{t})\|^{2} \numberthis \label{eq:common recursion}
	\end{align*}
	Under Assumption~\ref{ass:subgrad2}, we have $\mathbb{E}\|\vg_{t}\|^{2}\leq G^{2}$. Furthermore, we can check that our choice of $\alpha_t$ satisfies $\alpha_t$ is non-increasing and $\alpha_t \leq 2\alpha_{t+E}$, so we may plug in the bound $\mathbb{E}\sum_{k=1}^{N}p_{k}\|\ov{w}_{t}-\vw_{t}^{k}\|^{2} \leq 4E^{2}\alpha_{t}^{2}G^{2}$ to the above inequality (see Lemma~\ref{lem:bdw}).
	
	Therefore, we can conclude that, with $\nu_{\max}:=N\cdot\max_{k}p_{k}$ and $\nu_{\min}:=N\cdot\min_{k}p_{k}$, 
	\begin{align*}
	& \mathbb{E}\|\ov{w}_{t+1}-\vw^{\ast}\|^{2}\\
	\leq& \mathbb{E}(1-\mu\alpha_{t})\|\ov{w}_{t}-\vw^{\ast}\|^{2}+4E^{2}L\alpha_{t}^{3}G^{2}+4E^{2}L^{2}\alpha_{t}^{4}G^{2}+\alpha_{t}^{2}\sum_{k=1}^{N}p_{k}^{2}\sigma_{k}^{2}+\alpha_{t}^{3}LG^{2}\\
	= &\mathbb{E}(1-\mu\alpha_{t})\|\ov{w}_{t}-\vw^{\ast}\|^{2}+4E^{2}L\alpha_{t}^{3}G^{2}+4E^{2}L^{2}\alpha_{t}^{4}G^{2}+\alpha_{t}^{2}\frac{1}{N}\sum_{k=1}^{N}(p_{k}N)p_k\sigma_{k}^{2}+\alpha_{t}^{3}LG^{2}\\
	\leq &\mathbb{E}(1-\mu\alpha_{t})\|\ov{w}_{t}-\vw^{\ast}\|^{2}+4E^{2}L\alpha_{t}^{3}G^{2}+4E^{2}L^{2}\alpha_{t}^{4}G^{2}+\alpha_{t}^{2}\frac{1}{N}\nu_{\max}\sum_{k=1}^{N}p_k\sigma_{k}^{2}+\alpha_{t}^{3}LG^{2}\\
	\leq &\mathbb{E}(1-\mu\alpha_{t})\|\ov{w}_{t}-\vw^{\ast}\|^{2}+6E^{2}L\alpha_{t}^{3}G^{2}+\alpha_{t}^{2}\frac{1}{N}\nu_{\max}\sigma^{2}
	\end{align*}
	where in the last inequality we use $\sigma^2=\sum_{k=1}^{N}p_k\sigma_{k}^{2}$, and that by construction $\alpha_{t}$
	satisfies $L\alpha_{t}\leq\frac{1}{8}$. 
\end{proof}

One may ask whether the dependence on $E$ in the term $\frac{\kappa E^{2}G^{2}/\mu}{KT}$
can be removed, or equivalently whether $\sum_{k}p_{k}\|\mathbf{w}_{t}^{k}-\overline{\mathbf{w}}_{t}\|^{2}=\mathcal{O}(1/T^{2})$
can be independent of $E$. We provide a simple counterexample that
shows that this is not possible in general. 
\begin{proposition} \label{prop:tight}
	There exists a dataset such that if $E=\mathcal{O}(T^{\beta})$ for
	any $\beta>0$ then $\sum_{k}p_{k}\|\mathbf{w}_{t}^{k}-\overline{\mathbf{w}}_{t}\|^{2}=\Omega(\frac{1}{T^{2-2\beta}})$
	.
\end{proposition}
\begin{proof}
	Suppose that we have an even number of devices and each $F_{k}(\mathbf{w})=\frac{1}{n_{k}}\sum_{j=1}^{n_{k}}(\mathbf{x}_{k}^{j}-\mathbf{w})^{2}$
	contains data points $\mathbf{x}_{k}^{j}=\mathbf{w}^{\ast,k}$, with
	$n_{k}\equiv n$. Moreover, the $\mathbf{w}{}^{\ast,k}$'s come in
	pairs around the origin. As a result, the global objective $F$ is
	minimized at $\mathbf{w}^{\ast}=0$. Moreover, if we start from $\overline{\mathbf{w}}_{0}=0$,
	then by design of the dataset the updates in local steps exactly cancel
	each other at each iteration, resulting in $\overline{\mathbf{w}}_{t}=0$
	for all $t$. On the other hand, if $E=T^{\beta}$, then starting
	from any $t=\mathcal{O}(T)$ with constant step size $\mathcal{O}(\frac{1}{T})$,
	after $E$ iterations of local steps, the local parameters are updated
	towards $\mathbf{w}^{\ast,k}$ with $\|\mathbf{w}_{t+E}^{k}\|^{2}=\Omega((T^{\beta}\cdot\frac{1}{T})^{2})=\Omega(\frac{1}{T^{2-2\beta}})$.
	This implies that 
	\begin{align*}
	\sum_{k}p_{k}\|\mathbf{w}_{t+E}^{k}-\overline{\mathbf{w}}_{t+E}\|^{2} & =\sum_{k}p_{k}\|\mathbf{w}_{t+E}^{k}\|^{2}\\
	& =\Omega(\frac{1}{T^{2-2\beta}})
	\end{align*}
	which is at a slower rate than $\frac{1}{T^{2}}$ for any $\beta>0$.
	Thus the sampling variance $\mathbb{E}\|\overline{\mathbf{w}}_{t+1}-\overline{\mathbf{v}}_{t+1}\|^{2}=\Omega(\sum_{k}p_{k}\mathbb{E}\|\mathbf{w}_{t+1}^{k}-\overline{\mathbf{w}}_{t+1}\|^{2})$
	decays at a slower rate than $\frac{1}{T^{2}}$, resulting in a convergence
	rate slower than $\mathcal{O}(\frac{1}{T})$ with partial participation. 
\end{proof}

\subsection{Convex Smooth Objectives}
\label{sec:nasgdscvxsmth}
In this section we provide the proof of the convergence result for FedAvg with convex and smooth objectives. The key step is a one step progress result analogous to that in the strongly convex case, and their proofs share identical components as well. 
\begin{lemma} [\textbf{One step progress, convex case}]
Let $\overline{\mathbf{w}}_{t}=\sum_{k=1}^{N}p_{k}\mathbf{w}_{t}^{k}$ in FedAvg. Under assumptions~\ref{ass:lsmooth},\ref{ass:boundedvariance},\ref{ass:subgrad2}, the following bound holds for all $t$:
\begin{align*}
	\mathbb{E}\|\ov{w}_{t+1}-\vw^{\ast}\|^{2}+\alpha_{t}(F(\ov{w}_{t})-F(\vw^{\ast})) & \leq \mathbb{E}\|\ov{w}_{t}-\vw^{\ast}\|^{2}+\alpha_{t}^{2}\frac{1}{N}\nu_{\max}\sigma^{2}+6\alpha_{t}^{3}E^{2}LG^{2}
	\end{align*}
	\label{lem:cvxoner}
\end{lemma}
\begin{proof}
    The first part of the proof follows directly from \eq{\ref{eq:common one step}} in the proof of Lemma \ref{lem:scvxoner}. Setting $\mu=0$ in \eq{\ref{eq:common one step}} (since we are in the convex setting instead of strongly convex), we obtain 
    \begin{align*}
	\|\ov{w}_{t+1}-\vw^{\ast}\|^{2} & \leq\|\ov{w}_{t}-\vw^{\ast}\|^{2}+\alpha_{t}L\sum_{k=1}^{N}p_{k}\|\ov{w}_{t}-\vw_{t}^{k}\|^{2} \\ 
	& +2\alpha_{t}\sum_{k=1}^{N}p_{k}\left[F_{k}(\vw^{\ast})-F_{k}(\ov{w}_{t})\right]+\alpha_{t}^{2}\|\ov{g}_{t}\|^{2}+\alpha_{t}^{2}\sum_{k=1}^{N}p_{k}^{2}\sigma_{k}^{2}
	\end{align*}
	The difference of this bound with that in the strongly convex case
	is that we no longer have a contraction factor of $1-\mu\alpha_t$ in front of $\|\ov{w}_{t}-\vw^{\ast}\|^{2}$.
	In the strongly convex case, we were able to cancel $\alpha_{t}^{2}\|\ov{g}_{t}\|^{2}$
	with $2\alpha_{t}\sum_{k=1}^{N}p_{k}\left[F_{k}(\vw^{\ast})-F_{k}(\ov{w}_{t})\right]$
	and obtain only lower order terms. In the convex case, we use a different
	strategy and preserve $\sum_{k=1}^{N}p_{k}\left[F_{k}(\vw^{\ast})-F_{k}(\ov{w}_{t})\right]$
	in order to obtain the desired optimality gap. 
	
	More precisely, we have
	\begin{align*}
	\|\ov{g}_{t}\|^{2} & =\|\sum_{k}p_{k}\nabla F_{k}(\vw_{t}^{k})\|^{2}\\
	& =\|\sum_{k}p_{k}\nabla F_{k}(\vw_{t}^{k})-\sum_{k}p_{k}\nabla F_{k}(\ov{w}_{t})+\sum_{k}p_{k}\nabla F_{k}(\ov{w}_{t})\|^{2}\\
	& \leq2\|\sum_{k}p_{k}\nabla F_{k}(\vw_{t}^{k})-\sum_{k}p_{k}\nabla F_{k}(\ov{w}_{t})\|^{2}+2\|\sum_{k}p_{k}\nabla F_{k}(\ov{w}_{t})\|^{2}\\
	& \leq2L^{2}\sum_{k}p_{k}\|\vw_{t}^{k}-\ov{w}_{t}\|^{2}+2\|\sum_{k}p_{k}\nabla F_{k}(\ov{w}_{t})\|^{2}\\
	& =2L^{2}\sum_{k}p_{k}\|\vw_{t}^{k}-\ov{w}_{t}\|^{2}+2\|\nabla F(\ov{w}_{t})\|^{2}
	\end{align*}
	using $\nabla F(\vw^{\ast})=0$. Now using the $L$ smoothness of $F$,
	we have $\|\nabla F(\ov{w}_{t})\|^{2}\leq2L(F(\ov{w}_{t})-F(\vw^{\ast}))$,
	so that 
	\begin{align*}
	& \|\ov{w}_{t+1}-\vw^{\ast}\|^{2}\\
	\leq & \|\ov{w}_{t}-\vw^{\ast}\|^{2}+\alpha_{t}L\sum_{k=1}^{N}p_{k}\|\ov{w}_{t}-\vw_{t}^{k}\|^{2}+2\alpha_{t}\sum_{k=1}^{N}p_{k}\left[F_{k}(\vw^{\ast})-F_{k}(\ov{w}_{t})\right]\\
	& +2\alpha_{t}^{2}L^{2}\sum_{k}p_{k}\|\vw_{t}^{k}-\ov{w}_{t}\|^{2}+4\alpha_{t}^{2}L(F(\ov{w}_{t})-F(\vw^{\ast}))+\alpha_{t}^{2}\sum_{k=1}^{N}p_{k}^{2}\sigma_{k}^{2}\\
	= & \|\ov{w}_{t}-\vw^{\ast}\|^{2}+(2\alpha_{t}^{2}L^{2}+\alpha_{t}L)\sum_{k=1}^{N}p_{k}\|\ov{w}_{t}-\vw_{t}^{k}\|^{2}+\alpha_{t}\sum_{k=1}^{N}p_{k}\left[F_{k}(\vw^{\ast})-F_{k}(\ov{w}_{t})\right] \\ 
	 & +\alpha_{t}^{2}\sum_{k=1}^{N}p_{k}^{2}\sigma_{k}^{2}
	 +\alpha_{t}(1-4\alpha_{t}L)(F(\vw^{\ast})-F(\ov{w}_{t}))
	\end{align*}
	Since $F(\vw^{\ast})\leq F(\ov{w}_{t})$, as long as $4\alpha_{t}L\leq1$,
	we can ignore the last term, and rearrange the inequality to obtain
	\begin{align*}
	& \|\ov{w}_{t+1}-\vw^{\ast}\|^{2}+\alpha_{t}(F(\ov{w}_{t})-F(\vw^{\ast}))\\
 \leq & \|\ov{w}_{t}-\vw^{\ast}\|^{2}+(2\alpha_{t}^{2}L^{2}+\alpha_{t}L)\sum_{k=1}^{N}p_{k}\|\ov{w}_{t}-\vw_{t}^{k}\|^{2}+\alpha_{t}^{2}\sum_{k=1}^{N}p_{k}^{2}\sigma_{k}^{2}\\
	\leq & \|\ov{w}_{t}-\vw^{\ast}\|^{2}+\frac{3}{2}\alpha_{t}L\sum_{k=1}^{N}p_{k}\|\ov{w}_{t}-\vw_{t}^{k}\|^{2}+\alpha_{t}^{2}\sum_{k=1}^{N}p_{k}^{2}\sigma_{k}^{2}
	\end{align*}
	
	The same argument as before yields $\mathbb{E}\sum_{k=1}^{N}p_{k}\|\ov{w}_{t}-\vw_{t}^{k}\|^{2}\leq4E^{2}\alpha_{t}^{2}G^{2}$
	which gives 
	\begin{align*}
	\|\ov{w}_{t+1}-\vw^{\ast}\|^{2}+\alpha_{t}(F(\ov{w}_{t})-F(\vw^{\ast})) & \leq\|\ov{w}_{t}-\vw^{\ast}\|^{2}+\alpha_{t}^{2}\sum_{k=1}^{N}p_{k}^{2}\sigma_{k}^{2}+6\alpha_{t}^{3}E^{2}LG^{2}\\
	& \leq\|\ov{w}_{t}-\vw^{\ast}\|^{2}+\alpha_{t}^{2}\frac{1}{N}\nu_{\max}\sigma^{2}+6\alpha_{t}^{3}E^{2}LG^{2}
	\end{align*}
\end{proof}
With the one step progress result, we can now prove the convergence result in the convex setting, which we restate below.
\begin{thm}
	Under assumptions~\ref{ass:lsmooth},\ref{ass:boundedvariance},\ref{ass:subgrad2} and constant learning
	rate $\alpha_{t}=\mathcal{O}(\sqrt{\frac{N}{T}})$, FedAvg satisfies
	\begin{align*}
	\min_{t\leq T}F(\overline{\mathbf{w}}_{t})-F(\mathbf{w}^{\ast}) & =\mathcal{O}\left(\frac{\nu_{\max}\sigma^{2}}{\sqrt{NT}}+\frac{NE^{2}LG^{2}}{T}\right)
	\end{align*}
	with full participation, and with partial device participation with $K$ sampled devices at
	each communication round and learning rate $\alpha_{t}=\mathcal{O}(\sqrt{\frac{K}{T}})$,
	\begin{align*}
	\min_{t\leq T}F(\overline{\mathbf{w}}_{t})-F(\mathbf{w}^{\ast}) & =\mathcal{O}\left(\frac{\nu_{\max}\sigma^{2}}{\sqrt{KT}}+\frac{EG^{2}}{\sqrt{KT}}+\frac{KE^{2}LG^{2}}{T}\right)
	\end{align*}
\end{thm}

\begin{proof}
	We first prove the bound for full participation. Applying Lemma~\ref{lem:cvxoner}, we have
	\begin{align*}
	\|\ov{w}_{t+1}-\vw^{\ast}\|^{2}+\alpha_{t}(F(\ov{w}_{t})-F(\vw^{\ast})) & \leq\|\ov{w}_{t}-\vw^{\ast}\|^{2}+\alpha_{t}^{2}\frac{1}{N}\nu_{\max}\sigma^{2}+6\alpha_{t}^{3}E^{2}LG^{2}
	\end{align*}
	Summing the inequalities from $t=0$ to $t=T$, we obtain 
	\begin{align*}
	\sum_{t=0}^{T}\alpha_{t}(F(\ov{w}_{t})-F(\vw^{\ast})) & \leq\|\vw_{0}-\vw^{\ast}\|^{2}+\sum_{t=0}^{T}\alpha_{t}^{2}\cdot\frac{1}{N}\nu_{\max}\sigma^{2}+\sum_{t=0}^{T}\alpha_{t}^{3}\cdot6E^{2}LG^{2}
	\end{align*}
	so that
	\begin{align*}
	\min_{t\leq T}F(\ov{w}_{t})-F(\vw^{\ast}) & \leq\frac{1}{\sum_{t=0}^{T}\alpha_{t}}\left(\|\vw_{0}-\vw^{\ast}\|^{2}+\sum_{t=0}^{T}\alpha_{t}^{2}\cdot\frac{1}{N}\nu_{\max}\sigma^{2}+\sum_{t=0}^{T}\alpha_{t}^{3}\cdot6E^{2}LG^{2}\right)
	\end{align*}
	
	By setting the constant learning rate $\alpha_{t}\equiv\sqrt{\frac{N}{T}}$,
	we have 
\begin{align*}
&\min_{t\leq T}F(\ov{w}_{t})-F(\vw^{\ast}) \\
 & \leq\frac{1}{\sqrt{NT}}\cdot\|\vw_{0}-\vw^{\ast}\|^{2}+\frac{1}{\sqrt{NT}}T\cdot\frac{N}{T}\cdot\frac{1}{N}\nu_{\max}\sigma^{2}+\frac{1}{\sqrt{NT}}T(\sqrt{\frac{N}{T}})^{3}6E^{2}LG^{2}\\
& \leq\frac{1}{\sqrt{NT}}\cdot\|\vw_{0}-\vw^{\ast}\|^{2}+\frac{1}{\sqrt{NT}}T\cdot\frac{N}{T}\cdot\frac{1}{N}\nu_{\max}\sigma^{2}+\frac{N}{T}6E^{2}LG^{2}\\
& =(\|\vw_{0}-\vw^{\ast}\|^{2}+\nu_{\max}\sigma^{2})\frac{1}{\sqrt{NT}}+\frac{N}{T}6E^{2}LG^{2}\\
& =\mathcal{O}(\frac{\nu_{\max}\sigma^{2}}{\sqrt{NT}}+\frac{NE^{2}LG^{2}}{T})
\end{align*}
	
	\rebuttal{
	For partial participation, the one step progress bound in Lemma \ref{lem:cvxoner} is updated in a similar manner as the strongly convex case to incorporate the sampling variance. More precisely, with partial participation, we consider two consecutive communication rounds $t+1$ and $t+1+E$. Using $\mathbb{E}\ov{w}_{t+E+1}=\ov{v}_{t+E+1}$,
we can write 
\begin{align*}
\mathbb{E}\|{\ov{w}}_{t+E+1}-\vw^{\ast}\|^{2} & =\mathbb{E}\|\ov{w}_{t+E+1}-\ov{v}_{t+E+1}+\ov{v}_{t+E+1}-\vw^{\ast}\|^{2}\\
 & =\mathbb{E}\|\ov{v}_{t+E+1}-\vw^{\ast}\|^{2}+\mathbb{E}\|\ov{w}_{t+E+1}-\ov{v}_{t+E+1}\|^{2}
\end{align*}
	
	Again, the first term can be bounded by applying the one-step bound $E$ times and summing it up, giving
	\begin{align*}
	& \mathbb{E}\|\ov{v}_{t+E+1}-\vw^{\ast}\|^{2}+\sum_{t+1}^{t+E}\alpha_{t}(F(\ov{v}_{t})-F(\vw^{\ast}))\\  \leq & \mathbb{E}\|\ov{w}_{t+1}-\vw^{\ast}\|^{2}+\sum_{t+1}^{t+E}\alpha_{t}^{2}\frac{1}{N}\nu_{\max}\sigma^{2}+6\sum_{t+1}^{t+E}\alpha_{t}^{3}E^{2}LG^{2}
	\end{align*}
	
	The bound for $\mathbb{E}\|\ov{w}_{t+E+1}-\ov{v}_{t+E+1}\|^{2}$ for the two sampling schemes we consider is again provided in \eq{\ref{eq:partialsample}}, giving the following $E$-step progress bound
	\begin{align*}
	\mathbb{E}\|{\ov{w}}_{t+E+1}-\vw^{\ast}\|^{2} + \sum_{t+1}^{t+E}\alpha_{t}(F(\ov{v}_{t})-F(\vw^{\ast})) &
	   \leq\mathbb{E}\|\ov{w}_{t+1}-\vw^{\ast}\|^{2}+\sum_{t+1}^{t+E}\alpha_{t}^{2}\frac{1}{N}\nu_{\max}\sigma^{2}\\
&+6\sum_{t+1}^{t+E}\alpha_{t}^{3}E^{2}LG^{2} +\frac{4}{K}\alpha_{t+1}^{2}E^{2}G^{2}
	\end{align*}

	Summing up the above bounds $T/E$ times, 
	\begin{align*}
	& \min_{t\leq T}F(\ov{w}_{t})-F(\vw^{\ast})\\ \leq&\frac{1}{\sum_{t=0}^{T}\alpha_{t}}\left(\|\vw_{0}-\vw^{\ast}\|^{2}+\sum_{t=0}^{T}\alpha_{t}^{2}\frac{1}{N}\nu_{\max}\sigma^{2}+ \sum_{t=E,2E,\dots} \frac{4}{K}\alpha_{t+1}^{2}E^{2}G^{2} +\sum_{t=0}^{T}\alpha_{t}^{3}\cdot6E^{2}LG^{2}\right),
	\end{align*}
	so that with $\alpha_{t}=\sqrt{\frac{K}{T}}$, we have 
	\begin{align*}
	\min_{t\leq T}F(\ov{w}_{t})-F(\vw^{\ast}) & =\mathcal{O}(\frac{\nu_{\max}\sigma^{2}}{\sqrt{KT}}+\frac{EG^{2}}{\sqrt{KT}}+\frac{KE^{2}LG^{2}}{T}).
	\end{align*}
	}
\end{proof}

\section{Proof of Convergence Results for Nesterov Accelerated FedAvg}
\label{sec:app:Nesterovfedavg}
\subsection{Strongly Convex Smooth Objectives}
\label{sec:convexsmoothsgd}
Recall that the Nesterov accelerated FedAvg follows the updates 
\begin{align*}
\mathbf{v}_{t+1}^{k} & =\mathbf{w}_{t}^{k}-\alpha_{t}\mathbf{g}_{t,k}, \hspace{1em}
\mathbf{w}_{t+1}^{k} =\begin{cases}
\mathbf{v}_{t+1}^{k}+\beta_{t}(\mathbf{v}_{t+1}^{k}-\mathbf{v}_{t}^{k}) & \text{if }t+1\notin\mathcal{I}_{E},\\
\sum_{k \in \cS_{t+1}}q_k\left[\mathbf{v}_{t+1}^{k}+\beta_{t}(\mathbf{v}_{t+1}^{k}-\mathbf{v}_{t}^{k})\right] & \text{if }t+1\in\mathcal{I}_{E}.
\end{cases}
\end{align*}

The proofs of convergence results for Nesterov Accelerated FedAvg consists of components that are direct analogues of the FedAvg case. We first state these analogue results before proving the main theorem. Like before, the proofs of the lemmas are deferred to after the main proof. 

\begin{lemma}[\textbf{One step progress, Nesterov}] Let $\overline{\mathbf{v}}_{t}=\sum_{k=1}^{N}p_{k}\mathbf{v}_{t}^{k}$ in Nesterov accelerated FedAvg,
and suppose our functions satisfy Assumptions~\ref{ass:lsmooth},\ref{ass:stroncvx},\ref{ass:boundedvariance},\ref{ass:subgrad2}, and set step sizes $\alpha_{t}=\frac{6}{\mu}\frac{1}{t+\gamma}$,  $\beta_{t-1}=\frac{3}{14(t+\gamma)(1-\frac{6}{t+\gamma})\max\{\mu,1\}}$
	with $\gamma=\max\{32\kappa,E\}$ and $\kappa=\frac{L}{\mu}$, the updates of Nesterov accelerated FedAvg satisfy
\begin{align*}
\mathbb{E}\|\ov{v}_{t+1}-\vw^{\ast}\|^{2} & \leq\mathbb{E}(1-\mu\alpha_{t})(1+\beta_{t-1})^{2}\|\ov{v}_{t}-\vw^{\ast}\|^{2}+20E^{2}L\alpha_{t}^{3}G^{2}\\
&+(1-\alpha_{t}\mu)\beta_{t-1}^{2}\|(\ov{v}_{t-1}-\vw^{\ast})\|^{2}
+\alpha_{t}^{2}\frac{1}{N}\nu_{\max}\sigma^{2}\\
&+2\beta_{t-1}(1+\beta_{t-1})(1-\alpha_{t}\mu)\|\ov{v}_{t}-\vw^{\ast}\|\cdot\|\ov{v}_{t-1}-\vw^{\ast}\|.
\end{align*}
\label{lem:nest-scvxoner}
\end{lemma}
The one step progress result makes use of the same bound on the gradient variance in~Lemma~\ref{lem:bgv}, as well as a divergence bound analogous to Lemma~\ref{lem:bdw}, which we state below.
\begin{lemma}[\textbf{Bounding the divergence of $\vw_t^k$, Nesterov}]
Given Assumption~\ref{ass:subgrad2}, and assume that $\alpha_t$ is non-increasing, $\alpha_t \leq 2\alpha_{t+E}$, and $2\beta_{t-1}^{2}+2\alpha_{t}^{2}\leq1/2$ for all $t\geq 0$, $\overline{\mathbf{w}}_{t}=\sum_{k=1}^{N}p_{k}\mathbf{w}_{t}^{k}$ in Nesterov accelerated FedAvg satisfies
	\begin{align*}
	\mathbb{E}\left[\sum_{k=1}^{N}p_{k}\|\ov{w}_{t}-\vw_{t}^{k}\|^{2} \right]\leq16(E-1)^{2}\alpha_{t}^{2}G^{2}.
	\end{align*}
\label{lem:nest-bdw}
\end{lemma}

\begin{thm}
	Let $\overline{\mathbf{v}}_{T}=\sum_{k=1}^{N}p_{k}\mathbf{v}_{T}^{k}$ in Nesterov accelerated FedAvg
	and set learning rates $\alpha_{t}=\frac{6}{\mu}\frac{1}{t+\gamma}$,  $\beta_{t-1}=\frac{3}{14(t+\gamma)(1-\frac{6}{t+\gamma})\max\{\mu,1\}}$. Then under Assumptions~\ref{ass:lsmooth},\ref{ass:stroncvx},\ref{ass:boundedvariance},\ref{ass:subgrad2} with full device participation, 
	\begin{align*}
	\mathbb{E}F(\overline{\mathbf{v}}_{T})-F^{\ast}=\mathcal{O}\left(\frac{\kappa\nu_{\max}\sigma^{2}/\mu}{NT}+\frac{\kappa^{2}E^{2}G^{2}/\mu}{T^{2}}\right),
	\end{align*}
	and with partial device participation with $K$ sampled devices at
	each communication round, 
	\begin{align*}
	\mathbb{E}F(\overline{\mathbf{v}}_{T})-F^{\ast}=\mathcal{O}\left(\frac{\kappa\nu_{\max}\sigma^{2}/\mu}{NT}+\frac{\kappa E^{2}G^{2}/\mu}{KT}+\frac{\kappa^{2}E^{2}G^{2}/\mu}{T^{2}}\right).
	\end{align*}
\end{thm}
\textbf{}%

\begin{proof}
We first prove the result for full participation. Applying the one step progress bound in Lemma~\ref{lem:nest-scvxoner}, we have
\begin{align*}
\mathbb{E}\|\ov{v}_{t+1}-\vw^{\ast}\|^{2} & \leq\mathbb{E}(1-\mu\alpha_{t})(1+\beta_{t-1})^{2}\|\ov{v}_{t}-\vw^{\ast}\|^{2}+20E^{2}L\alpha_{t}^{3}G^{2}\\
&+(1-\alpha_{t}\mu)\beta_{t-1}^{2}\|(\ov{v}_{t-1}-\vw^{\ast})\|^{2} +\alpha_{t}^{2}\frac{1}{N}\nu_{\max}\sigma^{2}\\
&+2\beta_{t-1}(1+\beta_{t-1})(1-\alpha_{t}\mu)\|\ov{v}_{t}-\vw^{\ast}\|\cdot\|\ov{v}_{t-1}-\vw^{\ast}\|.
\end{align*}
Recall that we require $\alpha_{t_{0}}\leq2\alpha_{t}$ for any
$t-t_{0}\leq E-1$, $L\alpha_{t}\leq\frac{1}{5}$, and $2\beta_{t-1}^{2}+2\alpha_{t}^{2}\leq1/2$ in order for Lemmas~\ref{lem:nest-bdw} and~\ref{lem:nest-scvxoner} to hold,
which we can check by definition of $\alpha_{t}$ and
$\beta_{t}$.\\
 We show next that $\mathbb{E}\|\ov{v}_{t}-\vw^{\ast}\|^{2}=\mathcal{O}(\frac{\nu_{\max}\sigma^{2}/\mu}{tN}+\frac{E^{2}LG^{2}/\mu^2}{t^{2}})$ 
by induction. Assume that we have shown 
\begin{align*}
\mathbb{E}\|\ov{v}_{t}-\vw^{\ast}\|^{2} & \leq b(C\alpha_{t}^{2}+D\alpha_{t})
\end{align*}
for all iterations until $t$, where $C=20E^{2}LG^{2}$, $D=\frac{1}{N}\nu_{\max}\sigma^{2}$,
and $b$ is some constant to be chosen later. For step sizes recall that we choose $\alpha_{t}=\frac{6}{\mu}\frac{1}{t+\gamma}$
and $\beta_{t-1}=\frac{3}{14(t+\gamma)(1-\frac{6}{t+\gamma})\max\{\mu,1\}}$
where $\gamma=\max\{32\kappa,E\}$, so that $\beta_{t-1}\leq\alpha_{t}$
and 
\begin{align*}
(1-\mu\alpha_{t})(1+14\beta_{t-1}) & \leq(1-\frac{6}{t+\gamma})(1+\frac{3}{(t+\gamma)(1-\frac{6}{t+\gamma})})\\
& =1-\frac{6}{t+\gamma}+\frac{3}{t+\gamma}=1-\frac{3}{t+\gamma}=1-\frac{\mu\alpha_{t}}{2}
\end{align*}

Moreover, $\mathbb{E}\|\ov{v}_{t-1}-\vw^{\ast}\|^{2}\leq b(C\alpha_{t-1}^{2}+D\alpha_{t-1})\leq 4b(C\alpha_{t}^{2}+D\alpha_{t})$
with the chosen step sizes.
Therefore the bound for $\mathbb{E}\|\ov{v}_{t+1}-\vw^{\ast}\|^{2}$
can be further simplified with 
\begin{align*}
2\beta_{t-1}(1+\beta_{t-1})(1-\alpha_{t}\mu)\mathbb{E}\|\ov{v}_{t}-\vw^{\ast}\|\|\ov{v}_{t-1}-\vw^{\ast}\| & \leq4\beta_{t-1}(1+\beta_{t-1})(1-\alpha_{t}\mu) b(C\alpha_{t}^{2}+D\alpha_{t})
\end{align*}
and 
\begin{align*}
(1-\alpha_{t}\mu)\beta_{t-1}^{2}\mathbb{E}\|(\ov{v}_{t-1}-\vw^{\ast})\|^{2} & \leq4(1-\alpha_{t}\mu)\beta_{t-1}^{2}\cdot b(C\alpha_{t}^{2}+D\alpha_{t})
\end{align*}
so that
\begin{align*}
\mathbb{E}\|\ov{v}_{t+1}-\vw^{\ast}\|^{2} & \leq(1-\mu\alpha_{t})((1+\beta_{t-1})^{2}+4\beta_{t-1}(1+\beta_{t-1})+4\beta_{t-1}^{2})\cdot b(C\alpha_{t}^{2}+D\alpha_{t})\\
& +20E^{2}L\alpha_{t}^{3}G^{2}+\alpha_{t}^{2}\frac{1}{N}\nu_{\max}\sigma^{2}\\
& \leq\mathbb{E}(1-\mu\alpha_{t})(1+14\beta_{t-1})\cdot b(C\alpha_{t}^{2}+D\alpha_{t})+20E^{2}L\alpha_{t}^{3}G^{2}+\alpha_{t}^{2}\frac{1}{N}\nu_{\max}\sigma^{2}\\
& \leq b(1-\frac{\mu\alpha_{t}}{2})(C\alpha_{t}^{2}+D\alpha_{t})+C\alpha_{t}^{3}+D\alpha_{t}^{2}\\
& =(b(1-\frac{\mu\alpha_{t}}{2})+\alpha_{t})\alpha_{t}^{2}C+(b(1-\frac{\mu\alpha_{t}}{2})+\alpha_{t})\alpha_{t}D
\end{align*}
and so it remains to choose $b$ such that 
\begin{align*}
(b(1-\frac{\mu\alpha_{t}}{2})+\alpha_{t})\alpha_{t} & \leq b\alpha_{t+1}\\
(b(1-\frac{\mu\alpha_{t}}{2})+\alpha_{t})\alpha_{t}^{2} & \leq b\alpha_{t+1}^{2}
\end{align*}
from which we can conclude $\mathbb{E}\|\ov{v}_{t+1}-\vw^{\ast}\|^{2}\leq\alpha_{t+1}^{2}C+\alpha_{t+1}D$.

With $b=\frac{6}{\mu}$, we have
\begin{align*}
(b(1-\frac{\mu\alpha_{t}}{2})+\alpha_{t})\alpha_{t} & =(b(1-(\frac{3}{t+\gamma})+\frac{6}{\mu(t+\gamma)})\frac{6}{\mu(t+\gamma)}\\
& =(b\frac{t+\gamma-3}{t+\gamma}+\frac{6}{\mu(t+\gamma)})\frac{6}{\mu(t+\gamma)}\\
& \leq b(\frac{t+\gamma-1}{t+\gamma})\frac{6}{\mu(t+\gamma)}\\
& \leq b\frac{6}{\mu(t+\gamma+1)}=b\alpha_{t+1}
\end{align*}
where we have used $\frac{t+\gamma-1}{(t+\gamma)^{2}}\leq\frac{1}{t+\gamma+1}$.

Similarly 
\begin{align*}
(b(1-\frac{\mu\alpha_{t}}{2})+\alpha_{t})\alpha_{t}^{2} & =(b(1-(\frac{3}{t+\gamma})+\frac{6}{\mu(t+\gamma)})(\frac{6}{\mu(t+\gamma)})^{2}\\
& =(b\frac{t+\gamma-3}{t+\gamma}+\frac{6}{\mu(t+\gamma)})(\frac{6}{\mu(t+\gamma)})^{2}\\
& =b(\frac{t+\gamma-2}{t+\gamma})(\frac{6}{\mu(t+\gamma)})^{2}\\
& \leq b\frac{36}{\mu^{2}(t+\gamma+1)^{2}}=b\alpha_{t+1}^{2}
\end{align*}
where we have used $\frac{t+\gamma-2}{(t+\gamma)^{3}}\leq\frac{1}{(t+\gamma+1)^{2}}$.

Finally, to ensure $\|\vv_{0}-\vw^{\ast}\|^{2}\leq b(C\alpha_{0}^{2}+D\alpha_{0})$,
we can rescale $b$ by $c\|\vv_{0}-\vw^{\ast}\|^{2}$ for some $c.$ It
follows that $\mathbb{E}\|\ov{v}_{t}-\vw^{\ast}\|^{2}\leq b(C\alpha_{t}^{2}+D\alpha_{t})$
for all $t\geq0$. Using the $L$-smooothness of $F$,
\begin{align*}
\mathbb{E}(F(\ov{v}_{T}))-F^{\ast} & =\mathbb{E}(F(\ov{v}_{T})-F(\vw^{\ast}))\\
& \leq\frac{L}{2}\mathbb{E}\|\ov{v}_{T}-\vw^{\ast}\|^{2}\leq\frac{L}{2}c\|\vv_{0}-\vw^{\ast}\|^{2}\frac{6}{\mu}(D\alpha_{T}+C\alpha_{T}^{2})\\
& =3c\|\vv_{0}-\vw^{\ast}\|^{2}\kappa(D\alpha_{T}+C\alpha_{T}^{2})\\
& \leq3c\|\vv_{0}-\vw^{\ast}\|^{2}\kappa\left[\frac{6}{\mu(T+\gamma)}\cdot\frac{1}{N}\nu_{\max}\sigma^{2}+20E^{2}LG^{2}\cdot(\frac{6}{\mu(T+\gamma)})^{2}\right]\\
& =\mathcal{O}(\frac{\kappa}{\mu}\frac{1}{N}\nu_{\max}\sigma^{2}\cdot\frac{1}{T}+\frac{\kappa^{2}}{\mu}E^{2}G^{2}\cdot\frac{1}{T^{2}})
\end{align*}

With partial participation, the same argument as in the FedAvg case in Theorem \ref{thm:SGD_scvx} by adding a term for sampling error every $E$ steps yields
\begin{align*}
\mathbb{E}F(\ov{w}_{T})-F^{\ast}=\mathcal{O}(\frac{\kappa\nu_{\max}\sigma^{2}/\mu}{NT}+\frac{\kappa EG^{2}/\mu}{KT}+\frac{\kappa^{2}E^{2}G^{2}/\mu}{T^{2}})
\end{align*}
\end{proof}

\subsubsection{Deferred Proofs of Key Lemmas}

\begin{proof}[Proof of lemma~\ref{lem:nest-bdw}]
	The proof of bound for $\mathbb{E}\sum_{k=1}^{N}p_{k}\|\ov{w}_{t}-\vw_{t}^{k}\|^{2}$ in the Nesterov accelerated FedAvg follows a similar logic as in Lemma~\ref{lem:bdw}, but requires extra reasoning.  
Since communication is done every $E$ steps, for any $t\geq0$, we
can find a $t_{0}\leq t$ such that $t-t_{0}\leq E-1$ and $w_{t_{0}}^{k}=\ov{w}_{t_{0}}$for
all $k$. Moreover, using $\alpha_{t}$ is non-increasing, $\alpha_{t_{0}}\leq2\alpha{}_{t}$,
and $\beta_{t}\leq\alpha_{t}$ for any $t-t_{0}\leq E-1$, we have
\begin{align*}
\mathbb{E}\sum_{k=1}^{N}p_{k}\|\ov{w}_{t}-\vw_{t}^{k}\|^{2} & =\mathbb{E}\sum_{k=1}^{N}p_{k}\|\vw_{t}^{k}-\ov{w}_{t_{0}}-(\ov{w}_{t}-\ov{w}_{t_{0}})\|^{2}\\
& \leq\mathbb{E}\sum_{k=1}^{N}p_{k}\|\vw_{t}^{k}-\ov{w}_{t_{0}}\|^{2}\\
& =\mathbb{E}\sum_{k=1}^{N}p_{k}\|\vw_{t}^{k}-\vw_{t_{0}}^{k}\|^{2}\\
& =\mathbb{E}\sum_{k=1}^{N}p_{k}\|\sum_{i=t_{0}}^{t-1}\beta_{i}(\vv_{i+1}^{k}-\vv_{i}^{k})-\sum_{i=t_{0}}^{t-1}\alpha_{i}\vg_{i,k}\|^{2}\\
& \leq2\sum_{k=1}^{N}p_{k}\mathbb{E}\sum_{i=t_{0}}^{t-1}(E-1)\alpha_{i}^{2}\|\vg_{i,k}\|^{2}+2\sum_{k=1}^{N}p_{k}\mathbb{E}\sum_{i=t_{0}}^{t-1}(E-1)\beta_{i}^{2}\|(\vv_{i+1}^{k}-\vv_{i}^{k})\|^{2}\\
& \leq2\sum_{k=1}^{N}p_{k}\mathbb{E}\sum_{i=t_{0}}^{t-1}(E-1)\alpha_{i}^{2}(\|\vg_{i,k}\|^{2}+\|(\vv_{i+1}^{k}-\vv_{i}^{k})\|^{2})\\
& \leq4\sum_{k=1}^{N}p_{k}\mathbb{E}\sum_{i=t_{0}}^{t-1}(E-1)\alpha_{i}^{2}G^{2}\\
& \leq4(E-1)^{2}\alpha_{t_{0}}^{2}G^{2}\leq16(E-1)^{2}\alpha_{t}^{2}G^{2}
\end{align*}
where we have used $\mathbb{E}\|\vv_{t}^{k}-\vv_{t-1}^{k}\|^{2}\leq G^{2}$.
To see this identity for appropriate $\alpha_{t},\beta_{t}$, note
the recursion 
\begin{align*}
\vv_{t+1}^{k}-\vv_{t}^{k} & =\vw_{t}^{k}-\vw_{t-1}^{k}-(\alpha_{t}\vg_{t,k}-\alpha_{t-1}\vg_{t-1,k})\\
\vw_{t+1}^{k}-\vw_{t}^{k} & =-\alpha_{t}\vg_{t,k}+\beta_{t}(\vv_{t+1}^{k}-\vv_{t}^{k})
\end{align*}
so that 
\begin{align*}
\vv_{t+1}^{k}-\vv_{t}^{k} & =-\alpha_{t-1}\vg_{t-1,k}+\beta_{t-1}(\vv_{t}^{k}-\vv_{t-1}^{k})-(\alpha_{t}\vg_{t,k}-\alpha_{t-1}\vg_{t-1,k})\\
& =\beta_{t-1}(\vv_{t}^{k}-\vv_{t-1}^{k})-\alpha_{t}\vg_{t,k}
\end{align*}
Since the identity $\vv_{t+1}^{k}-\vv_{t}^{k}=\beta_{t-1}(\vv_{t}^{k}-\vv_{t-1}^{k})-\alpha_{t}\vg_{t,k}$
implies 
\begin{align*}
\mathbb{E}\|\vv_{t+1}^{k}-\vv_{t}^{k}\|^{2} & \leq2\beta_{t-1}^{2}\mathbb{E}\|\vv_{t}^{k}-\vv_{t-1}^{k}\|^{2}+2\alpha_{t}^{2}G^{2}
\end{align*}
as long as $\alpha_{t},\beta_{t-1}$ satisfy $2\beta_{t-1}^{2}+2\alpha_{t}^{2}\leq1/2$,
we can guarantee that $\mathbb{E}\|\vv_{t}^{k}-\vv_{t-1}^{k}\|^{2}\leq G^{2}$
for all $k$ by induction. This together with Jensen's inequality
also gives $\mathbb{E}\|\ov{v}_{t}-\ov{v}_{t-1}\|^{2}\leq G^{2}$
for all $t$. 
\end{proof}

Now we are ready to prove the one step progress result for Nesterov accelerated FedAvg. The first part of the proof is identical to that of the FedAvg case, while the main recursion takes a different form.
\begin{proof}[Proof of lemma~\ref{lem:nest-scvxoner}]
We again have 
\begin{align*}
\|\ov{v}_{t+1}-\vw^{\ast}\|^{2} & =\|(\ov{w}_{t}-\alpha_{t}\vg_{t})-\vw^{\ast}\|^{2}
\end{align*}
and using exactly the same derivation as the FedAvg case, we can obtain the following bound (same as \eq{\ref{eq:common recursion}} in the proof of Lemma~\ref{lem:scvxoner}):
\begin{align*}
	\mathbb{E}\|\ov{w}_{t+1}-\vw^{\ast}\|^{2} & \leq\mathbb{E}(1-\mu\alpha_{t})\|\ov{w}_{t}-\vw^{\ast}\|^{2}+\alpha_{t}L\sum_{k=1}^{N}p_{k}\|\ov{w}_{t}-\vw_{t}^{k}\|^{2}+\alpha_{t}^{2}\sum_{k=1}^{N}p_{k}^{2}\sigma_{k}^{2}\\
	& +\alpha_{t}^{2}L^{2}\sum_{k=1}^{N}p_{k}\|\ov{w}_{t}-\vw_{t}^{k}\|^{2}+\alpha_{t}^{3}L\mathbb{E}\|\vg_{t}\|^{2}-\alpha_{t}^{2}\|\nabla F(\ov{w}_{t})\|^{2}
	\end{align*}

Different from the FedAvg case, we no longer have $\ov{w}_{t}=\ov{v}_{t}$. Instead,
\begin{align*}
&\|\ov{w}_{t}-\vw^{\ast}\|^{2} \\
 =&\|\ov{v}_{t}+\beta_{t-1}(\ov{v}_{t}-\ov{v}_{t-1})-\vw^{\ast}\|^{2}\\
 =&\|(1+\beta_{t-1})(\ov{v}_{t}-\vw^{\ast})-\beta_{t-1}(\ov{v}_{t-1}-\vw^{\ast})\|^{2}\\
 =&(1+\beta_{t-1})^{2}\|\ov{v}_{t}-\vw^{\ast}\|^{2}-2\beta_{t-1}(1+\beta_{t-1})\langle\ov{v}_{t}-\vw^{\ast},\ov{v}_{t-1}-\vw^{\ast}\rangle+\beta_{t-1}^{2}\|(\ov{v}_{t-1}-\vw^{\ast})\|^{2}\\
 \leq &(1+\beta_{t-1})^{2}\|\ov{v}_{t}-\vw^{\ast}\|^{2}+2\beta_{t-1}(1+\beta_{t-1})\|\ov{v}_{t}-\vw^{\ast}\|\cdot\|\ov{v}_{t-1}-\vw^{\ast}\|+\beta_{t-1}^{2}\|(\ov{v}_{t-1}-\vw^{\ast})\|^{2}
\end{align*}
which gives a recursion involving both $\ov{v}_{t}$ and $\ov{v}_{t-1}$:
\begin{align*}
&\|\ov{v}_{t+1}-\vw^{\ast}\|^{2} \\
& \leq(1-\alpha_{t}\mu)(1+\beta_{t-1})^{2}\|\ov{v}_{t}-\vw^{\ast}\|^{2}+2(1-\alpha_{t}\mu)\beta_{t-1}(1+\beta_{t-1})\|\ov{v}_{t}-\vw^{\ast}\|\cdot\|\ov{v}_{t-1}-\vw^{\ast}\|\\
&+\alpha_{t}^{2}\sum_{k=1}^{N}p_{k}^{2}\sigma_{k}^{2} +\beta_{t-1}^{2}(1-\alpha_{t}\mu)\|(\ov{v}_{t-1}-\vw^{\ast})\|^{2}+\alpha_{t}L\sum_{k=1}^{N}p_{k}\|\ov{w}_{t}-\vw_{t}^{k}\|^{2}\\
&+\alpha_{t}^{2}L^{2}\sum_{k}p_{k}\|\ov{w}_{t}-\vw_{t}^{k}\|^{2}+\alpha_{t}^{3}LG^{2}
\end{align*}
and we will using this recursive relation to obtain the desired bound. 

We can check that our choice of $\alpha_t$ and $\beta_t$ satisfy $\alpha_t$ is non-increasing, $\alpha_t \leq 2\alpha_{t+E}$, and $2\beta_{t-1}^{2}+2\alpha_{t}^{2}\leq1/2$ for all $t\geq 0$, so that we can apply the bound from Lemma~\ref{lem:nest-bdw} on $\mathbb{E}\sum_{k=1}^{N}p_{k}\|\ov{w}_{t}-\vw_{t}^{k}\|^{2}$ to conclude that, with $\nu_{\max}:=N\cdot\max_{k}p_{k}$,
\begin{align*}
&\mathbb{E}\|\ov{v}_{t+1}-\vw^{\ast}\|^{2}\\  \leq & \mathbb{E}(1-\mu\alpha_{t})(1+\beta_{t-1})^{2}\|\ov{v}_{t}-\vw^{\ast}\|^{2}+16E^{2}L\alpha_{t}^{3}G^{2}+16E^{2}L^{2}\alpha_{t}^{4}G^{2}+\alpha_{t}^{3}LG^{2}\\
 +&(1-\alpha_{t}\mu)\beta_{t-1}^{2}\|(\ov{v}_{t-1}-\vw^{\ast})\|^{2} \\
 +&\alpha_{t}^{2}\sum_{k=1}^{N}p_{k}^{2}\sigma_{k}^{2}+2\beta_{t-1}(1+\beta_{t-1})(1-\alpha_{t}\mu)\|\ov{v}_{t}-\vw^{\ast}\|\cdot\|\ov{v}_{t-1}-\vw^{\ast}\|\\
 \leq&\mathbb{E}(1-\mu\alpha_{t})(1+\beta_{t-1})^{2}\|\ov{v}_{t}-\vw^{\ast}\|^{2}+20E^{2}L\alpha_{t}^{3}G^{2}+(1-\alpha_{t}\mu)\beta_{t-1}^{2}\|(\ov{v}_{t-1}-\vw^{\ast})\|^{2}\\
+&\alpha_{t}^{2}\frac{1}{N}\nu_{\max}\sigma^{2}+2\beta_{t-1}(1+\beta_{t-1})(1-\alpha_{t}\mu)\|\ov{v}_{t}-\vw^{\ast}\|\cdot\|\ov{v}_{t-1}-\vw^{\ast}\|
\end{align*}
where we have used $\sigma^{2}=\sum_{k}p_{k}\sigma_{k}^{2}$, and by construction our $\alpha_{t}$
satisfies $L\alpha_{t}\leq\frac{1}{5}$.
\end{proof}

\subsection{Convex Smooth Objectives}
\label{sec:nasgdcvxsmth}
In this section we provide proof of the convergence result for Nesterov accelerated FedAvg with convex and smooth objectives. Unlike with the FedAvg algorithm, where convex and strongly convex results share identical components, the proof for the convergence result in the convex setting for Nesterov FedAvg uses a change of variables, although the general ideas are in the same vein: we have a one step progress bound for $	\mathbb{E}\|\ov{w}_{t+1}-\vw^{\ast}\|^{2}+\eta_{t}(F(\ov{w}_{t})-F(\vw^{\ast}))$, which is then used to form a telescoping sum that gives an upper bound on $\min_{t\leq T}F(\ov{w}_{t})-F(\vw^{\ast})$.

\begin{lemma} [\textbf{One step progress, convex case, Nesterov}]
Let $\overline{\mathbf{w}}_{t}=\sum_{k=1}^{N}p_{k}\mathbf{w}_{t}^{k}$ in Nesterov accelerated FedAvg, and define $\eta_{t}=\frac{\alpha_{t}}{1-\beta_{t}}$. Under assumptions~\ref{ass:lsmooth},\ref{ass:boundedvariance},\ref{ass:subgrad2}, the following bound holds for all $t$:
\begin{align*}
	&\mathbb{E}\|\ov{w}_{t+1}-\vw^{\ast}\|^{2}+\eta_{t}(F(\ov{w}_{t})-F(\vw^{\ast})) \\
 \leq & \mathbb{E}\|\ov{w}_{t}-\vw^{\ast}\|^{2}+32LE^{2}\alpha_{t}^{2}\eta_{t}G^{2}+\eta_{t}^{2}\nu_{\max}\frac{1}{N}\sigma^{2}+2\eta_{t}\frac{\beta_{t}^{2}}{1-\beta_{t}}G^{2}.
	\end{align*}
	\label{lem:nest-cvxoner}
\end{lemma}

\begin{thm}
	Set learning rates $\alpha_{t}=\beta_{t}=\mathcal{O}(\sqrt{\frac{N}{T}})$. Then under Assumptions~\ref{ass:lsmooth},\ref{ass:boundedvariance},\ref{ass:subgrad2} Nesterov accelerated FedAvg with
	full device participation has rate
	\begin{align*}
	\min_{t\leq T}F(\overline{\mathbf{w}}_{t})-F^{\ast} & =\mathcal{O}\left(\frac{\nu_{\max}\sigma^{2}}{\sqrt{NT}}+\frac{NE^{2}LG^{2}}{T}\right),
	\end{align*}
	and with partial device participation with $K$ sampled devices at
	each communication round and learning rates $\alpha_{t}=\beta_{t}=\mathcal{O}(\sqrt{\frac{K}{T}})$,
	\begin{align*}
	\min_{t\leq T}F(\overline{\mathbf{w}}_{t})-F^{\ast} & =\mathcal{O}\left(\frac{\nu_{\max}\sigma^{2}}{\sqrt{KT}}+\frac{EG^{2}}{\sqrt{KT}}+\frac{KE^{2}LG^{2}}{T}\right).
	\end{align*}
\end{thm}

\begin{proof}
	
    Applying the bound from Lemma~\ref{lem:nest-cvxoner}, with $\eta_{t}=\frac{\alpha_{t}}{1-\beta_{t}}$ we have
	\begin{align*}
	& \mathbb{E}\|\ov{w}_{t+1}-\vw^{\ast}\|^{2}+\eta_{t}(F(\ov{w}_{t})-F(\vw^{\ast})) \\
  \leq & \mathbb{E}\|\ov{w}_{t}-\vw^{\ast}\|^{2}+32LE^{2}\alpha_{t}^{2}\eta_{t}G^{2}+\eta_{t}^{2}\nu_{\max}\frac{1}{N}\sigma^{2}+2\eta_{t}\frac{\beta_{t}^{2}}{1-\beta_{t}}G^{2}
	\end{align*}
	Summing the inequalities from $t=0$ to $t=T$, we obtain 
	\begin{align*}
	& \sum_{t=0}^{T}\eta_{t}(F(\ov{w}_{t})-F(\vw^{\ast})) \\
 \leq & \|\vw_{0}-\vw^{\ast}\|^{2}+\sum_{t=0}^{T}\eta_{t}^{2}\cdot\frac{1}{N}\nu_{\max}\sigma^{2}+\sum_{t=0}^{T}\eta_{t}\alpha_{t}^{2}\cdot32LE^{2}G^{2}+\sum_{t=0}^{T}2\eta_{t}\frac{\beta_{t}^{2}}{1-\beta_{t}}G^{2}
	\end{align*}
	so that
	\begin{align*}
	&\min_{t\leq T}F(\ov{w}_{t})-F(\vw^{\ast}) \\
  \leq & \frac{1}{\sum_{t=0}^{T}\eta_{t}}\left(\|\vw_{0}-\vw^{\ast}\|^{2}+\sum_{t=0}^{T}\eta_{t}^{2}\cdot\frac{1}{N}\nu_{\max}\sigma^{2}+\sum_{t=0}^{T}\eta_{t}\alpha_{t}^{2}\cdot32LE^{2}G^{2}+\sum_{t=0}^{T}2\eta_{t}\frac{\beta_{t}^{2}}{1-\beta_{t}}G^{2}\right)
	\end{align*}
	
	By setting the constant learning rates $\alpha_{t}\equiv\sqrt{\frac{N}{T}}$
	and $\beta_{t}\equiv c\sqrt{\frac{N}{T}}$ so that $\eta_{t}=\frac{\alpha_{t}}{1-\beta_{t}}=\frac{\sqrt{\frac{N}{T}}}{1-c\sqrt{\frac{N}{T}}}\leq2\sqrt{\frac{N}{T}}$,
	we have 
	\begin{align*}
	&\min_{t\leq T}F(\ov{w}_{t})-F(\vw^{\ast})\leq\frac{1}{2\sqrt{NT}}\cdot\|\vw_{0}-\vw^{\ast}\|^{2}\\
 &+\frac{2}{\sqrt{NT}}T\cdot\frac{N}{T}\cdot\frac{1}{N}\nu_{\max}\sigma^{2}+\frac{1}{\sqrt{NT}}T(\sqrt{\frac{N}{T}})^{3}32LE^{2}G^{2}+\frac{2}{\sqrt{NT}}T(\sqrt{\frac{N}{T}})^{3}G^{2}\\
	& =(\frac{1}{2}\|\vw_{0}-\vw^{\ast}\|^{2}+2\nu_{\max}\sigma^{2})\frac{1}{\sqrt{NT}}+\frac{N}{T}(32LE^{2}G^{2}+2G^{2})\\
	& =O(\frac{\nu_{\max}\sigma^{2}}{\sqrt{NT}}+\frac{NE^{2}LG^{2}}{T})
	\end{align*}
	
	Similarly, for partial participation, using the same argument to get the $E$-step bound in the proof of Theorem \ref{thm:SGD_cvx},
	we have 
	\begin{align*}
	\min_{t\leq T}F(\ov{w}_{t})-F(\vw^{\ast}) & =\mathcal{O}(\frac{\nu_{\max}\sigma^{2}}{\sqrt{KT}}+\frac{EG^{2}}{\sqrt{KT}}+\frac{KE^{2}LG^{2}}{T})
	\end{align*}
\end{proof}

\subsubsection{Deferred Proofs of Key Lemmas}
\begin{proof}[Proof of lemma~\ref{lem:nest-cvxoner}]
    Define $\ov{p}_{t}:=\frac{\beta_{t}}{1-\beta_{t}}\left[\ov{w}_{t}-\ov{w}_{t-1}+\alpha_{t}\vg_{t-1}\right]=\frac{\beta_{t}^{2}}{1-\beta_{t}}(\ov{v}_{t}-\ov{v}_{t-1})$
	for $t\geq1$ and 0 for $t=0$. We can check that 
	\begin{align*}
	\ov{w}_{t+1}+\ov{p}_{t+1} & =\ov{w}_{t}+\ov{p}_{t}-\frac{\alpha_{t}}{1-\beta_{t}}\vg_{t}
	\end{align*}
	Now we define $\ov{z}_{t}:=\ov{w}_{t}+\ov{p}_{t}$
	and $\eta_{t}=\frac{\alpha_{t}}{1-\beta_{t}}$ for all $t$, so that
	we have the recursive relation 
	\begin{align*}
	\ov{z}_{t+1} & =\ov{z}_{t}-\eta_{t}\vg_{t}
	\end{align*}
	Now 
	\begin{align*}
	\|\ov{z}_{t+1}-\vw^{\ast}\|^{2} & =\|(\ov{z}_{t}-\eta_{t}\vg_{t})-\vw^{\ast}\|^{2}\\
	& =\|(\ov{z}_{t}-\eta_{t}\ov{g}_{t}-\vw^{\ast})-\eta_{t}(\vg_{t}-\ov{g}_{t})\|^{2}\\
	& =A_{1}+A_{2}+A_{3}
	\end{align*}
	where 
	\begin{align*}
	A_{1} & =\|\ov{z}_{t}-\vw^{\ast}-\eta_{t}\ov{g}_{t}\|^{2}\\
	A_{2} & =2\eta_{t}\langle\ov{z}_{t}-\vw^{\ast}-\eta_{t}\ov{g}_{t},\ov{g}_{t}-\vg_{t}\rangle\\
	A_{3} & =\eta_{t}^{2}\|\vg_{t}-\ov{g}_{t}\|^{2}
	\end{align*}
	where again $\mathbb{E}A_{2}=0$ and $\mathbb{E}A_{3}\leq\eta_{t}^{2}\sum_{k}p_{k}^{2}\sigma_{k}^{2}$.
	For $A_{1}$ we have 
	\begin{align*}
	\|\ov{z}_{t}-\vw^{\ast}-\eta_{t}\ov{g}_{t}\|^{2} & =\|\ov{z}_{t}-\vw^{\ast}\|^{2}+2\langle\ov{z}_{t}-\vw^{\ast},-\eta_{t}\ov{g}_{t}\rangle+\|\eta_{t}\ov{g}_{t}\|^{2}
	\end{align*}
	Using the convexity and $L$-smoothness of $F_{k}$, 
	\begin{align*}
	& -2\eta_{t}\langle\ov{z}_{t}-\vw^{\ast},\ov{g}_{t}\rangle\\
	& =-2\eta_{t}\sum_{k=1}^{N}p_{k}\langle\ov{z}_{t}-\vw^{\ast},\nabla F_{k}(\vw_{t}^{k})\rangle\\
	& =-2\eta_{t}\sum_{k=1}^{N}p_{k}\langle\ov{z}_{t}-\vw_{t}^{k},\nabla F_{k}(\vw_{t}^{k})\rangle-2\eta_{t}\sum_{k=1}^{N}p_{k}\langle \vw_{t}^{k}-\vw^{\ast},\nabla F_{k}(\vw_{t}^{k})\rangle\\
	& =-2\eta_{t}\sum_{k=1}^{N}p_{k}\langle\ov{z}_{t}-\ov{w}_{t},\nabla F_{k}(\vw_{t}^{k})\rangle-2\eta_{t}\sum_{k=1}^{N}p_{k}\langle\ov{w}_{t}-\vw_{t}^{k},\nabla F_{k}(\vw_{t}^{k})\rangle\\
 &-2\eta_{t}\sum_{k=1}^{N}p_{k}\langle \vw_{t}^{k}-\vw^{\ast},\nabla F_{k}(\vw_{t}^{k})\rangle\\
	& \leq-2\eta_{t}\sum_{k=1}^{N}p_{k}\langle\ov{z}_{t}-\ov{w}_{t},\nabla F_{k}(\vw_{t}^{k})\rangle-2\eta_{t}\sum_{k=1}^{N}p_{k}\langle\ov{w}_{t}-\vw_{t}^{k},\nabla F_{k}(\vw_{t}^{k})\rangle\\
 &+2\eta_{t}\sum_{k=1}^{N}p_{k}(F_{k}(\vw^{\ast})-F_{k}(\vw_{t}^{k}))\\
	& \leq2\eta_{t}\sum_{k=1}^{N}p_{k}\left[F_{k}(\vw_{t}^{k})-F_{k}(\ov{w}_{t})+\frac{L}{2}\|\ov{w}_{t}-\vw_{t}^{k}\|^{2}+F_{k}(\vw^{\ast})-F_{k}(\vw_{t}^{k})\right]\\
	& -2\eta_{t}\sum_{k=1}^{N}p_{k}\langle\ov{z}_{t}-\ov{w}_{t},\nabla F_{k}(\vw_{t}^{k})\rangle\\
	& =\eta_{t}L\sum_{k=1}^{N}p_{k}\|\ov{w}_{t}-\vw_{t}^{k}\|^{2}+2\eta_{t}\sum_{k=1}^{N}p_{k}\left[F_{k}(\vw^{\ast})-F_{k}(\ov{w}_{t})\right]-2\eta_{t}\sum_{k=1}^{N}p_{k}\langle\ov{z}_{t}-\ov{w}_{t},\nabla F_{k}(\vw_{t}^{k})\rangle
	\end{align*}
	which results in 
	\begin{align*}
	\mathbb{E}\|\ov{w}_{t+1}-\vw^{\ast}\|^{2} & \leq\mathbb{E}\|\ov{w}_{t}-\vw^{\ast}\|^{2}+\eta_{t}L\sum_{k=1}^{N}p_{k}\|\ov{w}_{t}-\vw_{t}^{k}\|^{2}+2\eta_{t}\sum_{k=1}^{N}p_{k}\left[F_{k}(\vw^{\ast})-F_{k}(\ov{w}_{t})\right]\\
	& +\eta_{t}^{2}\|\ov{g}_{t}\|^{2}+\eta_{t}^{2}\sum_{k=1}^{N}p_{k}^{2}\sigma_{k}^{2}-2\eta_{t}\sum_{k=1}^{N}p_{k}\langle\ov{z}_{t}-\ov{w}_{t},\nabla F_{k}(\vw_{t}^{k})\rangle
	\end{align*}
	As before, $\|\ov{g}_{t}\|^{2}\leq2L^{2}\sum_{k}p_{k}\|\vw_{t}^{k}-\ov{w}_{t}\|^{2}+4L(F(\ov{w}_{t})-F(\vw^{\ast}))$,
	so that 
	\begin{align*}
\eta_{t}^{2}\|\ov{g}_{t}\|^{2}+\eta_{t}\sum_{k=1}^{N}p_{k}\left[F_{k}(\vw^{\ast})-F_{k}(\ov{w}_{t})\right] & \leq2L^{2}\eta_{t}^{2}\sum_{k}p_{k}\|\vw_{t}^{k}-\ov{w}_{t}\|^{2}\\
&+\eta_{t}(1-4\eta_{t}L)(F(\vw^{\ast})-F(\ov{w}_{t}))\\
	& \leq2L^{2}\eta_{t}^{2}\sum_{k}p_{k}\|\vw_{t}^{k}-\ov{w}_{t}\|^{2}
	\end{align*}
	for $\eta_{t}\le1/4L$. Using $\sum_{k=1}^{N}p_{k}\|\ov{w}_{t}-\vw_{t}^{k}\|^{2}\leq16E^{2}\alpha_{t}^{2}G^{2}$
	and $\sum_{k=1}^{N}p_{k}^{2}\sigma_{k}^{2}\leq\nu_{\max}\frac{1}{N}\sigma^{2}$,
	it follows that 
	\begin{align*}
	\mathbb{E}\|\ov{w}_{t+1}-\vw^{\ast}\|^{2}+\eta_{t}(F(\ov{w}_{t})-F(\vw^{\ast})) & \leq\mathbb{E}\|\ov{w}_{t}-\vw^{\ast}\|^{2}+(\eta_{t}L+2L^{2}\eta_{t}^{2})\sum_{k=1}^{N}p_{k}\|\ov{w}_{t}-\vw_{t}^{k}\|^{2}\\
 &+\eta_{t}^{2}\sum_{k=1}^{N}p_{k}^{2}\sigma_{k}^{2} -2\eta_{t}\sum_{k=1}^{N}p_{k}\langle\ov{z}_{t}-\ov{w}_{t},\nabla F_{k}(\vw_{t}^{k})\rangle\\
	& \leq\mathbb{E}\|\ov{w}_{t}-\vw^{\ast}\|^{2}+32LE^{2}\alpha_{t}^{2}\eta_{t}G^{2}+\eta_{t}^{2}\nu_{\max}\frac{1}{N}\sigma^{2}\\
	& -2\eta_{t}\sum_{k=1}^{N}p_{k}\langle\ov{z}_{t}-\ov{w}_{t},\nabla F_{k}(\vw_{t}^{k})\rangle
	\end{align*}
	if $\eta_{t}\leq\frac{1}{2L}$. It remains to bound $\mathbb{E}\sum_{k=1}^{N}p_{k}\langle\ov{z}_{t}-\ov{w}_{t},\nabla F_{k}(\vw_{t}^{k})\rangle$.
	Recall that $\ov{z}_{t}-\ov{w}_{t}=\frac{\beta_{t}}{1-\beta_{t}}\left[\ov{w}_{t}-\ov{w}_{t-1}+\alpha_{t}\vg_{t-1}\right]=\frac{\beta_{t}^{2}}{1-\beta_{t}}(\ov{v}_{t}-\ov{v}_{t-1})$
	and $\mathbb{E}\|\ov{v}_{t}-\ov{v}_{t-1}\|^{2}\leq G^{2}$,
	$\mathbb{E}\|\nabla F_{k}(\vw_{t}^{k})\|^{2}\leq G^{2}$. 
	
	Cauchy-Schwarz gives
	\begin{align*}
	\mathbb{E}\sum_{k=1}^{N}p_{k}\langle\ov{z}_{t}-\ov{w}_{t},\nabla F_{k}(\vw_{t}^{k})\rangle & \leq\sum_{k=1}^{N}p_{k}\sqrt{\mathbb{E}\|\ov{z}_{t}-\ov{w}_{t}\|^{2}}\cdot\sqrt{\mathbb{E}\|\nabla F_{k}(\vw_{t}^{k})\|^{2}}\\
	& \leq\frac{\beta_{t}^{2}}{1-\beta_{t}}G^{2}
	\end{align*}
	Thus 
	\begin{align*}
	\mathbb{E}\|\ov{w}_{t+1}-\vw^{\ast}\|^{2}+\eta_{t}(F(\ov{w}_{t})-F(\vw^{\ast})) & \leq\mathbb{E}\|\ov{w}_{t}-\vw^{\ast}\|^{2}+32LE^{2}\alpha_{t}^{2}\eta_{t}G^{2}\\
 &+\eta_{t}^{2}\nu_{\max}\frac{1}{N}\sigma^{2}+2\eta_{t}\frac{\beta_{t}^{2}}{1-\beta_{t}}G^{2}
	\end{align*}
\end{proof}

\section{Proof of Geometric Convergence Results for Overparameterized Problems}
\label{sec:interpolation}

\subsection{Geometric Convergence of FedAvg for General Strongly Convex and Smooth Objectives}

\begin{thm}
	For the overparameterized setting with general strongly convex and
	smooth objectives, FedAvg with local SGD updates and communication
	every $E$ iterations with constant step size $\overline{\alpha}=\frac{1}{2E}\frac{N}{l\nu_{\max}+L(N-\nu_{\min})}$
	gives the exponential convergence guarantee 
	\begin{align*}
	\mathbb{E}F(\ov{w}_{t}) & \leq\frac{L}{2}(1-\mu\overline{\alpha})^{t}\|\vw_{0}-\vw^{\ast}\|^{2}=O(\exp(-\frac{\mu}{2E}\frac{N}{l\nu_{\max}+L(N-\nu_{\min})}t)\cdot\|\vw_{0}-\vw^{\ast}\|^{2})
	\end{align*}
\end{thm}
\begin{proof}
	To illustrate the main ideas of the proof, we first present the proof
	for $E=2$. Let $t-1$ be a communication round, so that $\vw_{t-1}^{k}=\ov{w}_{t-1}$.
	We show that 
	
	\begin{align*}
	\|\ov{w}_{t+1}-\vw^{\ast}\|^{2} & \leq(1-\alpha_{t}\mu)(1-\alpha_{t-1}\mu)\|\ov{w}_{t-1}-\vw^{\ast}\|^{2}
	\end{align*}
	for appropriately chosen constant step sizes $\alpha_{t},\alpha_{t-1}$.
	We have 
	
	\begin{align*}
	\|\ov{w}_{t+1}-\vw^{\ast}\|^{2} & =\|(\ov{w}_{t}-\alpha_{t}\vg_{t})-\vw^{\ast}\|^{2}\\
	& =\|\ov{w}_{t}-\vw^{\ast}\|^{2}-2\alpha_{t}\langle\ov{w}_{t}-\vw^{\ast},\vg_{t}\rangle+\alpha_{t}^{2}\|\vg_{t}\|^{2}
	\end{align*}
	and the cross term can be bounded as usual using $\mu$-convexity
	and $L$-smoothness of $F_{k}$:
	\begin{align*}
	&-2\alpha_{t}\mathbb{E}_{t}\langle\ov{w}_{t}-\vw^{\ast},\vg_{t}\rangle\\
	& =-2\alpha_{t}\sum_{k=1}^{N}p_{k}\langle\ov{w}_{t}-\vw^{\ast},\nabla F_{k}(\vw_{t}^{k})\rangle\\
	& =-2\alpha_{t}\sum_{k=1}^{N}p_{k}\langle\ov{w}_{t}-\vw_{t}^{k},\nabla F_{k}(\vw_{t}^{k})\rangle-2\alpha_{t}\sum_{k=1}^{N}p_{k}\langle \vw_{t}^{k}-\vw^{\ast},\nabla F_{k}(\vw_{t}^{k})\rangle\\
	& \leq-2\alpha_{t}\sum_{k=1}^{N}p_{k}\langle\ov{w}_{t}-\vw_{t}^{k},\nabla F_{k}(\vw_{t}^{k})\rangle+2\alpha_{t}\sum_{k=1}^{N}p_{k}(F_{k}(\vw^{\ast})-F_{k}(\vw_{t}^{k}))-\alpha_{t}\mu\sum_{k=1}^{N}p_{k}\|\vw_{t}^{k}-\vw^{\ast}\|^{2}\\
	& \leq2\alpha_{t}\sum_{k=1}^{N}p_{k}\left[F_{k}(\vw_{t}^{k})-F_{k}(\ov{w}_{t})+\frac{L}{2}\|\ov{w}_{t}-\vw_{t}^{k}\|^{2}+F_{k}(\vw^{\ast})-F_{k}(\vw_{t}^{k})\right]-\alpha_{t}\mu\|\sum_{k=1}^{N}p_{k}(\vw_{t}^{k}-\vw^{\ast})\|^{2}\\
	& =\alpha_{t}L\sum_{k=1}^{N}p_{k}\|\ov{w}_{t}-\vw_{t}^{k}\|^{2}+2\alpha_{t}\sum_{k=1}^{N}p_{k}\left[F_{k}(\vw^{\ast})-F_{k}(\ov{w}_{t})\right]-\alpha_{t}\mu\|\ov{w}_{t}-\vw^{\ast}\|^{2}\\
	& =\alpha_{t}L\sum_{k=1}^{N}p_{k}\|\ov{w}_{t}-\vw_{t}^{k}\|^{2}-2\alpha_{t}\sum_{k=1}^{N}p_{k}F_{k}(\ov{w}_{t})-\alpha_{t}\mu\|\ov{w}_{t}-\vw^{\ast}\|^{2}
	\end{align*}
	and so 
	\begin{align*}
	\mathbb{E}\|\ov{w}_{t+1}-\vw^{\ast}\|^{2} & \leq\mathbb{E}(1-\alpha_{t}\mu)\|\ov{w}_{t}-\vw^{\ast}\|^{2}-2\alpha_{t}F(\ov{w}_{t})+\alpha_{t}^{2}\|\vg_{t}\|^{2}+\alpha_{t}L\sum_{k=1}^{N}p_{k}\|\ov{w}_{t}-\vw_{t}^{k}\|^{2}
	\end{align*}
	
	Applying this recursive relation to $\|\ov{w}_{t}-\vw^{\ast}\|^{2}$
	and using $\|\ov{w}_{t-1}-\vw_{t-1}^{k}\|^{2}\equiv0$, we further
	obtain 
	\begin{align*}
	\mathbb{E}\|\ov{w}_{t+1}-\vw^{\ast}\|^{2} & \leq\mathbb{E}(1-\alpha_{t}\mu)\left((1-\alpha_{t-1}\mu)\|\ov{w}_{t-1}-\vw^{\ast}\|^{2}-2\alpha_{t-1}F(\ov{w}_{t-1})+\alpha_{t-1}^{2}\|\vg_{t-1}\|^{2}\right)\\
	& -2\alpha_{t}F(\ov{w}_{t})+\alpha_{t}^{2}\|\vg_{t}\|^{2}+\alpha_{t}L\sum_{k=1}^{N}p_{k}\|\ov{w}_{t}-\vw_{t}^{k}\|^{2}
	\end{align*}
	Now instead of bounding $\sum_{k=1}^{N}p_{k}\|\ov{w}_{t}-\vw_{t}^{k}\|^{2}$
	using the arguments in the general convex case, we follow~\cite{ma2017power} and use the fact that
	in the overparameterized setting, $\vw^{\ast}$ is a minimizer of each
	$\ell(\vw,x_{k}^{j})$ and that each $\ell$ is $l$-smooth to obtain
	$\|\nabla F_{k}(\ov{w}_{t-1},\xi_{t-1}^{k})\|^{2}\leq2l(F_{k}(\ov{w}_{t-1},\xi_{t-1}^{k})-F_{k}(\vw^{\ast},\xi_{t-1}^{k}))$,
	where recall $F_{k}(\vw,\xi_{t-1}^{k})=\ell(\vw,\xi_{t-1}^{k})$, so that
	\begin{align*}
	\sum_{k=1}^{N}p_{k}\|\ov{w}_{t}-\vw_{t}^{k}\|^{2} & =\sum_{k=1}^{N}p_{k}\|\ov{w}_{t-1}-\alpha_{t-1}\vg_{t-1}-\vw_{t-1}^{k}+\alpha_{t-1}\vg_{t-1,k}\|^{2}\\
	& =\sum_{k=1}^{N}p_{k}\alpha_{t-1}^{2}\|\vg_{t-1}-\vg_{t-1,k}\|^{2}\\
	& =\alpha_{t-1}^{2}\sum_{k=1}^{N}p_{k}(\|\vg_{t-1,k}\|^{2}-\|\vg_{t-1}\|^{2})\\
	& =\alpha_{t-1}^{2}\sum_{k=1}^{N}p_{k}\|\nabla F_{k}(\ov{w}_{t-1},\xi_{t-1}^{k})\|^{2}-\alpha_{t-1}^{2}\|\vg_{t-1}\|^{2}\\
	& \le\alpha_{t-1}^{2}\sum_{k=1}^{N}p_{k}2l(F_{k}(\ov{w}_{t-1},\xi_{t-1}^{k})-F_{k}(\vw^{\ast},\xi_{t-1}^{k}))-\alpha_{t-1}^{2}\|\vg_{t-1}\|^{2}
	\end{align*}
	again using $\ov{w}_{t-1}=\vw_{t-1}^{k}$. Taking expectation
	with respect to $\xi_{t-1}^{k}$'s and using the fact that $F(\vw^{\ast})=0$,
	we have 
	\begin{align*}
	\mathbb{E}_{t-1}\sum_{k=1}^{N}p_{k}\|\ov{w}_{t}-\vw_{t}^{k}\|^{2} & \leq2l\alpha_{t-1}^{2}\sum_{k=1}^{N}p_{k}F_{k}(\ov{w}_{t-1})-\alpha_{t-1}^{2}\|\vg_{t-1}\|^{2}\\
	& =2l\alpha_{t-1}^{2}F(\ov{w}_{t-1})-\alpha_{t-1}^{2}\|\vg_{t-1}\|^{2}
	\end{align*}
	
	Note also that 
	\begin{align*}
	\|\vg_{t-1}\|^{2} & =\|\sum_{k=1}^{N}p_{k}\nabla F_{k}(\ov{w}_{t-1},\xi_{t-1}^{k})\|^{2}
	\end{align*}
	while
	\begin{align*}
&\|\vg_{t}\|^{2}=\|\sum_{k=1}^{N}p_{k}\nabla F_{k}(\vw_{t}^{k},\xi_{t}^{k})\|^{2} \\
 \leq& 2\|\sum_{k=1}^{N}p_{k}\nabla F_{k}(\ov{w}_{t},\xi_{t}^{k})\|^{2}+2\|\sum_{k=1}^{N}p_{k}(\nabla F_{k}(\ov{w}_{t},\xi_{t}^{k})-\nabla F_{k}(\vw_{t}^{k},\xi_{t}^{k}))\|^{2}\\
	 \leq& 2\|\sum_{k=1}^{N}p_{k}\nabla F_{k}(\ov{w}_{t},\xi_{t}^{k})\|^{2}+2\sum_{k=1}^{N}p_{k}l^{2}\|\ov{w}_{t}-\vw_{t}^{k}\|^{2}
	\end{align*}
	Substituting these into the bound for $\|\ov{w}_{t+1}-\vw^{\ast}\|^{2}$,
	we have 
	
	\begin{align*}
	&\mathbb{E}\|\ov{w}_{t+1}-\vw^{\ast}\|^{2} \\
 &\leq\mathbb{E}(1-\alpha_{t}\mu)((1-\alpha_{t-1}\mu)\|\ov{w}_{t-1}-\vw^{\ast}\|^{2}-2\alpha_{t-1}F(\ov{w}_{t-1})+\alpha_{t-1}^{2}\|\vg_{t-1}\|^{2})\\
	& -2\alpha_{t}F(\ov{w}_{t})+2\alpha_{t}^{2}\|\sum_{k=1}^{N}p_{k}\nabla F_{k}(\ov{w}_{t},\xi_{t}^{k})\|^{2}+\left(2l^{2}\alpha_{t-1}^{2}\alpha_{t}^{2}+\alpha_{t}\alpha_{t-1}^{2}L\right)\left(2lF(\ov{w}_{t-1})-\|\vg_{t-1}\|^{2}\right)\\
	& =\mathbb{E}(1-\alpha_{t}\mu)(1-\alpha_{t-1}\mu)\|\ov{w}_{t-1}-\vw^{\ast}\|^{2}\\
	& -2\alpha_{t}(F(\ov{w}_{t})-\alpha_{t}\|\sum_{k=1}^{N}p_{k}\nabla F_{k}(\ov{w}_{t},\xi_{t}^{k})\|^{2})\\
	& -2\alpha_{t-1}(1-\alpha_{t}\mu)\left((1-\frac{l\alpha_{t-1}(2l^{2}\alpha_{t}^{2}+\alpha_{t}L)}{1-\alpha_{t}\mu})F(\ov{w}_{t-1})-\frac{\alpha_{t-1}}{2}\|\sum_{k=1}^{N}p_{k}\nabla F_{k}(\ov{w}_{t-1},\xi_{t-1}^{k})\|^{2}\right)
	\end{align*}
	from which we can conclude that 
	\begin{align*}
	\mathbb{E}\|\ov{w}_{t+1}-\vw^{\ast}\|^{2} & \leq(1-\alpha_{t}\mu)(1-\alpha_{t-1}\mu)\mathbb{E}\|\ov{w}_{t-1}-\vw^{\ast}\|^{2}
	\end{align*}
	if we can choose $\alpha_{t},\alpha_{t-1}$ to guarantee
	\begin{align*}
	\mathbb{E}(F(\ov{w}_{t})-\alpha_{t}\|\sum_{k=1}^{N}p_{k}\nabla F_{k}(\ov{w}_{t},\xi_{t}^{k})\|^{2}) & \geq0\\
	\mathbb{E}\left((1-\frac{l\alpha_{t-1}(2l^{2}\alpha_{t}^{2}+\alpha_{t}L)}{1-\alpha_{t}\mu})F(\ov{w}_{t-1})-\frac{\alpha_{t-1}}{2}\|\sum_{k=1}^{N}p_{k}\nabla F_{k}(\ov{w}_{t-1},\xi_{t-1}^{k})\|^{2}\right) & \geq0
	\end{align*}
	
	Note that 
	\begin{align*}
 &\mathbb{E}_{t}\|\sum_{k=1}^{N}p_{k}\nabla F_{k}(\ov{w}_{t},\xi_{t}^{k})\|^{2} \\
=&\mathbb{E}_{t}\langle\sum_{k=1}^{N}p_{k}\nabla F_{k}(\ov{w}_{t},\xi_{t}^{k}),\sum_{k=1}^{N}p_{k}\nabla F_{k}(\ov{w}_{t},\xi_{t}^{k})\rangle\\
	& =\sum_{k=1}^{N}p_{k}^{2}\mathbb{E}_{t}\|\nabla F_{k}(\ov{w}_{t},\xi_{t}^{k})\|^{2}+\sum_{k=1}^{N}\sum_{j\neq k}p_{j}p_{k}\mathbb{E}_{t}\langle\nabla F_{k}(\ov{w}_{t},\xi_{t}^{k}),\nabla F_{j}(\ov{w}_{t},\xi_{t}^{j})\rangle\\
	& =\sum_{k=1}^{N}p_{k}^{2}\mathbb{E}_{t}\|\nabla F_{k}(\ov{w}_{t},\xi_{t}^{k})\|^{2}+\sum_{k=1}^{N}\sum_{j\neq k}p_{j}p_{k}\langle\nabla F_{k}(\ov{w}_{t}),\nabla F_{j}(\ov{w}_{t})\rangle\\
	& =\sum_{k=1}^{N}p_{k}^{2}\mathbb{E}_{t}\|\nabla F_{k}(\ov{w}_{t},\xi_{t}^{k})\|^{2}+\sum_{k=1}^{N}\sum_{j=1}^{N}p_{j}p_{k}\langle\nabla F_{k}(\ov{w}_{t}),\nabla F_{j}(\ov{w}_{t})\rangle-\sum_{k=1}^{N}p_{k}^{2}\|\nabla F_{k}(\ov{w}_{t})\|^{2}\\
	& \leq\sum_{k=1}^{N}p_{k}^{2}\mathbb{E}_{t}\|\nabla F_{k}(\ov{w}_{t},\xi_{t}^{k})\|^{2}+\|\sum_{k}p_{k}\nabla F_{k}(\ov{w}_{t})\|^{2}-\frac{1}{N}\nu_{\min}\|\sum_{k}p_{k}\nabla F_{k}(\ov{w}_{t})\|^{2}\\
	& =\sum_{k=1}^{N}p_{k}^{2}\mathbb{E}_{t}\|\nabla F_{k}(\ov{w}_{t},\xi_{t}^{k})\|^{2}+(1-\frac{1}{N}\nu_{\min})\|\nabla F(\ov{w}_{t})\|^{2}
	\end{align*}
	and so following~\cite{ma2017power} if we let $\alpha_{t}=\min\{\frac{qN}{2l\nu_{\max}},\frac{1-q}{2L(1-\frac{1}{N}\nu_{\min})}\}$
	for a $q\in[0,1]$ to be optimized later, we have 
	\begin{align*}
	& \mathbb{E}_{t}(F(\ov{w}_{t})-\alpha_{t}\|\sum_{k=1}^{N}p_{k}\nabla F_{k}(\ov{w}_{t},\xi_{t}^{k})\|^{2})\\
	& \geq\mathbb{E}_{t}\sum_{k=1}^{N}p_{k}F_{k}(\ov{w}_{t})-\alpha_{t}\left[\sum_{k=1}^{N}p_{k}^{2}\mathbb{E}_{t}\|\nabla F_{k}(\ov{w}_{t},\xi_{t}^{k})\|^{2}+(1-\frac{1}{N}\nu_{\min})\|\nabla F(\ov{w}_{t})\|^{2}\right]\\
	& \geq\mathbb{E}_{t}\sum_{k=1}^{N}p_{k}(qF_{k}(\ov{w}_{t},\xi_{t}^{k})-\alpha_{t}\frac{1}{N}\nu_{\max}\|\nabla F_{k}(\ov{w}_{t},\xi_{t}^{k})\|^{2})\\
 &+((1-q)F(\ov{w}_{t})-\alpha_{t}(1-\frac{1}{N}\nu_{\min})\|\nabla F(\ov{w}_{t})\|^{2})\\
	& \geq q\mathbb{E}_{t}\sum_{k=1}^{N}p_{k}(F_{k}(\ov{w}_{t},\xi_{t}^{k})-\frac{1}{2l}\|\nabla F_{k}(\ov{w}_{t},\xi_{t}^{k})\|^{2})+(1-q)(F(\ov{w}_{t})-\frac{1}{2L}\|\nabla F(\ov{w}_{t})\|^{2})\\
	& \geq0
	\end{align*}
	again using $\vw^{\ast}$ optimizes $F_{k}(\vw,\xi_{t}^{k})$ with $F_{k}(\vw^{\ast},\xi_{t}^{k})=0$. 
	
	Maximizing $\alpha_{t}=\min\{\frac{qN}{2l\nu_{\max}},\frac{1-q}{2L(1-\frac{1}{N}\nu_{\min})}\}$
	over $q\in[0,1]$, we see that $q=\frac{l\nu_{\max}}{l\nu_{\max}+L(N-\nu_{\min})}$
	results in the fastest convergence, and this translates to $\alpha_{t}=\frac{1}{2}\frac{N}{l\nu_{\max}+L(N-\nu_{\min})}$.
	Next we claim that $\alpha_{t-1}=c\frac{1}{2}\frac{N}{l\nu_{\max}+L(N-\nu_{\min})}$
	also guarantees
	\begin{align*}
	\mathbb{E}(1-\frac{l\alpha_{t-1}(2l^{2}\alpha_{t}^{2}+\alpha_{t}L)}{1-\alpha_{t}\mu})F(\ov{w}_{t-1})-\frac{\alpha_{t-1}}{2}\|\sum_{k=1}^{N}p_{k}\nabla F_{k}(\ov{w}_{t-1},\xi_{t-1}^{k})\|^{2} & \geq0
	\end{align*}
	
	Note that by scaling $\alpha_{t-1}$ by a constant $c\leq1$ if necessary,
	we can guarantee that $\frac{l\alpha_{t-1}(2l^{2}\alpha_{t}^{2}+\alpha_{t}L)}{1-\alpha_{t}\mu}\leq\frac{1}{2}$,
	and so the condition is equivalent to 
	\begin{align*}
	F(\ov{w}_{t-1})-\alpha_{t-1}\|\sum_{k=1}^{N}p_{k}\nabla F_{k}(\ov{w}_{t-1},\xi_{t-1}^{k})\|^{2} & \geq0
	\end{align*}
	which was shown to hold with $\alpha_{t-1}\leq\frac{1}{2}\frac{N}{l\nu_{\max}+L(N-\nu_{\min})}$. 
	
	For the proof of general $E\ge2$, we use the following two identities:
	\begin{align*}
	\|\vg_{t}\|^{2} & \leq2\|\sum_{k=1}^{N}p_{k}\nabla F_{k}(\ov{w}_{t},\xi_{t}^{k})\|^{2}+2\sum_{k=1}^{N}p_{k}l^{2}\|\ov{w}_{t}-\vw_{t}^{k}\|^{2}\\
	\mathbb{E}\sum_{k=1}^{N}p_{k}\|\ov{w}_{t}-\vw_{t}^{k}\|^{2} & \leq\mathbb{E}2(1+2l^{2}\alpha_{t-1}^{2})\sum_{k=1}^{N}p_{k}\|\ov{w}_{t-1}-\vw_{t-1}^{k}\|^{2}+8\alpha_{t-1}^{2}lF(\ov{w}_{t-1})-2\alpha_{t-1}^{2}\|\vg_{t-1}\|^{2}
	\end{align*}
	where the first inequality has been established before. To establish
	the second inequality, note that 
	\begin{align*}
	\sum_{k=1}^{N}p_{k}\|\ov{w}_{t}-\vw_{t}^{k}\|^{2} & =\sum_{k=1}^{N}p_{k}\|\ov{w}_{t-1}-\alpha_{t-1}\vg_{t-1}-\vw_{t-1}^{k}+\alpha_{t-1}\vg_{t-1,k}\|^{2}\\
	& \leq2\sum_{k=1}^{N}p_{k}\left(\|\ov{w}_{t-1}-\vw_{t-1}^{k}\|^{2}+\|\alpha_{t-1}\vg_{t-1}-\alpha_{t-1}\vg_{t-1,k}\|^{2}\right)
	\end{align*}
	and
	\begin{align*}
	& \sum_{k}p_{k}\|\vg_{t-1,k}-\vg_{t-1}\|^{2}=\sum_{k}p_{k}(\|\vg_{t-1,k}\|^{2}-\|\vg_{t-1}\|^{2})\\
	& =\sum_{k}p_{k}\|\nabla F_{k}(\ov{w}_{t-1},\xi_{t-1}^{k})+\nabla F_{k}(\vw_{t-1}^{k},\xi_{t-1}^{k})-\nabla F_{k}(\ov{w}_{t-1},\xi_{t-1}^{k})\|^{2}-\|\vg_{t-1}\|^{2}\\
	& \leq2\sum_{k}p_{k}\left(\|\nabla F_{k}(\ov{w}_{t-1},\xi_{t-1}^{k})\|^{2}+l^{2}\|\vw_{t-1}^{k}-\ov{w}_{t-1}\|^{2}\right)-\|\vg_{t-1}\|^{2}
	\end{align*}
	so that using the $l$-smoothness of $\ell$, 
	\begin{align*}
	& \mathbb{E}\sum_{k=1}^{N}p_{k}\|\ov{w}_{t}-\vw_{t}^{k}\|^{2}\\
	& \leq\mathbb{E}2(1+2l^{2}\alpha_{t-1}^{2})\sum_{k=1}^{N}p_{k}\|\ov{w}_{t-1}-\vw_{t-1}^{k}\|^{2}+4\alpha_{t-1}^{2}\sum_{k}p_{k}\|\nabla F_{k}(\ov{w}_{t-1},\xi_{t-1}^{k})\|^{2}-2\alpha_{t-1}^{2}\|\vg_{t-1}\|^{2}\\
	& \leq\mathbb{E}2(1+2l^{2}\alpha_{t-1}^{2})\sum_{k=1}^{N}p_{k}\|\ov{w}_{t-1}-\vw_{t-1}^{k}\|^{2}\\
 &+4\alpha_{t-1}^{2}2l\sum_{k}p_{k}(F_{k}(\ov{w}_{t-1},\xi_{t-1}^{k})-F_{k}(\vw^{\ast},\xi_{t-1}^{k}))-2\alpha_{t-1}^{2}\|\vg_{t-1}\|^{2}\\
	& =\mathbb{E}2(1+2l^{2}\alpha_{t-1}^{2})\sum_{k=1}^{N}p_{k}\|\ov{w}_{t-1}-\vw_{t-1}^{k}\|^{2}+8\alpha_{t-1}^{2}lF(\ov{w}_{t-1})-2\alpha_{t-1}^{2}\|\vg_{t-1}\|^{2}
	\end{align*}

	Using the first inequality, we have 
	\begin{align*}
	\mathbb{E}\|\ov{w}_{t+1}-\vw^{\ast}\|^{2} & \leq\mathbb{E}(1-\alpha_{t}\mu)\|\ov{w}_{t}-\vw^{\ast}\|^{2}\\
	& -2\alpha_{t}F(\ov{w}_{t})+2\alpha_{t}^{2}\|\sum_{k=1}^{N}p_{k}\nabla F_{k}(\ov{w}_{t},\xi_{t}^{k})\|^{2}\\
	& +(2\alpha_{t}^{2}l^{2}+\alpha_{t}L)\sum_{k=1}^{N}p_{k}\|\ov{w}_{t}-\vw_{t}^{k}\|^{2}
	\end{align*}
	and we choose $\alpha_{t}$ and $\alpha_{t-1}$ such that $\mathbb{E}(F(\ov{w}_{t})-\alpha_{t}\|\sum_{k=1}^{N}p_{k}\nabla F_{k}(\ov{w}_{t},\xi_{t}^{k})\|^{2})\geq0$
	and $(2\alpha_{t}^{2}l^{2}+\alpha_{t}L)\leq(1-\alpha_{t}\mu)(2\alpha_{t-1}^{2}l^{2}+\alpha_{t-1}L)/3$.
	This gives 
	\begin{align*}
	&\mathbb{E}\|\ov{w}_{t+1}-\vw^{\ast}\|^{2}\\
 & \leq\mathbb{E}(1-\alpha_{t}\mu)[(1-\alpha_{t-1}\mu)\|\ov{w}_{t-1}-\vw^{\ast}\|^{2}-2\alpha_{t-1}F(\ov{w}_{t-1})+2\alpha_{t-1}^{2}\|\sum_{k=1}^{N}p_{k}\nabla F_{k}(\ov{w}_{t-1},\xi_{t-1}^{k})\|^{2}\\
	& +(2\alpha_{t-1}^{2}l^{2}+\alpha_{t-1}L)(\sum_{k=1}^{N}p_{k}\|\ov{w}_{t-1}-\vw_{t-1}^{k}\|^{2}+\sum_{k=1}^{N}p_{k}\|\ov{w}_{t}-\vw_{t}^{k}\|^{2})/3]
	\end{align*}
	
	Using the second inequality
	\begin{align*}
	\sum_{k=1}^{N}p_{k}\|\ov{w}_{t}-\vw_{t}^{k}\|^{2} & \leq\mathbb{E}2(1+2l^{2}\alpha_{t-1}^{2})\sum_{k=1}^{N}p_{k}\|\ov{w}_{t-1}-\vw_{t-1}^{k}\|^{2}+8\alpha_{t-1}^{2}lF(\ov{w}_{t-1})-2\alpha_{t-1}^{2}\|\vg_{t-1}\|^{2}
	\end{align*}
	and that $2(1+2l^{2}\alpha_{t-1}^{2})\leq3$, $2\alpha_{t-1}^{2}l^{2}+\alpha_{t-1}L\le1$,
	we have 
	\begin{align*}
	\mathbb{E}\|\ov{w}_{t+1}-\vw^{\ast}\|^{2} & \leq\mathbb{E}(1-\alpha_{t}\mu)[(1-\alpha_{t-1}\mu)\|\ov{w}_{t-1}-\vw^{\ast}\|^{2}\\
	& -2\alpha_{t-1}F(\ov{w}_{t-1})+2\alpha_{t-1}^{2}\|\sum_{k=1}^{N}p_{k}\nabla F_{k}(\ov{w}_{t-1},\xi_{t-1}^{k})\|^{2}+8\alpha_{t-1}^{2}lF(\ov{w}_{t-1})\\
	& +(2\alpha_{t-1}^{2}l^{2}+\alpha_{t-1}L)(2\sum_{k=1}^{N}p_{k}\|\ov{w}_{t-1}-\vw_{t-1}^{k}\|^{2})]
	\end{align*}
	and if $\alpha_{t-1}$ is chosen such that 
	\begin{align*}
	(F(\ov{w}_{t-1})-4\alpha_{t-1}lF(\ov{w}_{t-1}))-\alpha_{t-1}\|\sum_{k=1}^{N}p_{k}\nabla F_{k}(\ov{w}_{t-1},\xi_{t-1}^{k})\|^{2}\geq0
	\end{align*}
	and
	\begin{align*} (2\alpha_{t-1}^{2}l^{2}+\alpha_{t-1}L)(1-\alpha_{t-1}\mu)
	&\leq(2\alpha_{t-2}^{2}l^{2}+\alpha_{t-2}L)/3
	\end{align*}
	we again have 
	\begin{align*}
	&\mathbb{E}\|\ov{w}_{t+1}-\vw^{\ast}\|^{2} \\
 & \leq\mathbb{E}(1-\alpha_{t}\mu)(1-\alpha_{t-1}\mu)[\|\ov{w}_{t-1}-\vw^{\ast}\|^{2}+(2\alpha_{t-2}^{2}l^{2}+\alpha_{t-2}L)\cdot(2\sum_{k=1}^{N}p_{k}\|\ov{w}_{t-1}-\vw_{t-1}^{k}\|^{2})/3]
	\end{align*}
	
	Applying the above derivation iteratively $\tau<E$ times, we have
	\begin{align*}
	\mathbb{E}\|\ov{w}_{t+1}-\vw^{\ast}\|^{2} & \leq\mathbb{E}(1-\alpha_{t}\mu)\cdots(1-\alpha_{t-\tau+1}\mu)[(1-\alpha_{t-\tau}\mu)\|\ov{w}_{t-\tau}-\vw^{\ast}\|^{2}\\
	& -2\alpha_{t-\tau}F(\ov{w}_{t-\tau})+2\alpha_{t-\tau}^{2}\|\sum_{k=1}^{N}p_{k}\nabla F_{k}(\ov{w}_{t-\tau},\xi_{t-\tau}^{k})\|^{2}+8\tau\alpha_{t-\tau}^{2}lF(\ov{w}_{t-\tau})\\
	& +(2\alpha_{t-\tau}^{2}l^{2}+\alpha_{t-\tau}L)((\tau+1)\sum_{k=1}^{N}p_{k}\|\ov{w}_{t-\tau}-\vw_{t-\tau}^{k}\|^{2})]
	\end{align*}
	as long as the step sizes $\alpha_{t-\tau}$ are chosen such that
	the following inequalities hold 
	\begin{align*}
	(2\alpha_{t-\tau}^{2}l^{2}+\alpha_{t-\tau}L)(1-\alpha_{t-\tau}\mu) & \leq(2\alpha_{t-\tau-1}^{2}l^{2}+\alpha_{t-\tau-1}L)/3\\
	2(1+2l^{2}\alpha_{t-\tau}^{2}) & \leq3\\
	2\alpha_{t-\tau}^{2}l^{2}+\alpha_{t-\tau}L & \leq1\\
	(F(\ov{w}_{t-\tau})-4\tau\alpha_{t-\tau}lF(\ov{w}_{t-\tau}))-\alpha_{t-\tau}\|\sum_{k=1}^{N}p_{k}\nabla F_{k}(\ov{w}_{t-\tau},\xi_{t-\tau}^{k})\|^{2} & \geq0
	\end{align*}
	We can check that setting $\alpha_{t-\tau}=c\frac{1}{\tau+1}\frac{N}{l\nu_{\max}+L(N-\nu_{\min})}$
	for some small constant $c$ satisfies the requirements. 
	
	Since communication is done every $E$ iterations, $\ov{w}_{t_{0}}=\vw_{t_{0}}^{k}$
	for some $t_{0}>t-E$ , from which we can conclude that 
	
	\begin{align*}
	\mathbb{E}\|\ov{w}_{t}-\vw^{\ast}\|^{2} & \leq(\prod_{\tau=1}^{t-t_{0}-1}(1-\mu\alpha_{t-\tau}))\|\vw_{t_{0}}-\vw^{\ast}\|^{2}\\
	& \leq(1-c\frac{\mu}{E}\frac{N}{l\nu_{\max}+L(N-\nu_{\min})})^{t-t_{0}}\|\vw_{t_{0}}-\vw^{\ast}\|^{2}
	\end{align*}
	and applying this inequality to iterations between each communication
	round, 
	\begin{align*}
	\mathbb{E}\|\ov{w}_{t}-\vw^{\ast}\|^{2} & \leq(1-c\frac{\mu}{E}\frac{N}{l\nu_{\max}+L(N-\nu_{\min})})^{t}\|\vw_{0}-\vw^{\ast}\|^{2}\\
	& =O(\exp(\frac{\mu}{E}\frac{N}{l\nu_{\max}+L(N-\nu_{\min})}t))\|\vw_{0}-\vw^{\ast}\|^{2}
	\end{align*}
	
	With partial participation, we note that 
	\begin{align*}
	\mathbb{E}\|\ov{w}_{t+1}-\vw^{\ast}\|^{2} & =\mathbb{E}\|\ov{w}_{t+1}-\ov{v}_{t+1}+\ov{v}_{t+1}-\vw^{\ast}\|^{2}\\
	& =\mathbb{E}\|\ov{w}_{t+1}-\ov{v}_{t+1}\|^{2}+\mathbb{E}\|\ov{v}_{t+1}-\vw^{\ast}\|^{2}\\
	& =\frac{1}{K}\sum_{k}p_{k}\mathbb{E}\|\vw_{t+1}^{k}-\ov{w}_{t+1}\|^{2}+\mathbb{E}\|\ov{v}_{t+1}-\vw^{\ast}\|^{2}
	\end{align*}
	and so the recursive identity becomes 
	\begin{align*}
	\mathbb{E}\|\ov{w}_{t+1}-\vw^{\ast}\|^{2} & \leq\mathbb{E}(1-\alpha_{t}\mu)\cdots(1-\alpha_{t-\tau+1}\mu)[(1-\alpha_{t-\tau}\mu)\|\ov{w}_{t-\tau}-\vw^{\ast}\|^{2}\\
	& -2\alpha_{t-\tau}F(\ov{w}_{t-\tau})+2\alpha_{t-\tau}^{2}\|\sum_{k=1}^{N}p_{k}\nabla F_{k}(\ov{w}_{t-\tau},\xi_{t-\tau}^{k})\|^{2}+8\tau\alpha_{t-\tau}^{2}lF(\ov{w}_{t-\tau})\\
	& +(2\alpha_{t-\tau}^{2}l^{2}+\alpha_{t-\tau}L+\frac{1}{K})((\tau+1)\sum_{k=1}^{N}p_{k}\|\ov{w}_{t-\tau}-\vw_{t-\tau}^{k}\|^{2})]
	\end{align*}
	which requires 
	\begin{align*}
	(2\alpha_{t-\tau}^{2}l^{2}+\alpha_{t-\tau}L+\frac{1}{K})(1-\alpha_{t-\tau}\mu) & \leq(2\alpha_{t-\tau-1}^{2}l^{2}+\alpha_{t-\tau-1}L+\frac{1}{K})/3\\
	2(1+2l^{2}\alpha_{t-\tau}^{2}) & \leq3\\
	2\alpha_{t-\tau}^{2}l^{2}+\alpha_{t-\tau}L+\frac{1}{K} & \leq1\\
	(F(\ov{w}_{t-\tau})-4\tau\alpha_{t-\tau}lF(\ov{w}_{t-\tau}))&-\alpha_{t-\tau}\|\sum_{k=1}^{N}p_{k}\nabla F_{k}(\ov{w}_{t-\tau},\xi_{t-\tau}^{k})\|^{2}  \geq0
	\end{align*}
	to hold. Again setting $\alpha_{t-\tau}=c\frac{1}{\tau+1}\frac{N}{l\nu_{\max}+L(N-\nu_{\min})}$
	for a possibly different constant from before satisfies the requirements.
	
	Finally, using the $L$-smoothness of $F$, 
	\begin{align*}
	F(\ov{w}_{T})-F(\vw^{\ast}) & \leq\frac{L}{2}\mathbb{E}\|\ov{w}_{T}-\vw^{\ast}\|^{2}=O(L\exp(-\frac{\mu}{E}\frac{N}{l\nu_{\max}+L(N-\nu_{\min})}T))\|\vw_{0}-\vw^{\ast}\|^{2}
	\end{align*}
\end{proof}

\subsection{Geometric Convergence of FedAvg for Overparameterized Linear Regression}
\label{app:geometric_proof}

We first provide details on quantities used in the proof of results on linear regression in Section~\ref{sec:app:overparameterized}. Recall that the local device objectives are now given by the sum of squares {\small$F_{k}(\mathbf{w})=\frac{1}{2n_{k}}\sum_{j=1}^{n_{k}}(\mathbf{w}^{T}\mathbf{x}_{k}^{j}-\vz_{k}^{j})^{2}$},
and there exists $\mathbf{w}^{\ast}$ such that $F(\mathbf{w}^{\ast})\equiv0$. 
Define the local Hessian matrix as $\vH^{k}:=\frac{1}{n_{k}}\sum_{j=1}^{n_{k}}\mathbf{x}_{k}^{j}(\mathbf{x}_{k}^{j})^{T}$, and the stochastic Hessian matrix as $\tilde{\vH}_{t}^{k}:=\xi_{t}^{k}(\xi_{t}^{k})^{T}$, where $\xi_{t}^{k}$ is the stochastic sample on the $k$th device at
time $t$. Define $l$ to be the smallest positive number such that $\mathbb{E}\|\xi_{t}^{k}\|^{2}$$\mathbf{\xi}_{t}^{k}$($\mathbf{\xi}_{t}^{k})^{T}\preceq l\vH^{k}$ for all $k$. Note that $l\leq\max_{k,j}\|\mathbf{x}_{k}^{j}\|^{2}$.
Let $L$ and $\mu$ be lower and upper bounds of non-zero eigenvalues
of $\vH^{k}$. Define $\kappa_{1}:=l/\mu$ and $\kappa:=L/\mu$. The condition number $\kappa_{1}$ is important in the characterization of convergence rates for FedAvg
algorithms. Note that $\kappa_{1}>\kappa$.

Let $\vH=\sum_{k}p_k\vH^k$. In general $\vH$ has zero eigenvalues. However, because the null space
of $\vH$ and range of $\vH$ are orthogonal, in our subsequence analysis
it suffices to project $\overline{\mathbf{w}}_{t}-\mathbf{w}^{\ast}$
onto the range of $\vH$, thus we may restrict to the non-zero eigenvalue
of $\vH$. 

A useful observation is that we can use $\mathbf{w}^{\ast T}\mathbf{x}_{k}^{j}-\vz_{k}^{j}\equiv0$
to rewrite the local objectives as $F_{k}(\mathbf{w})=\frac{1}{2}\langle\mathbf{w}-\mathbf{w}^{\ast},\vH^{k}(\mathbf{w}-\mathbf{w}^{\ast})\rangle\equiv\frac{1}{2}\|\mathbf{w}-\mathbf{w}^{\ast}\|_{\vH^{k}}^{2}$:
\begin{align*}
F_{k}(\vw) & =\frac{1}{2n_{k}}\sum_{j=1}^{n_{k}}(\vw^{T}\vx_{k,j}-\vz_{k,j}-(\vw^{\ast T}\vx_{k,j}-\vz_{k,j}))^{2}=\frac{1}{2n_{k}}\sum_{j=1}^{n_{k}}((\vw-\vw^{\ast})^{T}\vx_{k,j})^{2}\\
& =\frac{1}{2}\langle \vw-\vw^{\ast},\vH^{k}(\vw-\vw^{\ast})\rangle=\frac{1}{2}\|\vw-\vw^{\ast}\|_{\vH^{k}}^{2}
\end{align*}

so that $F(\mathbf{w})=\frac{1}{2}\|\mathbf{w}-\mathbf{w}^{\ast}\|_{H}^{2}$.

Finally, note that $\mathbb{E}\tilde{\vH}_{t}^{k}=\frac{1}{n_{k}}\sum_{j=1}^{n_{k}}\mathbf{x}_{k}^{j}(\mathbf{x}_{k}^{j})^{T}=\vH^{k}$
and $\mathbf{g}_{t,k}=\nabla F_{k}(\mathbf{w}_{t}^{k},\xi_{t}^{k})=\tilde{\vH}_{t}^{k}(\mathbf{w}_{t}^{k}-\mathbf{w}^{\ast})$
while $\mathbf{g}_{t}=\sum_{k=1}^{N}p_{k}\nabla F_{k}(\mathbf{w}_{t}^{k},\xi_{t}^{k})=\sum_{k=1}^{N}p_{k}\tilde{\vH}_{t}^{k}(\mathbf{w}_{t}^{k}-\mathbf{w}^{\ast})$ and $\overline{\mathbf{g}}_{t}=\sum_{k=1}^{N}p_{k}\vH^{k}(\mathbf{w}_{t}^{k}-\mathbf{w}^{\ast})$ 
\\
	\begin{thm}
		For the overparamterized linear regression problem, FedAvg with communication every $E$
		iterations with constant step size $\overline{\alpha}=\mathcal{O}(\frac{1}{E}\frac{N}{l\nu_{\max}+\mu(N-\nu_{\min})})$
		has geometric convergence:
		\begin{align*}
		\mathbb{E}F(\overline{\mathbf{w}}_{T}) & \leq\mathcal{O}\left(L\exp(-\frac{NT}{E(\nu_{\max}\kappa_{1}+(N-\nu_{\min}))})\|\mathbf{w}_{0}-\mathbf{w}^{\ast}\|^{2}\right).
		\end{align*}
	\end{thm}
\begin{proof}
	We again show the result first when $E=2$ and $t-1$ is a communication
	round. We have 
	\begin{align*}
	\|\ov{w}_{t+1}-\vw^{\ast}\|^{2} & =\|(\ov{w}_{t}-\alpha_{t}\vg_{t})-\vw^{\ast}\|^{2}\\
	& =\|\ov{w}_{t}-\vw^{\ast}\|^{2}-2\alpha_{t}\langle\ov{w}_{t}-\vw^{\ast},\vg_{t}\rangle+\alpha_{t}^{2}\|\vg_{t}\|^{2}
	\end{align*}
	and 
	\begin{align*}
	& -2\alpha_{t}\mathbb{E}_{t}\langle\ov{w}_{t}-\vw^{\ast},\vg_{t}\rangle\\
	& =-2\alpha_{t}\sum_{k=1}^{N}p_{k}\langle\ov{w}_{t}-\vw^{\ast},\nabla F_{k}(\vw_{t}^{k})\rangle\\
	& =-2\alpha_{t}\sum_{k=1}^{N}p_{k}\langle\ov{w}_{t}-\vw_{t}^{k},\nabla F_{k}(\vw_{t}^{k})\rangle-2\alpha_{t}\sum_{k=1}^{N}p_{k}\langle \vw_{t}^{k}-\vw^{\ast},\nabla F_{k}(\vw_{t}^{k})\rangle\\
	& =-2\alpha_{t}\sum_{k=1}^{N}p_{k}\langle\ov{w}_{t}-\vw_{t}^{k},\nabla F_{k}(\vw_{t}^{k})\rangle-2\alpha_{t}\sum_{k=1}^{N}p_{k}\langle \vw_{t}^{k}-\vw^{\ast},\vH^{k}(\vw_{t}^{k}-\vw^{\ast})\rangle\\
	& =-2\alpha_{t}\sum_{k=1}^{N}p_{k}\langle\ov{w}_{t}-\vw_{t}^{k},\nabla F_{k}(\vw_{t}^{k})\rangle-4\alpha_{t}\sum_{k=1}^{N}p_{k}F_{k}(\vw_{t}^{k})\\
	& \leq2\alpha_{t}\sum_{k=1}^{N}p_{k}(F_{k}(\vw_{t}^{k})-F_{k}(\ov{w}_{t})+\frac{L}{2}\|\ov{w}_{t}-\vw_{t}^{k}\|^{2})-4\alpha_{t}\sum_{k=1}^{N}p_{k}F_{k}(\vw_{t}^{k})\\
	& =\alpha_{t}L\sum_{k=1}^{N}p_{k}\|\ov{w}_{t}-\vw_{t}^{k}\|^{2}-2\alpha_{t}\sum_{k=1}^{N}p_{k}F_{k}(\ov{w}_{t})-2\alpha_{t}\sum_{k=1}^{N}p_{k}F_{k}(\vw_{t}^{k})\\
	& =\alpha_{t}L\sum_{k=1}^{N}p_{k}\|\ov{w}_{t}-\vw_{t}^{k}\|^{2}-\alpha_{t}\sum_{k=1}^{N}p_{k}\langle(\ov{w}_{t}-\vw^{\ast}),\vH^{k}(\ov{w}_{t}-\vw^{\ast})\rangle-2\alpha_{t}\sum_{k=1}^{N}p_{k}F_{k}(\vw_{t}^{k})
	\end{align*}
	and 
	\begin{align*}
	\|\vg_{t}\|^{2} & =\|\sum_{k=1}^{N}p_{k}\tilde{\vH}_{t}^{k}(\vw_{t}^{k}-\vw^{\ast})\|^{2}\\
	& =\|\sum_{k=1}^{N}p_{k}\tilde{\vH}_{t}^{k}(\ov{w}_{t}-\vw^{\ast})+\sum_{k=1}^{N}p_{k}\tilde{\vH}_{t}^{k}(\vw_{t}^{k}-\ov{w}_{t})\|^{2}\\
	& \leq2\|\sum_{k=1}^{N}p_{k}\tilde{\vH}_{t}^{k}(\ov{w}_{t}-\vw^{\ast})\|^{2}+2\|\sum_{k=1}^{N}p_{k}\tilde{\vH}_{t}^{k}(\vw_{t}^{k}-\ov{w}_{t})\|^{2}
	\end{align*}
	which gives 
	\begin{align*}
	&\mathbb{E}\|\ov{w}_{t+1}-\vw^{\ast}\|^{2} \\
 & \leq\mathbb{E}\|\ov{w}_{t}-\vw^{\ast}\|^{2}-\alpha_{t}\sum_{k=1}^{N}p_{k}\langle\ov{w}_{t}-\vw^{\ast},\vH^{k}\ov{w}_{t}-\vw^{\ast}\rangle+2\alpha_{t}^{2}\|\sum_{k=1}^{N}p_{k}\tilde{\vH}_{t}^{k}(\ov{w}_{t}-\vw^{\ast})\|^{2}\\
	& +\alpha_{t}L\sum_{k=1}^{N}p_{k}\|\ov{w}_{t}-\vw_{t}^{k}\|^{2}+2\alpha_{t}^{2}\|\sum_{k=1}^{N}p_{k}\tilde{\vH}_{t}^{k}(\vw_{t}^{k}-\ov{w}_{t})\|^{2}-2\alpha_{t}\sum_{k=1}^{N}p_{k}F_{k}(\vw_{t}^{k})
	\end{align*}
	following~\cite{ma2017power} we first prove that 
	\begin{align*}
	\mathbb{E}\|\ov{w}_{t}-\vw^{\ast}\|^{2}-\alpha_{t}\sum_{k=1}^{N}p_{k}\langle(\ov{w}_{t}-\vw^{\ast}),\vH^{k}(\ov{w}_{t}-\vw^{\ast})\rangle+2\alpha_{t}^{2}\|\sum_{k=1}^{N}p_{k}\tilde{\vH}_{t}^{k}(\ov{w}_{t}-\vw^{\ast})\|^{2}\\
	\leq(1-\frac{N}{8(\nu_{\max}\kappa_{1}+(N-\nu_{\min}))})\mathbb{E}\|\ov{w}_{t}-\vw^{\ast}\|^{2}
	\end{align*}
	with appropriately chosen $\alpha_{t}$. Compared to the rate $O(\frac{\mu N}{l\nu_{\max}+L(N-\nu_{\min})})=O(\frac{N}{\nu_{\max}\kappa_{1}+(N-\nu_{\min})\kappa})$
	for general strongly convex and smooth objectives, this is an improvement
	as linear speedup is now available for a larger range of $N$. 
	
	We have 
	\begin{align*}
	& \mathbb{E}_{t}\|\sum_{k=1}^{N}p_{k}\tilde{\vH}_{t}^{k}(\ov{w}_{t}-\vw^{\ast})\|^{2}\\
	& =\mathbb{E}_{t}\langle\sum_{k=1}^{N}p_{k}\tilde{\vH}_{t}^{k}(\ov{w}_{t}-\vw^{\ast}),\sum_{k=1}^{N}p_{k}\tilde{\vH}_{t}^{k}(\ov{w}_{t}-\vw^{\ast})\rangle\\
	& =\sum_{k=1}^{N}p_{k}^{2}\mathbb{E}_{t}\|\tilde{\vH}_{t}^{k}(\ov{w}_{t}-\vw^{\ast})\|^{2}+\sum_{k=1}^{N}\sum_{j\neq k}p_{j}p_{k}\mathbb{E}_{t}\langle\tilde{\vH}_{t}^{k}(\ov{w}_{t}-\vw^{\ast}),\tilde{\vH}_{t}^{j}(\ov{w}_{t}-\vw^{\ast})\rangle\\
	& =\sum_{k=1}^{N}p_{k}^{2}\mathbb{E}_{t}\|\tilde{\vH}_{t}^{k}(\ov{w}_{t}-\vw^{\ast})\|^{2}+\sum_{k=1}^{N}\sum_{j\neq k}p_{j}p_{k}\mathbb{E}_{t}\langle \vH^{k}(\ov{w}_{t}-\vw^{\ast}),\vH^{j}(\ov{w}_{t}-\vw^{\ast})\rangle\\
	& =\sum_{k=1}^{N}p_{k}^{2}\mathbb{E}_{t}\|\tilde{\vH}_{t}^{k}(\ov{w}_{t}-\vw^{\ast})\|^{2}+\sum_{k=1}^{N}\sum_{j=1}^{N}p_{j}p_{k}\mathbb{E}_{t}\langle \vH^{k}(\ov{w}_{t}-\vw^{\ast}),\vH^{j}(\ov{w}_{t}-\vw^{\ast})\rangle\\
 &-\sum_{k=1}^{N}p_{k}^{2}\|\vH^{k}(\ov{w}_{t}-\vw^{\ast})\|^{2}\\
	& =\sum_{k=1}^{N}p_{k}^{2}\mathbb{E}_{t}\|\tilde{\vH}_{t}^{k}(\ov{w}_{t}-\vw^{\ast})\|^{2}+\|\sum_{k}p_{k}\vH^{k}(\ov{w}_{t}-\vw^{\ast})\|^{2}-\sum_{k=1}^{N}p_{k}^{2}\|\vH^{k}(\ov{w}_{t}-\vw^{\ast})\|^{2}\\
	& \leq\sum_{k=1}^{N}p_{k}^{2}\mathbb{E}_{t}\|\tilde{\vH}_{t}^{k}(\ov{w}_{t}-\vw^{\ast})\|^{2}+\|\sum_{k}p_{k}\vH^{k}(\ov{w}_{t}-\vw^{\ast})\|^{2}-\frac{1}{N}\nu_{\min}\|\sum_{k}p_{k}\vH^{k}(\ov{w}_{t}-\vw^{\ast})\|^{2}\\
	& \leq\frac{1}{N}\nu_{\max}\sum_{k=1}^{N}p_{k}\mathbb{E}_{t}\|\tilde{\vH}_{t}^{k}(\ov{w}_{t}-\vw^{\ast})\|^{2}+(1-\frac{1}{N}\nu_{\min})\|\sum_{k}p_{k}\vH^{k}(\ov{w}_{t}-\vw^{\ast})\|^{2}\\
	& \leq\frac{1}{N}\nu_{\max}l\sum_{k=1}^{N}p_{k}\langle(\ov{w}_{t}-\vw^{\ast}),\vH^{k}(\ov{w}_{t}-\vw^{\ast})\rangle+(1-\frac{1}{N}\nu_{\min})\|\sum_{k}p_{k}\vH^{k}(\ov{w}_{t}-\vw^{\ast})\|^{2}\\
	& =\frac{1}{N}\nu_{\max}l\langle(\ov{w}_{t}-\vw^{\ast}),\vH(\ov{w}_{t}-\vw^{\ast})\rangle+(1-\frac{1}{N}\nu_{\min})\langle\ov{w}_{t}-\vw^{\ast},\vH^{2}(\ov{w}_{t}-\vw^{\ast})\rangle
	\end{align*}
	using $\|\tilde{\vH}_{t}^{k}\|\leq l$. 
	
	Now we have 
	\begin{align*}
	\mathbb{E}\|\ov{w}_{t}-\vw^{\ast}\|^{2}-\alpha_{t}\sum_{k=1}^{N}p_{k}\langle(\ov{w}_{t}-\vw^{\ast}),\vH^{k}(\ov{w}_{t}-\vw^{\ast})\rangle+2\alpha_{t}^{2}\|\sum_{k=1}^{N}p_{k}\tilde{\vH}_{t}^{k}(\ov{w}_{t}-\vw^{\ast})\|^{2} & =\\
	\langle\ov{w}_{t}-\vw^{\ast},(I-\alpha_{t}\vH+2\alpha_{t}^{2}(\frac{\nu_{\max}l}{N}\vH+\frac{N-\nu_{\min}}{N}\vH^{2}))(\ov{w}_{t}-\vw^{\ast})\rangle
	\end{align*}
	and it remains to bound the maximum eigenvalue of 
	\begin{align*}
	(I-\alpha_{t}\vH+2\alpha_{t}^{2}(\frac{\nu_{\max}l}{N}\vH+\frac{N-\nu_{\min}}{N}\vH^{2}))
	\end{align*}
	and we bound this following~\cite{ma2017power}. If we choose $\alpha_{t}<\frac{N}{2(\nu_{\max}l+(N-\nu_{\min})L)}$,
	then 
	\begin{align*}
	-\alpha_{t}\vH+2\alpha_{t}^{2}(\frac{\nu_{\max}l}{N}\vH+\frac{N-\nu_{\min}}{N}\vH^{2}) & \prec0
	\end{align*}
	and the convergence rate is given by the maximum of $1-\alpha_{t}\lambda+2\alpha_{t}^{2}(\frac{\nu_{\max}l}{N}\lambda+\frac{N-\nu_{\min}}{N}\lambda^{2})$
	maximized over the non-zero eigenvalues $\lambda$ of $\vH$. To select
	the step size $\alpha_{t}$ that gives the smallest upper bound, we
	then minimize over $\alpha_{t}$, resulting in 
	\begin{align*}
	\min_{\alpha_{t}<\frac{N}{2(\nu_{\max}l+(N-\nu_{\min})L)}}\max_{\lambda>0:\exists v,\vH v=\lambda v}\left\{ 1-\alpha_{t}\lambda+2\alpha_{t}^{2}(\frac{\nu_{\max}l}{N}\lambda+\frac{N-\nu_{\min}}{N}\lambda^{2})\right\} 
	\end{align*}
	Since the objective is quadratic in $\lambda$, the maximum is achieved
	at either the largest eigenvalue $\lambda_{\max}$ of $\vH$ or the
	smallest non-zero eigenvalue $\lambda_{\min}$ of $\vH$. 
	
	When $N\leq\frac{4\nu_{\max}l}{L-\lambda_{\min}}+4\nu_{\min}$, i.e.
	when $N=O(l/\lambda_{\min})=O(\kappa_{1})$, the optimal objective
	value is achieved at $\lambda_{\min}$ and the optimal step size is
	given by $\alpha_{t}=\frac{N}{4(\nu_{\max}l+(N-\nu_{\min})\lambda_{\min})}$.
	The optimal convergence rate (i.e. the optimal objective value) is
	equal to $1-\frac{1}{8}\frac{N\lambda_{\min}}{(\nu_{\max}l+(N-\nu_{\min})\lambda_{\min})}=1-\frac{1}{8}\frac{N}{(\nu_{\max}\kappa_{1}+(N-\nu_{\min}))}$.
	This implies that when $N=O(\kappa_{1})$, the optimal convergence
	rate has a linear speedup in $N$. When $N$ is larger, this step
	size is no longer optimal, but we still have $1-\frac{1}{8}\frac{N}{(\nu_{\max}\kappa_{1}+(N-\nu_{\min}))}$
	as an upper bound on the convergence rate. 
	
	Now we have proved 
	\begin{align*}
	\mathbb{E}\|\ov{w}_{t+1}-\vw^{\ast}\|^{2} & \leq(1-\frac{1}{8}\frac{N}{(\nu_{\max}\kappa_{1}+(N-\nu_{\min}))})\mathbb{E}\|\ov{w}_{t}-\vw^{\ast}\|^{2}\\
	& +\alpha_{t}L\sum_{k=1}^{N}p_{k}\|\ov{w}_{t}-\vw_{t}^{k}\|^{2}+2\alpha_{t}^{2}\|\sum_{k=1}^{N}p_{k}\tilde{\vH}_{t}^{k}(\vw_{t}^{k}-\ov{w}_{t})\|^{2}-2\alpha_{t}\sum_{k=1}^{N}p_{k}F_{k}(\vw_{t}^{k})
	\end{align*}
	Next we bound terms in the second line using a similar argument as
	the general case. We have 
	\begin{align*}
	2\alpha_{t}^{2}\|\sum_{k=1}^{N}p_{k}\tilde{\vH}_{t}^{k}(\vw_{t}^{k}-\ov{w}_{t})\|^{2} & \leq2\alpha_{t}^{2}l^{2}\sum_{k=1}^{N}p_{k}\|\ov{w}_{t}-\vw_{t}^{k}\|^{2}
	\end{align*}
	and 
	\begin{align*}
	\mathbb{E}\sum_{k=1}^{N}p_{k}\|\ov{w}_{t}-\vw_{t}^{k}\|^{2} & \leq\mathbb{E}2(1+2l^{2}\alpha_{t-1}^{2})\sum_{k=1}^{N}p_{k}\|\ov{w}_{t-1}-\vw_{t-1}^{k}\|^{2}+8\alpha_{t-1}^{2}lF(\ov{w}_{t-1})\\
	& =4\alpha_{t-1}^{2}l\langle\ov{w}_{t-1}-\vw^{\ast},\vH(\ov{w}_{t-1}-\vw^{\ast})\rangle
	\end{align*}
	and if $\alpha_{t},\alpha_{t-1}$ satisfy 
	\begin{align*}
	\alpha_{t}L+2\alpha_{t}^{2} & \leq(1-\frac{1}{8}\frac{N}{(\nu_{\max}\kappa_{1}+(N-\nu_{\min}))})(\alpha_{t-1}L+2\alpha_{t-1}^{2})/3\\
	2(1+2l^{2}\alpha_{t-1}^{2}) & \leq3\\
	\alpha_{t}L+2\alpha_{t}^{2} & \leq1
	\end{align*}
	we have 
	\begin{align*}
	& \mathbb{E}\|\ov{w}_{t+1}-\vw^{\ast}\|^{2}\\
	& \leq(1-\frac{1}{8}\frac{N}{(\nu_{\max}\kappa_{1}+(N-\nu_{\min}))})
 ( \mathbb{E}\|\ov{w}_{t-1}-\vw^{\ast}\|^{2}-\alpha_{t}\langle\ov{w}_{t-1}-\vw^{\ast},\vH\ov{w}_{t-1}-\vw^{\ast}\rangle\\
 &+2\alpha_{t}^{2}\|\sum_{k=1}^{N}p_{k}\tilde{\vH}_{t}^{k}(\ov{w}_{t}-\vw^{\ast})\|^{2}+(\alpha_{t-1}L+2\alpha_{t-1}^{2})\cdot2\sum_{k=1}^{N}p_{k}\|\ov{w}_{t-1}-\vw_{t-1}^{k}\|^{2}\\
	&+4\alpha_{t-1}^{2}l\langle\ov{w}_{t-1}-\vw^{\ast},\vH(\ov{w}_{t-1}-\vw^{\ast}))
	\end{align*}
	and again by choosing $\alpha_{t-1}=c\frac{N}{8(\nu_{\max}l+(N-\nu_{\min})\lambda_{\min})}$
	for a small constant $c$, we can guarantee that 
	\begin{align*}
	\mathbb{E}\|\ov{w}_{t-1}-\vw^{\ast}\|^{2}-\alpha_{t-1}\langle\ov{w}_{t-1}-\vw^{\ast},\vH\ov{w}_{t-1}-\vw^{\ast}\rangle\\
	+2\alpha_{t-1}^{2}\|\sum_{k=1}^{N}p_{k}\tilde{\vH}_{t-1}^{k}(\ov{w}_{t-1}-\vw^{\ast})\|^{2}+4\alpha_{t-1}^{2}l\langle\ov{w}_{t-1}-\vw^{\ast},\vH(\ov{w}_{t-1}-\vw^{\ast})\rangle\\
	\leq(1-c\frac{N}{16(\nu_{\max}l+(N-\nu_{\min})\lambda_{\min})})\mathbb{E}\|\ov{w}_{t-1}-\vw^{\ast}\|^{2}
	\end{align*}
	
	For general $E$, we have the recursive relation
	\begin{align*}
	&\mathbb{E}\|\ov{w}_{t+1}-\vw^{\ast}\|^{2}\\
 &\leq\mathbb{E}(1-c\frac{1}{8}\frac{N}{(\nu_{\max}\kappa_{1}+(N-\nu_{\min}))})\cdots(1-c\frac{1}{8\tau}\frac{N}{(\nu_{\max}\kappa_{1}+(N-\nu_{\min}))})[\|\ov{w}_{t-\tau}-\vw^{\ast}\|^{2}\\
	& -\alpha_{t-\tau}\langle\ov{w}_{t-\tau}-\vw^{\ast},\vH\ov{w}_{t-\tau}-\vw^{\ast}\rangle+2\alpha_{t-\tau}^{2}\|\sum_{k=1}^{N}p_{k}\tilde{\vH}_{t-\tau}^{k}(\ov{w}_{t-\tau}-\vw^{\ast})\|^{2}\\
	& +4\tau\alpha_{t-1}^{2}l\langle\ov{w}_{t-1}-\vw^{\ast},\vH(\ov{w}_{t-1}-\vw^{\ast})\rangle\\
	& +(2\alpha_{t-\tau}^{2}l^{2}+\alpha_{t-\tau}L)((\tau+1)\sum_{k=1}^{N}p_{k}\|\ov{w}_{t-\tau}-\vw_{t-\tau}^{k}\|^{2})]
	\end{align*}
	as long as the step sizes are chosen $\alpha_{t-\tau}=c\frac{N}{4\tau(\nu_{\max}l+(N-\nu_{\min})\lambda_{\min})}$
	such that the following inequalities hold 
	\begin{align*}
	(2\alpha_{t-\tau}^{2}l^{2}+\alpha_{t-\tau}L) & \leq(1-\alpha_{t-\tau}\mu)(2\alpha_{t-\tau-1}^{2}l^{2}+\alpha_{t-\tau-1}L)/3\\
	2(1+2l^{2}\alpha_{t-\tau}^{2}) & \leq3\\
	2\alpha_{t-\tau}^{2}l^{2}+\alpha_{t-\tau}L & \leq1
	\end{align*}
	and 
	\begin{align*}
	& \|\ov{w}_{t-\tau}-\vw^{\ast}\|^{2}-\alpha_{t-\tau}\langle\ov{w}_{t-\tau}-\vw^{\ast},\vH\ov{w}_{t-\tau}-\vw^{\ast}\rangle\\
	& +2\alpha_{t-\tau}^{2}\|\sum_{k=1}^{N}p_{k}\tilde{\vH}_{t-\tau}^{k}(\ov{w}_{t-\tau}-\vw^{\ast})\|^{2}+4\tau\alpha_{t-1}^{2}l\langle\ov{w}_{t-1}-\vw^{\ast},\vH(\ov{w}_{t-1}-\vw^{\ast})\rangle\\
	& \leq(1-c\frac{N}{8(\tau+1)(\nu_{\max}\kappa_{1}+(N-\nu_{\min}))})\mathbb{E}\|\ov{w}_{t-\tau}-\vw^{\ast}\|^{2}
	\end{align*}
	which gives 
	\begin{align*}
	\mathbb{E}\|\ov{w}_{t}-\vw^{\ast}\|^{2} & \leq(1-c\frac{1}{8E}\frac{N}{(\nu_{\max}\kappa_{1}+(N-\nu_{\min}))})^{t}\|\vw_{0}-\vw^{\ast}\|^{2}\\
	& =O(\exp(-\frac{1}{E}\frac{N}{(\nu_{\max}\kappa_{1}+(N-\nu_{\min}))}t))\|\vw_{0}-\vw^{\ast}\|^{2}.
	\end{align*}
\end{proof}

\section{Details on Experiments and Additional Results}
\label{sec:expsupp}

We describe the precise procedure to reproduce the results in this paper.
As we mentioned in Section~\ref{sec:exp}, we empirically verified the
linear speed up on various convex settings for both FedAvg and its
accelerated variants. For all the results, we set random seeds as $0, 1, 2$
and report the best convergence rate across the three folds. For each
run, we initialize $\vw_0 = \mathbf{0}$ and measure the number of iteration
to reach the target accuracy $\epsilon$. We use the small-scale dataset
w8a~\citep{platt1998fast}, which consists of $n = 49749 $ samples with
feature dimension $d = 300$. The label is either positive one or negative one.
The dataset has sparse binary features in $\{0, 1\}$. Each sample
has 11.15 non-zero feature values out of $300$ features on average.
We set the batch size equal to four across all experiments.
In the next following subsections,
we introduce parameter searching in each objective separately.

\subsection{Strongly Convex Objectives}
We first consider the strongly convex objective function, where we use
a regularized binary logistic regression with regularization $\lambda=1/n\approx 2e-5$. We evenly distributed on $1, 2, 4, 8, 16, 32$ devices and  report the number of iterations/rounds needed to converge to $\epsilon-$accuracy, where $\epsilon=0.005$. The optimal objective function value $f^*$
is set as $f^* = 0.126433176216545$. This is determined numerically and we follow the setting in~\cite{stich2018local}. The learning rate is decayed as the $\eta_t = \min(\eta_0, \frac{nc}{1 + t})$, where we extensively search the best learning rate $c \in \{2^{-1}c_0, 2^{-2}c_0, c_0, 2c_0, 2^{2}c_0\}$. In this case, we search the initial learning rate $\eta_0\in \{1, 32\}$ and $c_0 = 1/8$.

\subsection{Convex Smooth Objectives}
We also use binary logistic regression without regularization.
The setting is almost same as its regularized counter part. We also evenly distributed all the samples on $1, 2, 4, 8, 16, 32$ devices. The figure shows the number of iterations needed to converge to $\epsilon-$accuracy, where $\epsilon=0.02$. The optiaml objective function value is set as $f^*=0.11379089057514849$, determined numerically. 
The learning rate is decayed as the $\eta_t = \min(\eta_0, \frac{nc}{1 + t})$, where we extensively search the best learning rate $c \in \{2^{-1}c_0, 2^{-2}c_0, c_0, 2c_0, 2^{2}c_0\}$. In this case, we search the initial learning rate $\eta_0\in \{1, 32\}$ and $c_0 = 1/8$.

\subsection{Linear Regression}
For linear regression, we use the same feature vectors from w8a dataset 
and generate ground truth $[\vw^*, b^*]$ from a multivariate normal distribution
with zero mean and standard deviation one. Then we generate label 
based on $y_i = \vx_i^t\vw^* + b^*$. This procedure will ensure we satisfy
the over-parameterized setting as required in our theorems. 
We also evenly distributed all the samples on $1, 2, 4, 8, 16, 32$ devices. The figure shows the number of iterations needed to converge to $\epsilon-$accuracy, where $\epsilon=0.02$. The optiaml objective function value is $f^*=0$. 
The learning rate is decayed as the $\eta_t = \min(\eta_0, \frac{nc}{1 + t})$, where we extensively search the best learning rate $c \in \{2^{-1}c_0, 2^{-2}c_0, c_0, 2c_0, 2^{2}c_0\}$. In this case, we search the initial learning rate $\eta_0\in \{0.1, 0.12\}$ and $c_0 = 1/256$.

\subsection{Partial Participation}
To examine the linear speedup of FedAvg in partial participation setting,
we evenly distributed data on $4, 8, 16, 32, 64, 128$ devices and 
uniformly sample $50\%$ devices without replacement. 
All other hyperparameters are the same as previous sections. 

\subsection{Nesterov Accelerated FedAvg}
The experiments of Nesterov accelerated FedAvg (the update formula is given as follows) uses the same setting as
previous three sections for vanilia FedAvg.
\begin{align*}
\vy_{t+1}^{k} & =\vw_{t}^{k}-\alpha_{t}\vg_{t,k}\\
\vw_{t+1}^{k} & =\begin{cases}
\vy_{t+1}^{k}+\beta_{t}(\vy_{t+1}^{k}-\vy_{t}^{k}) & \text{if }t+1\notin\mathcal{I}_{E}\\
\sum_{k\in\mathcal{S}_{t+1}}\left(\vy_{t+1}^{k}+\beta_{t}(\vy_{t+1}^{k}-\vy_{t}^{k})\right) & \text{if }t+1\in\mathcal{I}_{E}
\end{cases}
\end{align*}
We set $\beta_t = 0.1$ and search $\alpha_t$ in the same way as $\eta_t$
in FedAvg.

\subsection{The Impact of $E$.}
In this subsection, we further examine how does the number of local steps ($E$) 
affect convergence. As shown in Figure~\ref{fig:e}, the number of iterations increases as $E$ increase, which slow down the convergence in
terms of gradient computation. However, it can save communication costs as
the number of rounds decreased when the $E$ increases. This showcases that
we need a proper choice of $E$ to trade-off the communication cost and
convergence speed. 
\begin{figure}
\centering
	\begin{tabular}{ccc}
	\hspace{-2em} \includegraphics[width=0.33\textwidth]{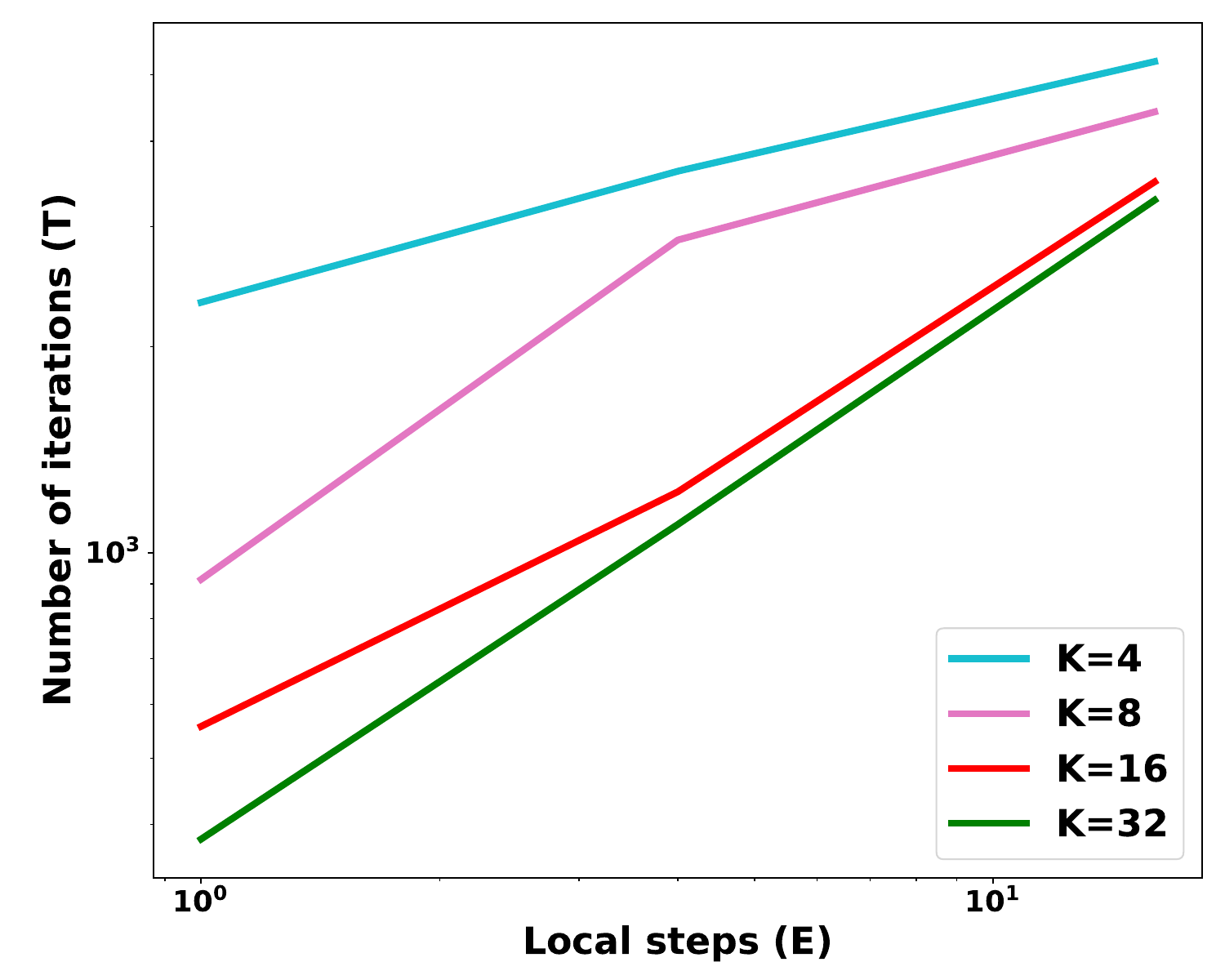} &
	\includegraphics[width=0.33\textwidth]{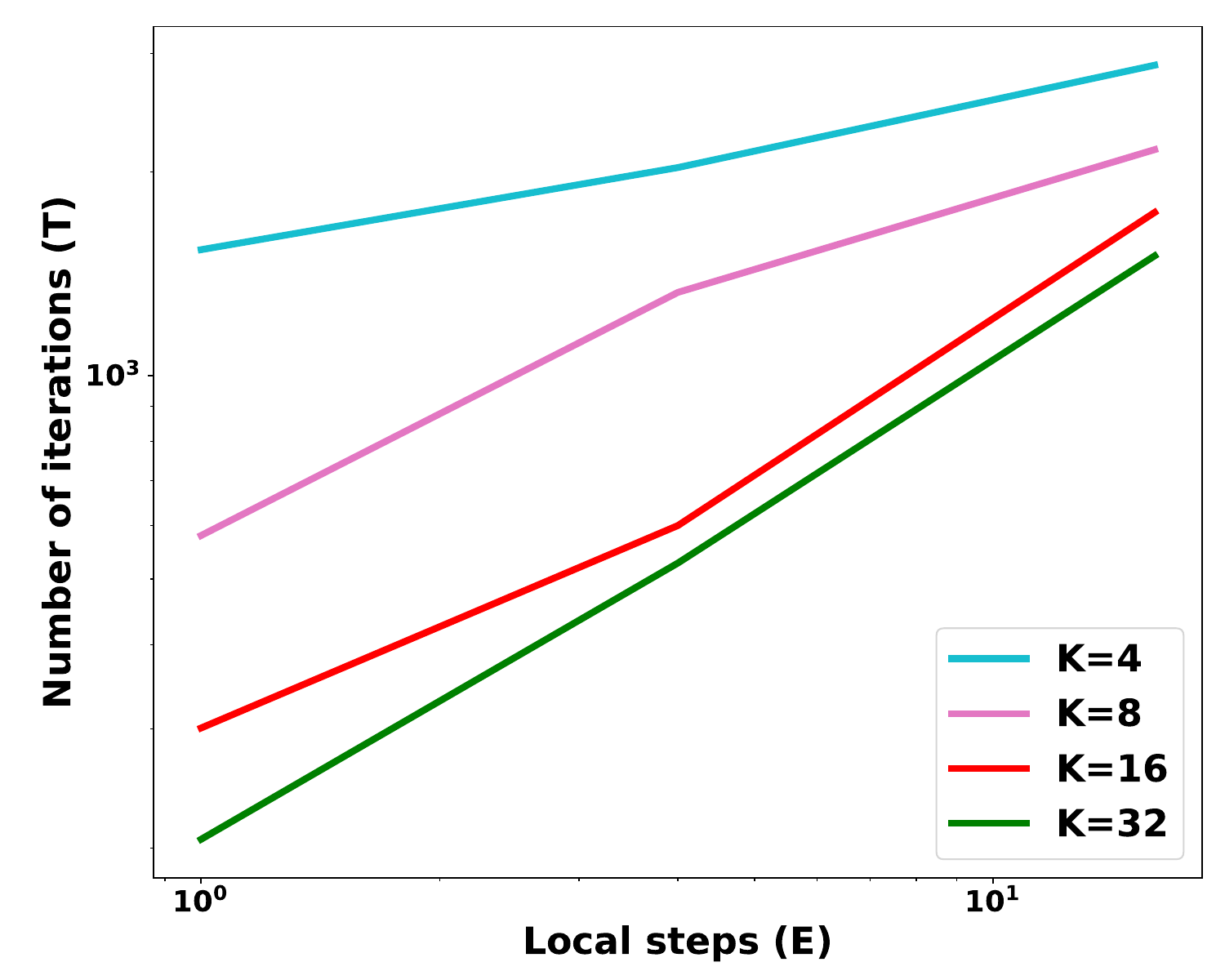} & 
	\includegraphics[width=0.33\textwidth]{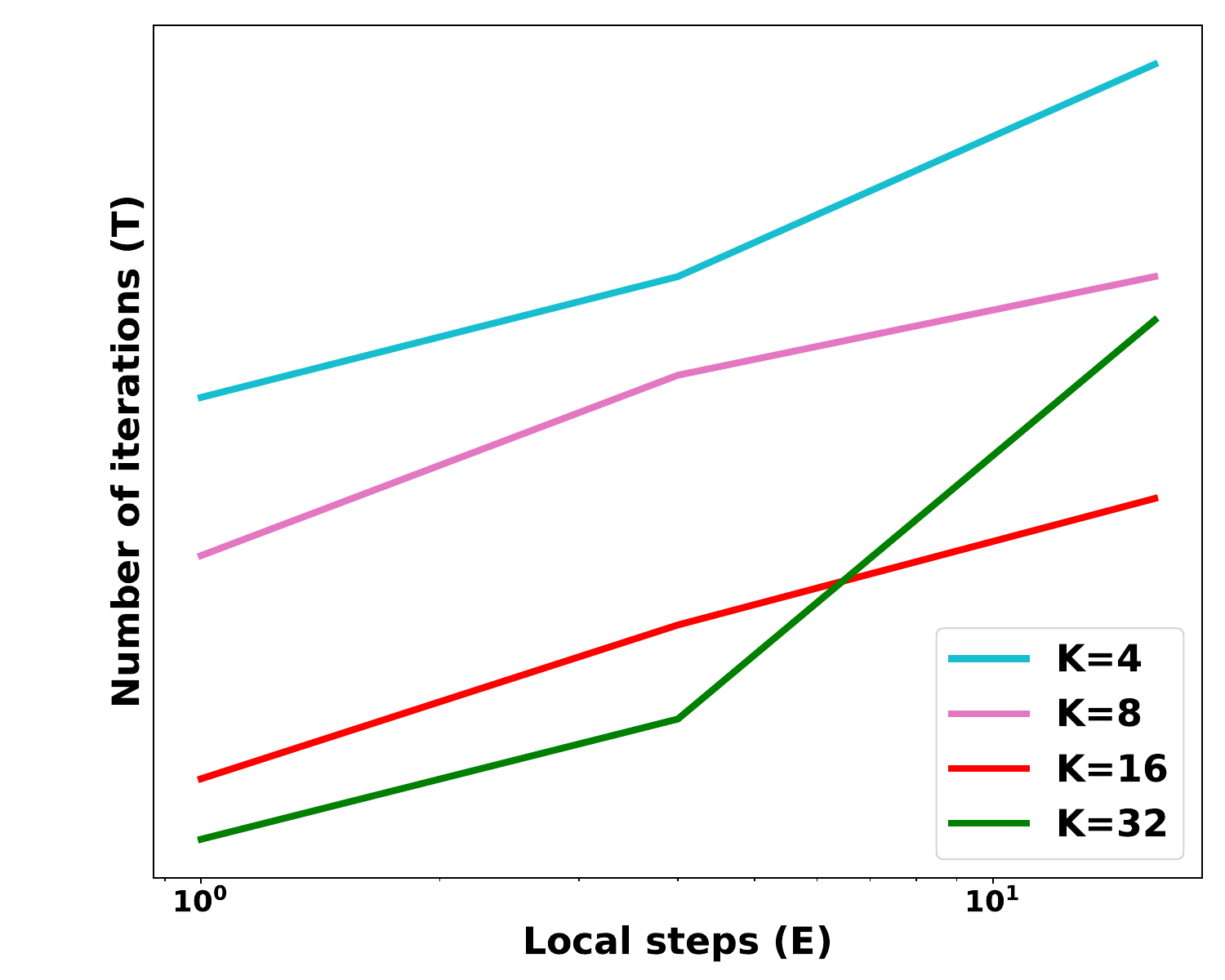} \\
	\hspace{-2em} \includegraphics[width=0.33\textwidth]{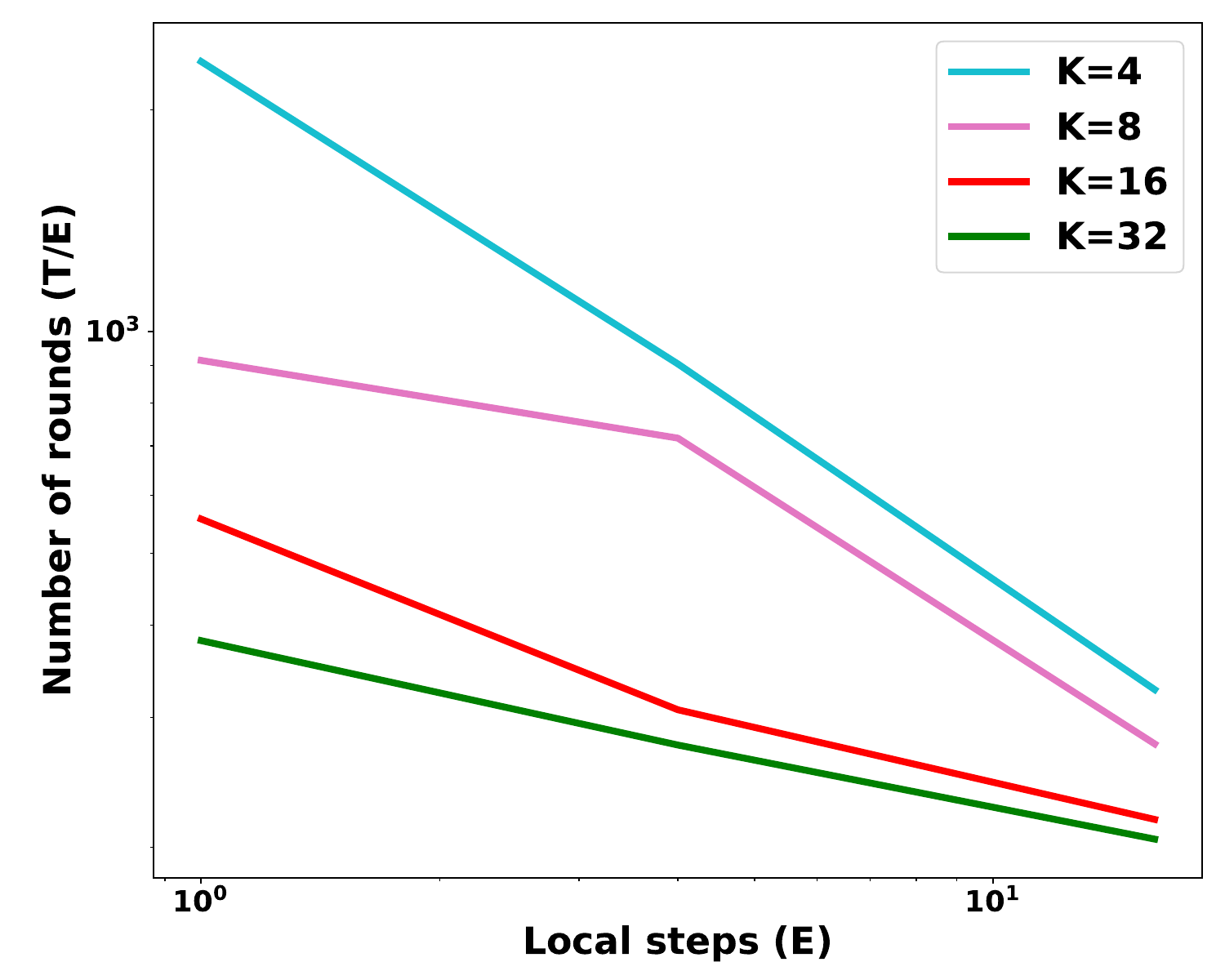} &
	\includegraphics[width=0.33\textwidth]{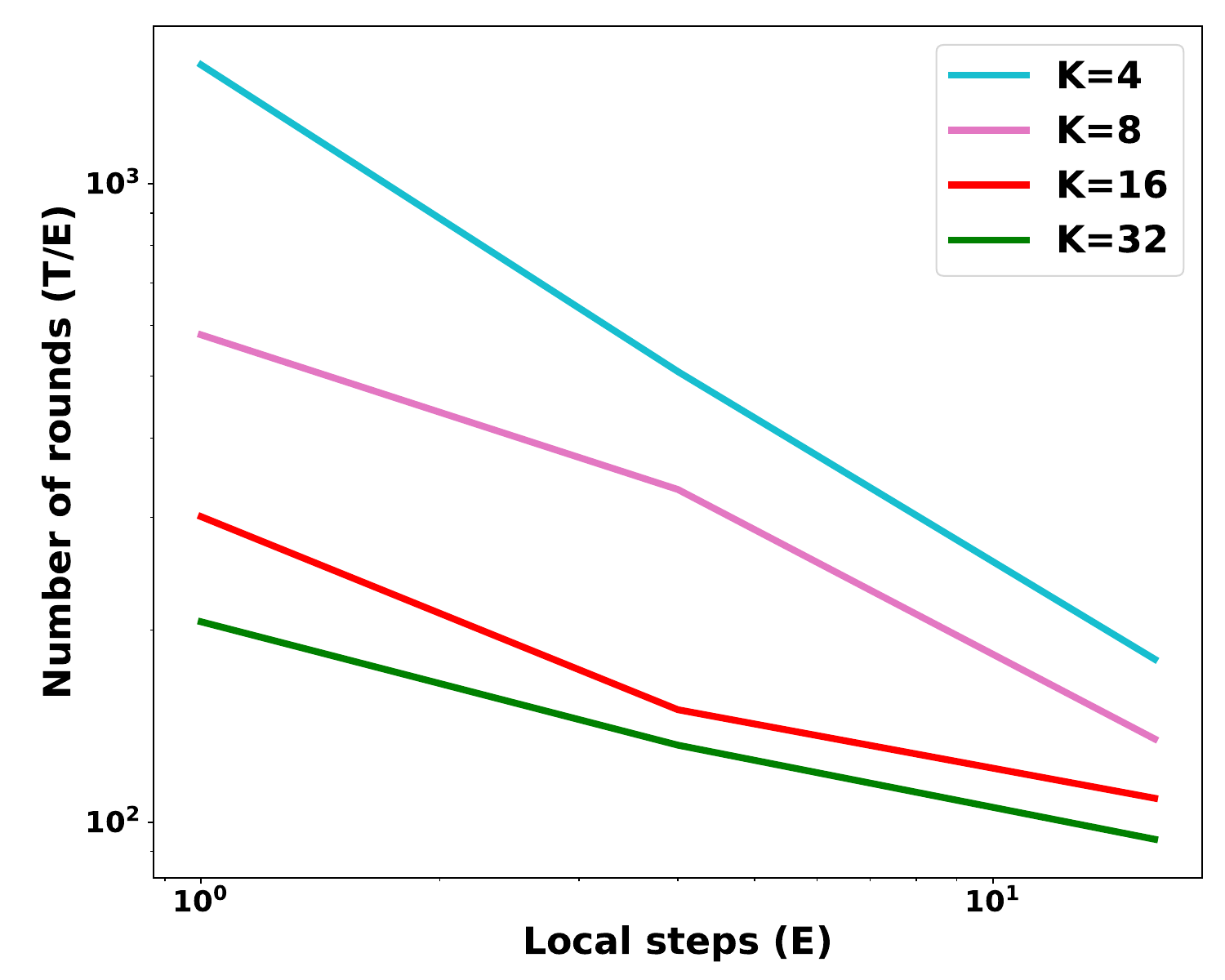} & 
	\includegraphics[width=0.33\textwidth]{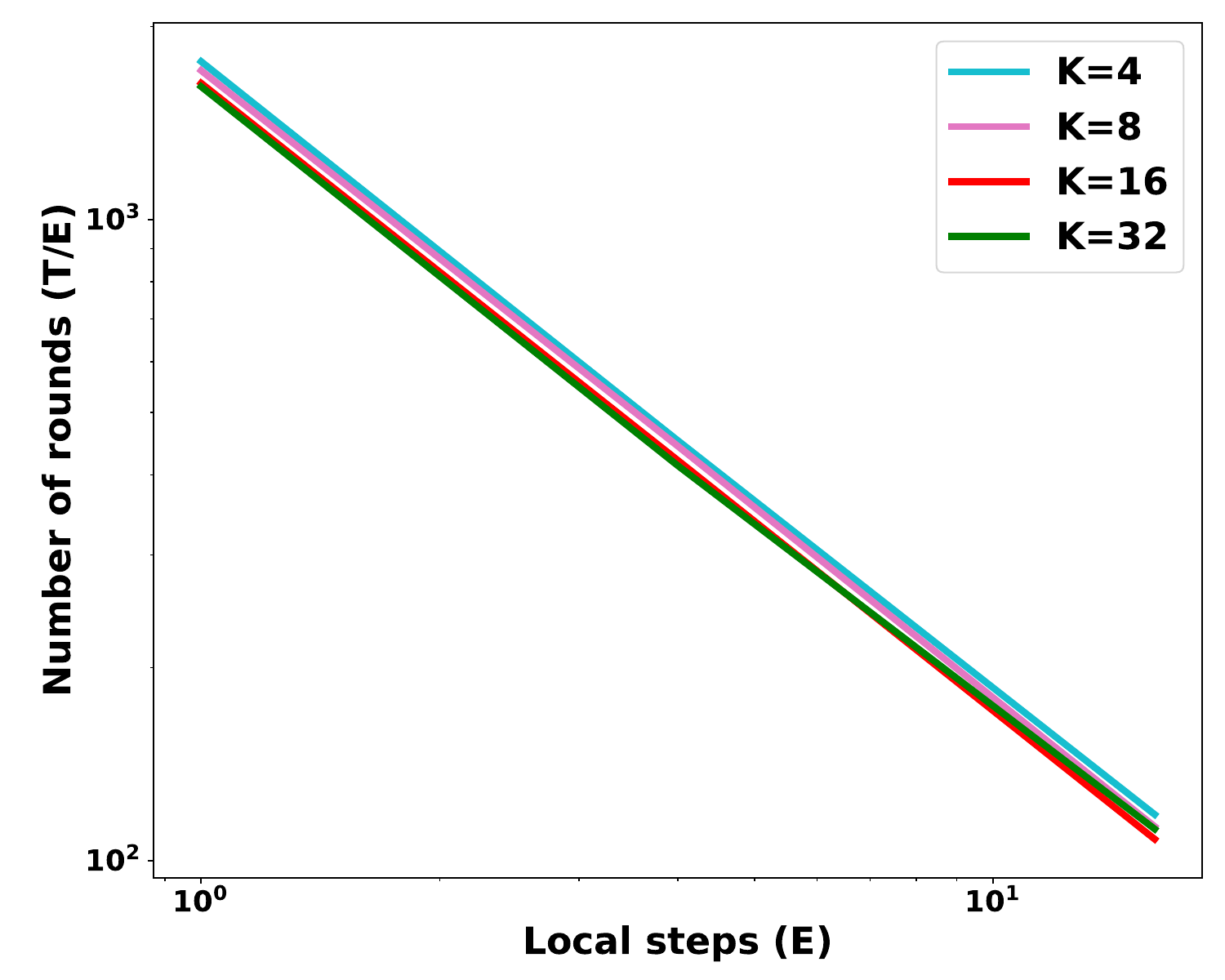} \\
(a) Strongly convex objective & (b) Convex smooth objective & (c) Linear regression
	\end{tabular}
\caption{The convergence of FedAvg w.r.t the number of local steps $E$. }
\label{fig:e}
\end{figure}

\vskip 0.2in
\bibliographystyle{apalike}
\bibliography{ref}

\end{document}